\documentclass[11pt]{article}

\usepackage{fullpage}

\usepackage[T1]{fontenc}
\usepackage[utf8]{inputenc}
\usepackage{lmodern}
\usepackage{microtype}

\usepackage{amsmath,mathtools}
\usepackage{amsfonts,amssymb}
\usepackage{amsthm,thmtools}
\usepackage{bm}
\usepackage{bbm}
\usepackage{accents}

\usepackage{graphicx}
\usepackage{xcolor}
\usepackage{booktabs}
\usepackage{array}
\usepackage{tabularx}
\usepackage{threeparttable}
\usepackage{caption}
\usepackage{subcaption}
\usepackage{makecell}
\usepackage{pifont}

\usepackage{enumitem}

\usepackage{algorithm}
\usepackage{algpseudocode}

\usepackage[round]{natbib}
\usepackage[colorlinks=true,citecolor=blue,linkcolor=red,urlcolor=magenta]{hyperref}
\usepackage{cleveref}

\crefformat{equation}{(#2#1#3)}
\Crefformat{equation}{(#2#1#3)}
\crefrangeformat{equation}{(#3#1#4)--(#5#2#6)}
\Crefrangeformat{equation}{(#3#1#4)--(#5#2#6)}
\crefmultiformat{equation}{(#2#1#3)}{ and (#2#1#3)}{, (#2#1#3)}{, and (#2#1#3)}
\Crefmultiformat{equation}{(#2#1#3)}{ and (#2#1#3)}{, (#2#1#3)}{, and (#2#1#3)}

\theoremstyle{definition}
\newtheorem{theorem}{Theorem}
\newtheorem{lemma}[theorem]{Lemma}

\newtheorem{assumption}[theorem]{Assumption}
\newtheorem{remark}[theorem]{Remark}

\newcommand{\para}[1]{\paragraph{#1.}}

\newcommand{\R}{\mathbb{R}}

\DeclarePairedDelimiterX{\cbrp}[1]{\lbrace}{\rbrace}{#1}
\DeclarePairedDelimiterX{\brp}[1]{(}{)}{#1}
\DeclarePairedDelimiterX{\sqbrp}[1]{[}{]}{#1}
\DeclarePairedDelimiterX{\normp}[1]{\lVert}{\rVert}{#1}
\DeclarePairedDelimiterX{\absp}[1]{\lvert}{\rvert}{#1}
\DeclarePairedDelimiterX{\ipp}[1]{\langle}{\rangle}{#1}

\newcommand{\cbr}[1]{\ensuremath{\mathchoice{\cbrp*{#1}}{\cbrp{#1}}{\cbrp{#1}}{\cbrp{#1}}}}

\newcommand{\norm}[1]{\ensuremath{\mathchoice{\normp*{#1}}{\normp{#1}}{\normp{#1}}{\normp{#1}}}}

\newcommand{\ip}[1]{\ensuremath{\mathchoice{\ipp*{#1}}{\ipp{#1}}{\ipp{#1}}{\ipp{#1}}}}

\DeclareMathOperator*{\argmax}{arg\,max}
\DeclareMathOperator*{\argmin}{arg\,min}

\newcommand{\E}{\mathbb{E}}

\newcommand{\ind}[1]{\ensuremath{\mathbf{1}\,\cbr{#1}}}
\DeclareMathOperator{\Bern}{Bern}

\DeclareMathOperator{\Unif}{Unif}

\newcommand{\KL}{D_{\mathrm{KL}}}

\renewcommand{\v}[1]{\boldsymbol{#1}}
\newcommand{\m}[1]{\mathbf{#1}}

\DeclareRobustCommand{\term}[1]{%
  \ifmmode
    A_{\text{#1}}%
  \else
    term~\ensuremath{A_{\text{#1}}}%
  \fi
}

\newcommand{\eps}{\epsilon}

\DeclarePairedDelimiterX{\parens}[1]{(}{)}{#1}

\algnewcommand\algorithmicinit{\textbf{initialize}}
\algnewcommand\Init{\item[\algorithmicinit]}
\algnewcommand\algorithmicinput{\textbf{input}}
\algnewcommand\Input{\item[\algorithmicinput]}

\newcommand{\Pp}{\mathbb{P}}
\newcommand{\calC}{\mathcal{C}}
\newcommand{\calB}{\mathcal{B}}
\newcommand{\calD}{\mathcal{D}}
\newcommand{\calE}{\mathcal{E}}
\newcommand{\calF}{\mathcal{F}}
\newcommand{\calG}{\mathcal{G}}
\newcommand{\calI}{\mathcal{I}}
\newcommand{\calJ}{\mathcal{J}}

\newcommand{\calX}{\mathcal{X}}
\newcommand{\calZ}{\mathcal{Z}}
\newcommand{\algname}{WISE}
\renewcommand{\ip}[2]{\ensuremath{\mathchoice{\ipp*{#1,#2}}{\ipp{#1,#2}}{\ipp{#1,#2}}{\ipp{#1,#2}}}}
\renewcommand{\para}[1]{\paragraph{#1}}
\algrenewcommand{\algorithmiccomment}[1]{\hfill\textcolor{blue}{// #1}}
\crefname{assumption}{assumption}{assumptions}
\Crefname{assumption}{Assumption}{Assumptions}

\title{Learning While Scheduling Jobs under Context-Dependent Service Rates: An Anytime Rate-Optimal Algorithm}

\author{
Seoungbin Bae$^{1}$ \quad Dabeen Lee$^{2}$\\[0.3em]
$^{1}$Department of Industrial \& Systems Engineering, KAIST\\
$^{2}$Department of Mathematical Sciences, Seoul National University\\
\texttt{sbbae31@kaist.ac.kr, dabeenl@snu.ac.kr}
}

\date{}

\begin{document}
\maketitle

\begin{abstract}
We study contextual queueing bandits, where a learner schedules jobs while learning unknown service rates modeled by logistic functions of job-server features. Performance is measured by queue length regret, the expected excess queue length at round $t$ relative to an oracle that knows the service rates. Existing decaying-regret guarantees either have a suboptimal decay rate or require a known fixed horizon. They also assume context-wise slack and a strictly positive minimum eigenvalue of the feature covariance. In this paper, we propose WISE (Widest Interval Selection with Elimination), achieving rate-optimal $\widetilde{\mathcal O}(t^{-1/2})$ queue length regret at every sufficiently large time without knowing the horizon. We assume capacity slack, meaning that expected incoming workload under best-server service is below service capacity, and impose no covariance lower bound. Our analysis uses a workload potential measuring the expected service attempts needed by waiting jobs on their best servers. Its drift on nonempty rounds combines a negative term ensured by capacity slack with errors from suboptimal service choices. Then an elliptical potential count bounds how often WISE selects wide confidence intervals, thereby limiting the number of rounds with large service errors. We also sharpen the arrival-rate dependence of an existing lower bound and make its dependence on feature dimension and server count explicit. We prove another lower bound that quantifies the increase in regret as the normalized capacity slack decreases; to our knowledge, this is the first such lower bound for CQB. Simulations show small regret even when context-wise slack fails.
\end{abstract}

\section{Introduction}
\label{sec:introduction}

Queueing systems play an important role in cloud computing \citep{vilaplana2014queuing}, call centers \citep{koole2002queueing}, and services using large language models (LLMs) \citep{lee2024design,mitzenmacher2025queueing}. Scheduling in these systems requires accounting for two characteristics. First, service rates are often difficult to predict a priori and estimated from observed service outcomes. Second, jobs have different contexts, and the service rate of a job depends on both the job itself and the server assigned to it. For example, an LLM's probability of successfully answering a query depends on the query's specifications and the model's capabilities.

Queueing bandits provide a framework for learning unknown service rates while scheduling jobs \citep{krishnasamy2021learning,choudhury2021job,stahlbuhk2021learning}. Scheduling decisions determine both which jobs are served and which service outcomes are observed to estimate these rates. Performance is measured by \emph{queue length regret}, the expected excess queue length relative to an oracle that knows the service rates. A (queue length) regret bound that converges to zero implies that the learner's expected queue length approaches that of the oracle.

Contextual queueing bandits (CQB) extend this framework to incorporate job-specific contexts \citep{bae2026queue}. Earlier work on queueing bandits considers context-independent service rates \citep{krishnasamy2021learning} or a fixed set of queue features \citep{kim2024queueing}. In contrast, CQB allows each job to have its own feature vector. The departure rate is assumed to follow a logistic model of job-server features \citep{bae2026queue}, as in logistic bandits \citep{filippi2010parametric,faury2020improved,jun2021improved}. This allows scheduling decisions to reflect heterogeneous jobs and servers. \citet{bae2026queue} establish the first decaying queue length regret bound for CQB, of order $\widetilde{\mathcal O}(t^{-1/4})$. \citet{bae2026learning} obtain an anytime regret bound of the same order. \citet{bae2026algorithm} improve the rate to $\widetilde{\mathcal O}(t^{-1/2})$, which is optimal up to logarithmic factors, for a fixed horizon $t$ known in advance.

However, the earlier results on CQB have the following three limitations. First, they either have a suboptimal rate \citep{bae2026queue,bae2026learning} or require advance knowledge of the horizon \citep{bae2026queue,bae2026algorithm}. Cloud and LLM services operate over an indefinite horizon and must maintain stable queues throughout operation, motivating an anytime horizon-free algorithm. Second, all three results require a context-wise slack condition \citep{bae2026queue,bae2026learning,bae2026algorithm}: the best service rate exceeds the arrival rate by a uniform positive gap for every context, regardless of how frequently it arrives. A natural relaxation is to require \emph{capacity slack}: the expected incoming workload under best-server service is below the system's service capacity. Third, the results also require that the minimum eigenvalue of the feature covariance is strictly positive. This condition fails when similar jobs have features concentrated in a lower-dimensional subspace.

Motivated by these limitations, we develop an anytime horizon-free algorithm with a rate-optimal regret bound under minimal assumptions. Instead of a context-wise slack condition, we assume capacity slack stated in \Cref{ass:average-workload}. Moreover,  the result does not require a positive lower bound on the minimum eigenvalue of the feature covariance. Our contributions are summarized as follows.
\begin{itemize}
\item We propose WISE (Widest Interval Selection with Elimination), which eliminates job-server pairs that cannot be optimal and selects the widest remaining confidence interval. It achieves anytime horizon-free queue length regret of order $\widetilde{\mathcal O}(t^{-1/2})$ without knowing the normalized capacity slack $\eps$. The guarantee holds under capacity slack without a covariance lower bound, allowing both difficult job contexts and features confined to a lower-dimensional subspace.

\item We refine the $\Omega(\lambda t^{-1/2})$ lower bound of \citet{bae2026algorithm} by proving a round of order $\Omega(\min\{\lambda,\min\{d,\log K\}\sqrt{\lambda/t}\})$. This sharpens the dependence on the arrival rate $\lambda$ and makes the dependence on the feature dimension $d$ and the number of servers $K$ explicit.
\item We prove another lower bound of order $\Omega(\min\{\eps^{-2},\min\{d,\log K\}/(\eps^2\sqrt t)\})$. To our knowledge, this is the first lower bound for CQB that captures the increase in regret as the normalized capacity slack $\eps$ decreases. This quantifies how service errors persisting for $\Theta(\eps^{-2})$ rounds increase the difficulty of learning while scheduling.

\item We compare WISE with seven baselines in simulations with capacity slack, including settings that violate context-wise slack. Varying slow-job service probabilities and the normalized capacity slack shows how these conditions affect queue length regret. WISE attains small regret even when some contexts violate context-wise slack.
\end{itemize}

Our analysis for WISE uses two main ideas. First, we replace queue length with a workload potential based on the expected service attempts needed by waiting jobs on their best servers \citep{gamarnik2012multiclass}. On nonempty rounds, capacity slack gives a negative workload drift term plus an error term from suboptimal service choices, even when queue length itself has positive drift. Second, under valid confidence bounds, WISE's elimination and widest interval selection link large departure rate gaps to wide selected intervals. The elliptical potential count \citep{abbasi2011improved,kim2022improved} limits such selections without a covariance lower bound. Combined with workload drift, this makes long periods without an empty queue unlikely. When a new job arrives after an empty round, WISE serves it using the most uncertain candidate server. We use these observations to bound uncertainty under the arrival distribution, then combine recent empty rounds with confidence interval intersections to control the expected departure rate gap at each sufficiently large time.

\section{Problem Setting}
\label{sec:setting}

\subsection{Queueing System and Queue Length Regret}
We consider a queueing system with contextual jobs, a single queue and $K$ servers indexed by $[K]=\{1,\ldots,K\}$. Let $\calX_t$ be the multiset of waiting job contexts in $\calX$ at the beginning of round $t$, with $Q(t)=|\calX_t|$ and $\calX_1=\varnothing$. In a nonempty round, the agent selects a job $x_t\in\calX_t$ and a server $k_t\in[K]$. The service probability is $p(x,k)=\mu(\ip{\v{\phi}(x,k)}{\v{\theta}^*})$, where $\mu(z)=(1+e^{-z})^{-1}$, $\v{\phi}:\calX\times[K]\to\R^d$ is a known measurable feature map, and $\v{\theta}^*\in\R^d$ is unknown. Let $\calF_t$ contain the history before the service outcome, including the chosen pair and the agent's randomization. The departure indicator satisfies $D(t)\mid\calF_t\sim\Bern(p(x_t,k_t))$ when $Q(t)>0$, and $D(t)=0$ otherwise.

After service, a job arrives with indicator $A(t)\sim\Bern(\lambda)$, $\lambda\in(0,1)$, and context $X_t\sim\calD$. The arrival pairs are independent across rounds, $A(t)$ is independent of $X_t$, and $(A(t),X_t)$ is independent of $\calG_t=\calF_t\vee\sigma(D(t))$. The context is observed when $A(t)=1$. The queue evolves as
\begin{align*}
 \calX_{t+1}=\calX_t\setminus\{x_t:D(t)=1\}\cup\{X_t:A(t)=1\},
 \qquad Q(t+1)=Q(t)+A(t)-D(t).
\end{align*}

Our goal is to minimize queue length regret against an oracle that knows the service probabilities and chooses a job-server pair with maximum service probability $p(x,k)$ in its own queue \citep{krishnasamy2021learning,bae2026queue}. Write $Q^\pi(t)=|\calX_t^\pi|$ for the queue length under policy $\pi$, and $Q^*(t)$ for the oracle queue length, with arrivals shared across policies. The queue length regret is $R_t^\pi=\E[Q^\pi(t)-Q^*(t)]$. We omit the learner's policy superscript and write $\calX_t$, $Q(t)$, and $R_t$. We seek a single policy whose regret vanishes at every sufficiently large time $t$, without being given $t$. For the learner's queue, define $(x_i^*,k_i^*)\in\argmax_{x\in\calX_i,k\in[K]}p(x,k)$ and write $\v{\phi}_i=\v{\phi}(x_i,k_i)$, $\v{\phi}_i^*=\v{\phi}(x_i^*,k_i^*)$, setting both features to zero when $Q(i)=0$. This best pair may differ from the oracle's choice because the two queues may contain different jobs.

\subsection{Weaker Assumptions}
We weaken the two strong conditions used for queue length analyses in previous CQB works \citep{bae2026queue,bae2026learning,bae2026algorithm}; we replace context-wise slack by capacity slack and remove the covariance lower bound. The following assumes standard logistic boundedness and i.i.d.\ contexts \citep{jun2021improved,yu2026practical}.
\begin{assumption}
\label{ass:bounded-logistic}
We have: \textbf{1.} For every $(x,k)\in\calX\times[K]$, $\norm{\v{\phi}(x,k)}_2\le1$. For some known constant $S>0$, $\v{\theta}^*\in\calB\coloneqq \{\v{\theta}\in\R^d:\norm{\v{\theta}}_2\le S\}$. \newline \textbf{2.} Let $\kappa^{-1}=\dot\mu(-S)=\dot\mu(S)$. Also, let $\underline p=\mu(-S)$ and $\overline p=\mu(S)$. \newline \textbf{3.} The contexts of the arriving jobs are sampled i.i.d. from an unknown distribution $\calD$.
\end{assumption}
Then $\underline p\le p(x,k)\le\overline p$ and $\dot\mu(\ip{\v{\phi}(x,k)}{\v{\theta}})\ge1/\kappa$ for $\v{\theta}\in\calB$. Next, we impose a capacity slack condition based on the expected incoming workload. Let the best service probability of a context $x$ be $p^*(x)=\max_{k\in[K]} \mu(\ip{\v{\phi}(x,k)}{\v{\theta}^*})$.
\begin{assumption}[Capacity slack]
\label{ass:average-workload}
The (normalized) capacity slack is positive:
\begin{align*}
  \eps:=1-\lambda\E_{X\sim\calD}[1/p^*(X)]>0.
\end{align*}
\end{assumption}

The quantity $1/p^*(x)$ is the expected number of service attempts needed to complete a job with context $x$ on a best server. Hence, $\rho=\lambda\E_{X\sim\calD}[1/p^*(X)]$ is the \emph{traffic intensity} when jobs are served by their best servers. Let $\mu_{\mathrm{eff}}=(\E_{X\sim\calD}[1/p^*(X)])^{-1}$ be the corresponding effective service rate. Then $\eps=1-\rho=(\mu_{\mathrm{eff}}-\lambda)/\mu_{\mathrm{eff}}$. Thus, \Cref{ass:average-workload} is the contextual form of the classical \emph{load condition} $\rho<1$, equivalently \emph{capacity slack} $\mu_{\mathrm{eff}}-\lambda>0$. We use \emph{normalized capacity slack} for the parameter $\eps$. This condition concerns the system's load under best-server service; stability of the learner's queue depends on its scheduling policy.

\begin{remark}[Remark on \Cref{ass:average-workload}]
\label{rem:average-workload}
The decaying regret bounds in previous works \citep{bae2026queue,bae2026learning,bae2026algorithm} require context-wise slack: there exists $\eps_0>0$ such that $p^*(x)\ge\lambda+\eps_0$ for every $x\in\calX$. Here, $\eps_0$ is a positive lower bound on the service-rate gap $p^*(x)-\lambda$ across contexts. Our condition is  weaker than this since $1-\lambda\E_{X\sim\calD}[1/p^*(X)] \ge 1-\lambda/(\lambda+\eps_0) =\eps_0/(\lambda+\eps_0)>0$. Capacity slack allows contexts with $p^*(x)<\lambda$, so even the best available decision can give positive queue length drift. Our first idea measures the expected service attempts required by waiting jobs. For this workload potential, capacity slack gives a negative drift term independent of the selected context. The remaining term is controlled by the departure rate gap. We elaborate this in \Cref{sec:workload}.
\end{remark}

\begin{remark}[Remark on \Cref{ass:bounded-logistic}-3]
\label{rem:context-distribution}
The previous works assume that the minimum eigenvalue of the context covariance is positive, i.e., $\lambda_{\min}( \E_{X\sim\calD}[ \frac{1}{K}\sum_{k\in[K]} \v{\phi}(X,k)\v{\phi}(X,k)^\top ])>0$. We relax this condition, thereby allowing more general context distributions, e.g., a distribution supported on a low-dimensional subspace. However, without the condition, random exploration does not give the same uniform control of estimation error. Instead, we use the elliptical potential count, which holds for the observed feature sequence, to control workload and then the departure rate gap at a specified time. We elaborate this in \Cref{sec:learning-tools}.
\end{remark}

\section{Challenges and Main Ideas}
\label{sec:challenges}

In this section, we present two main ideas for addressing the challenges in achieving anytime rate-optimal queue length regret under the weaker assumptions: (i) we design new workload potentials that yield negative drift under capacity slack, and (ii) we use the elliptical potential count to control the number of rounds with a large departure rate gap, thereby enabling learning without a covariance lower bound. In \Cref{sec:previous-drift}, we first explain why previous analyses require stronger assumptions. Then, in \Cref{sec:workload,sec:learning-tools}, we show how we address each corresponding difficulty.

\subsection{Negative Drift under Strong Assumptions}
\label{sec:previous-drift}

A key step in previous analyses of decaying queue length regret is to establish negative drift, $\E[Q(t+1)-Q(t)\mid\calF_t]<0$, thus queue length decreases in expectation on these rounds. Define the departure rate gap $g_t=\mu(\ip{\v\phi_t^*}{\v\theta^*}) -\mu(\ip{\v\phi_t}{\v\theta^*})$, with $g_t=0$ when $Q(t)=0$. Let $D^*(t)\mid\calF_t\sim\Bern(\mu(\ip{\v\phi_t^*}{\v\theta^*}))$ be the departure indicator for the best pair in the learner's queue. Then
\begin{align*}
 \E[Q(t+1)-Q(t)\mid\calF_t]
 &=\E[A(t)-D^*(t)\mid\calF_t]
   +\E[D^*(t)-D(t)\mid\calF_t]\\
 &=\underbrace{\lambda-\mu(\ip{\v\phi_t^*}{\v\theta^*})}
      _{\le-\eps_0,\ \text{context-wise slack}}
   +\underbrace{\mu(\ip{\v\phi_t^*}{\v\theta^*})
               -\mu(\ip{\v\phi_t}{\v\theta^*})}
      _{=g_t,\ \text{departure rate gap}}.
\end{align*}
Context-wise slack (see \Cref{rem:average-workload}) makes the first term at most $-\eps_0$. A covariance lower bound (see \Cref{rem:context-distribution}) allows random exploration to reduce uncertainty in every feature direction. Therefore, after sufficient exploration, the learner can reduce the departure rate gap $g_t$ below $\eps_0$ for all subsequent decisions. Combining the two terms gives negative queue length drift.

However, neither conclusion follows under our assumptions. First, under capacity slack, the first term is not guaranteed to be negative, even for the best decision. Second, without a covariance lower bound, random exploration does not give the same uniform control of uncertainty. In the next subsection, we address the first difficulty by studying the drift of a new potential in place of $Q(t)$.

\subsection{Workload Drift under Capacity Slack}
\label{sec:workload}

Our first idea is to define a workload potential that measures the \emph{expected number of service attempts} required by the waiting jobs. A job with context $x$ requires $1/p^*(x)$ service attempts on average on a best server. Define the normalized workload and its update by
\begin{equation*}
 W(t)=\underline p\sum_{x\in\calX_t}\frac{1}{p^*(x)},
 \qquad
 W(t+1)=W(t)+\frac{\underline p A(t)}{p^*(X_t)}
             -\frac{\underline p D(t)}{p^*(x_t)}.
\end{equation*}
The departure term is defined as zero when $Q(t)=0$. Since $\underline p Q(t)\le W(t)$, workload bounds also control queue length. A service attempt on a best server removes expected workload $p^*(x)\,\underline p/p^*(x)=\underline p$ for every context $x$. On a nonempty round,
\begin{align}
 \E[W(t+1)-W(t)\mid\calF_t]
 &=\lambda\underline p\,\E_{X\sim\calD}[1/p^*(X)]
   -\underline p\,p(x_t,k_t)/p^*(x_t)\notag\\
 &=\underbrace{-\eps\underline p}
      _{\text{\Cref{ass:average-workload}: $\eps>0$}}
   +\underbrace{\underline p\bigl(p^*(x_t)-p(x_t,k_t)\bigr)/p^*(x_t)}
                _{\le g_t}.\label{eq:workload-drift-main}
\end{align}
Capacity slack makes the first term negative for every selected job. The second term is at most $g_t$ because $\underline p\le p^*(x_t)\le p(x_t^*,k_t^*)$. Consequently, $g_t<\eps\underline p$ gives negative workload drift. It remains to control $g_t$ without a covariance lower bound.

\subsection{Learning with Elliptical Potential Count}
\label{sec:learning-tools}

Our second idea is to use the \emph{elliptical potential count} to control the departure rate gap $E[g_t]$ at a specified time. The elliptical potential count (\Cref{lem:width-count}) is a consequence of the well-known elliptical potential lemma (Lemma~11 of \citealp{abbasi2011improved}, \Cref{lem:aux-elliptical}), which says that we cannot choose too many feature vectors that are far from, or uncertain with respect to, the previously observed feature vectors. We use the count to bound the number of rounds with a large departure rate gap, and show that learning is possible without a covariance lower bound:

\begin{lemma} \label{lem:departure-rate-gap}
For an absolute constant $C>0$ and every $t$ satisfying \Cref{eq:learning-threshold,eq:gap-threshold},
\begin{align*}
 \E[g_t]\le C\eps^{-5/2}\underline p^{-2}\beta_t\log t
             \sqrt{\lambda d\log(e/\lambda)/t}.
\end{align*}
\end{lemma}
We now explain the main ideas behind the proof of \Cref{lem:departure-rate-gap}.

\para{Coarse learning over an interval.} \Cref{eq:workload-drift-main} gives negative drift when $g_i<\eps\underline p$. But without a covariance lower bound, direct control of $g_i$ for \emph{every round} is hard. Instead, we focus on the number of rounds with a large departure rate gap over an \emph{interval}. Suppose our algorithm gives confidence intervals $[\underline p_i(x,k),\overline p_i(x,k)]$ that contain the true service probabilities $p(x,k)$ with high probability, thereby gives $g_i\le2w_i\le\beta_{i-1}\norm{\v\phi_i}_{\m V_{i-1}^{-1}}$. We call $w_i(x,k)=\overline p_i(x,k)-\underline p_i(x,k)$ the width and write $w_i=w_i(x_i,k_i)$ for the selected pair. Here $\beta_i=\mathcal O(\sqrt{d\log i})$ is the confidence radius, and $\m V_i=\kappa\m I+\sum_{j=1}^i\v\phi_j\v\phi_j^\top$. For an interval of length $\ell=t_2-t_1+1$ and $0<r\le1$, the elliptical potential count gives
\begin{equation}
 \sum_{i=t_1}^{t_2}\ind{w_i>r}
 \le\frac{d\beta_{t_2}^2}{2r^2}\log(1+\ell/(d\kappa)).
 \label{eq:count-main}
\end{equation}
Write $\Delta W(i)=W(i+1)-W(i)$. If the interval consists of nonempty rounds, summing \Cref{eq:workload-drift-main} gives
\begin{align}
 W(t_2+1)-W(t_1)
 &\le-\eps\underline p\,\ell+\sum_{i=t_1}^{t_2}g_i
   +\sum_{i=t_1}^{t_2}
       \bigl(\Delta W(i)-\E[\Delta W(i)\mid\calF_i]\bigr)\notag\\
 &\le-\eps\underline p\,\ell/2
   +\sum_{i=t_1}^{t_2}\ind{g_i>\eps\underline p/2}
   +\sum_{i=t_1}^{t_2}
       \bigl(\Delta W(i)-\E[\Delta W(i)\mid\calF_i]\bigr).
 \label{eq:workload-interval-tools}
\end{align}
The second inequality uses $g_i\le\eps\underline p/2+ \ind{g_i>\eps\underline p/2}$. Since $g_i\le2w_i$, \Cref{eq:count-main} with $r=\eps\underline p/4$ bounds the second term on the right-hand side. The last term is the sum of martingale differences, so the Azuma-Hoeffding inequality (\Cref{lem:aux-azuma}) bounds the last term by $\sqrt{8\ell\log(1/\delta)}$. Therefore, on a sufficiently long nonempty interval, the linear decrease $-\eps\underline p\,\ell/2$ dominates both terms (\Cref{lem:workload}-(ii)). If this decrease exceeds the initial workload $W(t_1)$, this implies that the queue must empty at some round in the interval. This gives an upper bound of order $\mathcal O(d\beta_t^2\log t/(\eps\underline p)^3)$ on the length of nonempty intervals, and a bound on the workload $W(t)$ at a specified time (\Cref{lem:recent-empty}).

\para{Fine learning from arrivals after empty rounds.} The coarse analysis limits the number of rounds with a large departure rate gap, but does not locate those rounds. To control $\E[g_t]$, we now bound the probability that a newly arriving job has uncertainty. Define a candidate server as one that is not ruled out as best for the job. Define
\begin{equation*}
 z_i(r)=\Pp_{X\sim\calD}\left(
 \exists k\in[K]: \text{$k$ is a candidate for $X$ at }i,\quad w_i(X,k)>r
 \right).
\end{equation*}
$z_i(r)$ measures the probability that a newly arriving job at round $i$ has high uncertainty. Unlike previous works, which control the largest uncertainty over all possible contexts with a covariance lower bound, we only control the uncertainty under the arrival distribution $\calD$.

We first connect $z_i(r)$ to the actual number of rounds with selected width greater than $r$, which is also bounded by the elliptical potential count. Assume the intervals are nested on $[t_1,t_2+1]$, remain unchanged on empty rounds, and we select the widest candidate. If $Q(i)=0$ and $A(i)=1$, the new job is alone in round $i+1$. Whenever its context has a candidate wider than $r$, the selected interval is also wider than $r$. Independence of the arrival gives $\Pp(w_{i+1}>r\mid\calF_i)\ge\lambda z_i(r)\ind{Q(i)=0}$. Taking expectations, summing, and applying \Cref{eq:count-main} yields the upper bound of
\begin{align}
 \lambda\E\left[\sum_{i=t_1}^{t_2}z_i(r)\ind{Q(i)=0}\right]
 &\le\E\left[\sum_{i=t_1+1}^{t_2+1}\ind{w_i>r}\right]\le\frac{d\beta_{t_2+1}^2}{2r^2}\log(1+\ell/(d\kappa)).
 \label{eq:empty-count-tools}
\end{align}
Here, what we actually need is an upper bound on $\E[z_{t_2}(r)]$. Therefore, we need to account for nonempty rounds when $Q(i)\neq 0$. On an empty round, we have $\E[\Delta W(i)\mid\calF_i]\le(1-\eps)\underline p$. Hence we extend \Cref{eq:workload-drift-main} to every empty/nonempty round as $\E[\Delta W(i)\mid\calF_i]\le-\eps\underline p+g_i+ \underline p\ind{Q(i)=0}$. Since $0\le g_i \le 1$, we have $g_i\le\eps\underline p/2+\ind{g_i>\eps\underline p/2}$, which gives
\begin{equation*}
 \eps\underline p/2
 \le\underline p\ind{Q(i)=0}-\E[\Delta W(i)\mid\calF_i]
       +\ind{g_i>\eps\underline p/2}.
\end{equation*}
We multiply by $z_i(r)$, sum, and take expectations to use \Cref{eq:empty-count-tools}. Since $z_i(r)$ is known at round $i$, $\E[z_i(r)\E[\Delta W(i)\mid\calF_i]] =\E[z_i(r)\Delta W(i)]$. Therefore, we have
\begin{align*}
\frac{\eps \underline p}{2} \E \left[\sum_{i=t_1}^{t_2} z_i(r) \right] &\leq \underline p\,\E \left[\sum_{i=t_1}^{t_2} z_i(r) \ind{Q(i)=0}\right] - \E \left[\sum_{i=t_1}^{t_2} z_i(r) \Delta W(i) \right] \\
&\quad {}+ \E \left[\sum_{i=t_1}^{t_2} z_i(r) \ind{g_i>\eps\underline p/2} \right].
\end{align*}
For the first and third terms on the right-hand side, we apply \Cref{eq:empty-count-tools} and the elliptical potential count with $z_i(r)\leq 1$, respectively. For the rest of the terms, we assume monotonicity, $z_{i-1}(r) \geq z_i(r)$. We later show that our \emph{interval intersection} procedure in \Cref{alg:wise} can guarantee this property within each epoch. Then, for the second term on the right-hand side,
\begin{align}
 -\sum_{i=t_1}^{t_2}z_i(r)\Delta W(i)
 =z_{t_1}(r)W(t_1)-z_{t_2}(r)W(t_2+1)
 -\sum_{i=t_1+1}^{t_2}(z_{i-1}(r)-z_i(r))W(i)\le W(t_1).
 \label{eq:weighted-workload-tools}
\end{align}
Also, for the left-hand side, $\ell z_{t_2}(r)\le\sum_{i=t_1}^{t_2}z_i(r)$. Then, we have
\begin{align}
 \frac{\eps\underline p\,\ell}{2}\E[z_{t_2}(r)]
 \le\underline p\,\E\left[\sum_{i=t_1}^{t_2}z_i(r)\ind{Q(i)=0}\right]
      +\E[W(t_1)]+\E\left[\sum_{i=t_1}^{t_2}
                  \ind{g_i>\eps\underline p/2}\right].
 \label{eq:fresh-context-tools}
\end{align}
Dividing both sides by $\eps\underline p\,\ell /2$ gives the desired upper bound of $\E[z_{t_2}(r)]$ (\Cref{lem:fresh-context}).

Next, we need to connect the uncertainty of a new job sampled from $\calD$ to that of the currently selected pair. The difficulty is that the job selected at time \(t\) is not an independent sample from \(\mathcal D\), so we cannot write \(\Pr(w_t>r)=\mathbb E[z_t(r)]\). Instead, we return to the job's arrival time to use independence. First, suppose the queue was empty at some round in $[t-\ell,t-1]$. Then every job remaining at time \(t\) must have arrived after that empty round. Therefore, the currently selected job arrived in one of the rounds $t-\ell,\ t-\ell+1,\ \ldots,\ t-1$. The \emph{interval intersection} assumed above ensures that if the current width is greater than $r$, then that server's width was also greater than $r$ just after arrival and it was a candidate at that time. Thus, the job satisfied the event of \(z_i(r)\). Independence of the arrival and its context gives $\Pr\!\left( A(i-1)=1,\ X_{i-1}\text{ has a wide candidate at time }i \right)=\lambda\,\mathbb E[z_i(r)]$. Taking a union bound over the possible arrival times $t-\ell,\ldots,t-1$ gives (see the proof of \Cref{lem:width-transfer})
\begin{equation}
 \Pp(w_t>r)
 \le\lambda\sum_{i=t-\ell+1}^{t}\E[z_i(r)]
 +\Pp\bigl(\text{the job selected at }t
                 \text{ arrived before }t-\ell\bigr).
 \label{eq:arrival-transfer-tools}
\end{equation}
We bound the first term using \Cref{eq:fresh-context-tools} and the second term using the fact that busy periods are short with high probability, as shown in the coarse learning analysis. Finally, using the fact that $\mathbb E[g_t]=\int_0^1\Pr(g_t>r)\,dr$ with $g_t\leq 2w_t$ we obtain the desired result. The assumptions made in this section are guaranteed by \Cref{alg:wise} and elaborated in \Cref{sec:algorithm}.

\section{Algorithm}
\label{sec:algorithm}

We propose WISE (Widest Interval Selection with Elimination), shown in \Cref{alg:wise}. The algorithm estimates service probabilities, intersects their confidence intervals, and selects the widest interval among pairs that have not been eliminated. We estimate the unknown parameter by regularized logistic maximum likelihood,
\begin{align}
 \widehat{\v\theta}_t
 =\argmin_{\v\theta\in\calB}
 \biggl\{\sum_{i=1}^t\left[-D(i)\log\mu(\ip{\v\phi_i}{\v\theta})
 -(1-D(i))\log\bigl(1-\mu(\ip{\v\phi_i}{\v\theta})\bigr)\right]
 +\norm{\v\theta}_2^2/2\biggr\}.
 \label{eq:estimator}
\end{align}
With $\m V_t=\kappa\m I+\sum_{i=1}^t\v\phi_i\v\phi_i^\top$, define $\beta_t$ by
\begin{align}
 \beta_t=(\kappa/2)\sqrt{d\log(1+t/(d\kappa))+4\log(t+1)}
                +\sqrt\kappa S.
 \label{eq:confidence-radius}
\end{align}
\begin{algorithm}[t]
\caption{\algname{} (Widest Interval Selection with Elimination)}
\label{alg:wise}
\begin{algorithmic}[1]
\Input $d,K,S,\v\phi$
\Init $\m V_0\gets\kappa\m I$, $\widehat{\v\theta}_0\gets\v0$, $\beta_0\gets\sqrt\kappa S$, $U_0\sim\Unif(0,1)$, $\calI_0(x,k)\gets[\underline p,\overline p]$
\For{$t=1,2,\ldots$}
 \State Determine $m_t$ and $\tau_{m_t}$ by \Cref{eq:epoch} \Comment{current epoch}
 \If{$\calX_t\ne\varnothing$}
  \For{$(x,k)\in\calX_t\times[K]$}
   \State $\calJ\gets\calI_{\tau_{m_t}-1}(x,k)$
   \For{$i=\tau_{m_t},\ldots,t-1$}
    \If{$\calJ\cap\calI_i(x,k)=\varnothing$}
     \State $\calJ\gets\{(\min\calJ+\max\calJ)/2\}$ and \textbf{break} \Comment{freeze until the next epoch}
    \EndIf
    \State $\calJ\gets\calJ\cap\calI_i(x,k)$ \Comment{intersect confidence intervals}
   \EndFor
   \State $[\underline p_t(x,k),\overline p_t(x,k)]\gets\calJ$
  \EndFor
  \State $\calZ_t\gets\{(x,k)\in\calX_t\times[K]:
    \overline p_t(x,k)\ge\max_{(x',k')\in\calX_t\times[K]}\underline p_t(x',k')\}$ \Comment{elimination}
  \State $(x_t,k_t)\in\argmax_{(x,k)\in\calZ_t}
    [\overline p_t(x,k)-\underline p_t(x,k)]$ \Comment{widest interval selection}
  \State Serve $(x_t,k_t)$ and observe $D(t)$
  \State $\m V_t\gets\m V_{t-1}+\v\phi_t\v\phi_t^\top$, update $\widehat{\v\theta}_t$ by \Cref{eq:estimator}, $\beta_t$ by \Cref{eq:confidence-radius}
 \Else
  \State $\m V_t\gets\m V_{t-1}$, $\widehat{\v\theta}_t\gets\widehat{\v\theta}_{t-1}$, update $\beta_t$ by \Cref{eq:confidence-radius} \Comment{when queue is empty}
 \EndIf
\EndFor
\end{algorithmic}
\end{algorithm}

\para{Intersecting confidence intervals.} Within each epoch, intersecting confidence intervals makes their widths and $z_i(r)$ nonincreasing. This is the property used in \Cref{eq:weighted-workload-tools} and in comparing a job's current width with its width at arrival. From the estimates, form
\begin{equation*}
 \calI_i(x,k)=
 [\mu(\ip{\v\phi(x,k)}{\widehat{\v\theta}_i}
       -\beta_i\norm{\v\phi(x,k)}_{\m V_i^{-1}}),
 \mu(\ip{\v\phi(x,k)}{\widehat{\v\theta}_i}
       +\beta_i\norm{\v\phi(x,k)}_{\m V_i^{-1}})]
 \cap[\underline p,\overline p].
\end{equation*}
We restart the intersection at epoch boundaries to prevent an invalid interval from affecting all later decisions. With $U_0\sim\Unif(0,1)$ drawn once, set
\begin{align}
 \tau_0=1,\qquad
 \tau_m=\lceil2^{m-1+U_0}\rceil,\qquad
 m_t=\max\{m:\tau_m\le t\}.
 \label{eq:epoch}
\end{align}
At an epoch boundary, start from $\calI_{t-1}(x,k)$. Within the epoch, intersect this interval with each new confidence interval. If the intersection becomes empty, set both endpoints to the midpoint of the last nonempty intersection and keep them fixed until the next epoch, preserving nesting. Estimation retains all observations across epochs. The random shift bounds the probability of a reset shortly before any specified time \Cref{eq:new-reset-probability}.

\para{Elimination and widest interval selection.} The candidate set $\calZ_t$ retains pairs whose upper bound is at least the largest lower bound in the queue. Selecting the widest candidate gives $g_t\le2w_t$ (\Cref{lem:confidence}-(ii)), so the elliptical potential count bounds the number of rounds with a large departure rate gap. For a job arriving alone, a candidate wider than $r$ also guarantees an observation of width greater than $r$, giving the relation in \Cref{eq:empty-count-tools}. These are the two properties used in \Cref{sec:learning-tools}. Appendix~\ref{app:related} compares the purpose and guarantees of this rule with previous widest interval selection algorithms.

\section{Main Results}
\label{sec:results}

\subsection{Upper Bound}
\label{sec:upper-analysis}

\begin{theorem}
\label{thm:main-upper}
Under \Cref{ass:bounded-logistic,ass:average-workload}, there exist an absolute constant $C_1>0$ and a threshold $t_0$ satisfying the conditions in \Cref{eq:upper-threshold} such that, for every $t\ge t_0$, the \algname{} algorithm satisfies
\begin{equation*}
 R_t\le C_1\eps^{-9/2}\underline p^{-4}\beta_t\log^2 t
       \sqrt{\lambda d\log(e/\lambda)/t}.
\end{equation*}
\end{theorem}

Since $\beta_t=\mathcal O(\sqrt{d\log t})$, the WISE algorithm attains $\widetilde{\mathcal O}(t^{-1/2})$ queue length regret at every sufficiently large time without knowing $t$.

\begin{proof}[Proof sketch]
Here $C$ may depend on $d,\kappa,S,\underline p$, but not on $\eps,\lambda,t$, and may increase between inequalities. Let $\pi_i$ follow the learner through round $i$ and the oracle afterwards, with queue length $Q^{\pi_i}(j)$ at round $j$. The coupling of \citet{bae2026queue} (\Cref{lem:switching-decomposition}) gives $R_t\le\sum_{i=1}^{t-1}\E[g_iq_{i,t}]$, where $g_i$ is the departure rate gap in \Cref{sec:previous-drift} and, for $g_i>0$,
\begin{equation*}
 q_{i,t}=\Pp\bigl(Q^{\pi_i}(t)-Q^{\pi_{i-1}}(t)=1
 \mid\calF_i,\ Q^{\pi_i}(i+1)-Q^{\pi_{i-1}}(i+1)=1\bigr).
\end{equation*}
Set $q_{i,t}=0$ when $g_i=0$. The conditioning event is a service failure under $\pi_i$ and a success under $\pi_{i-1}$ at round $i$. Thus $q_{i,t}$ measures whether the resulting extra job still affects queue length at $t$. On this event, $\calX_j^{\pi_{i-1}}\subseteq\calX_j^{\pi_i}$ for $j\ge i+1$. Therefore,
\begin{equation*}
 Q^{\pi_i}(j)=0\text{ for some }j\in[i+1,t]
 \quad\Longrightarrow\quad
 Q^{\pi_i}(t)=Q^{\pi_{i-1}}(t).
\end{equation*}
The workload of $Q^{\pi_i}$ at $i+1$ is at most $W(i)+1$. Its oracle decisions give workload drift at most $-\eps/C$ until $Q^{\pi_i}$ empties. If $W(i)+1$ is small compared with the remaining $t-i$ rounds, Azuma-Hoeffding (\Cref{lem:aux-azuma}) bounds the probability that $Q^{\pi_i}(j)>0$ throughout $[i+1,t]$. This gives (\Cref{lem:oracle-forgetting})
\begin{equation}
 \E[q_{i,t}]\le
 \Pp\bigl(W(i)+1>\eps(t-i)/C\bigr)
 +e^{-\eps^2(t-i)/C}.
 \label{eq:sketch-forgetting}
\end{equation}
Since $0\le g_i,q_{i,t}\le1$, we can bound $\E[g_iq_{i,t}]$ by either $\E[q_{i,t}]$ or $\E[g_i]$. We use the former for the first two intervals and the latter for the last. We choose $t-\tau_{1,t}=\Theta(\log^2 t/\eps^4)$ and $t-\tau_{2,t}=\Theta(\log t/\eps^2)$.

\para{First interval ($1\le i\le\tau_{1,t}$).} The workload bound of order $\log^2 i/\eps^3$ suffices here because we choose $\tau_{1,t}$ so that $\eps(t-i)/C\ge C\log^2 i/\eps^3+1$. Thus the probability in \Cref{eq:sketch-forgetting} is at most $\Pp(W(i)>C\log^2 i/\eps^3)$. For $i>\eps t/C$, \Cref{lem:recent-empty}-(ii), using \Cref{eq:workload-interval-tools}, bounds this by $C/(\eps^2i^2)$. For smaller $i$, $W(i)+1\le i$ makes the original probability zero. Summing these probabilities and the exponential terms gives $\sum_{i=1}^{\tau_{1,t}}\E[g_iq_{i,t}]\le C/(\eps^3t)$.

\para{Second interval ($\tau_{1,t}<i\le\tau_{2,t}$).} Now $\eps(t-i)/C$ can be smaller than $C\log^2 i/\eps^3+1$, so \Cref{lem:recent-empty}-(ii) no longer suffices. The learning bounds in \Cref{eq:fresh-context-tools,eq:arrival-transfer-tools} give the tail (\Cref{lem:service-error}-(ii)), valid for every $z\ge0$, $\Pp(W(i)>z)\le Ce^{-\eps z/C}/\eps^2 +C(\eps+\lambda)\log^7i/(\eps^{10}i)$. Applying it with $z=\eps(t-i)/C-1$ bounds the probability in \Cref{eq:sketch-forgetting}. Summing over $O(\log^2 t/\eps^4)$ rounds gives $\sum_{i=\tau_{1,t}+1}^{\tau_{2,t}}\E[g_iq_{i,t}] \le C/(\eps^4t) +C(\eps+\lambda)\log^9t/(\eps^{14}t)$.

\para{Third interval ($\tau_{2,t}<i<t$).} A service error just before $t$ can still affect $Q(t)$, so we bound the departure rate gap directly using $\E[g_iq_{i,t}]\le\E[g_i]$. Applying \Cref{lem:departure-rate-gap} over the remaining $O(\log t/\eps^2)$ rounds gives the leading term in the theorem. The conditions in \Cref{eq:upper-threshold} bound the first two contributions by that term. See \Cref{app:upper} for the full proof and cutoff choices.
\end{proof}

\begin{figure}[!t]
 \centering
 \includegraphics[width=\linewidth]{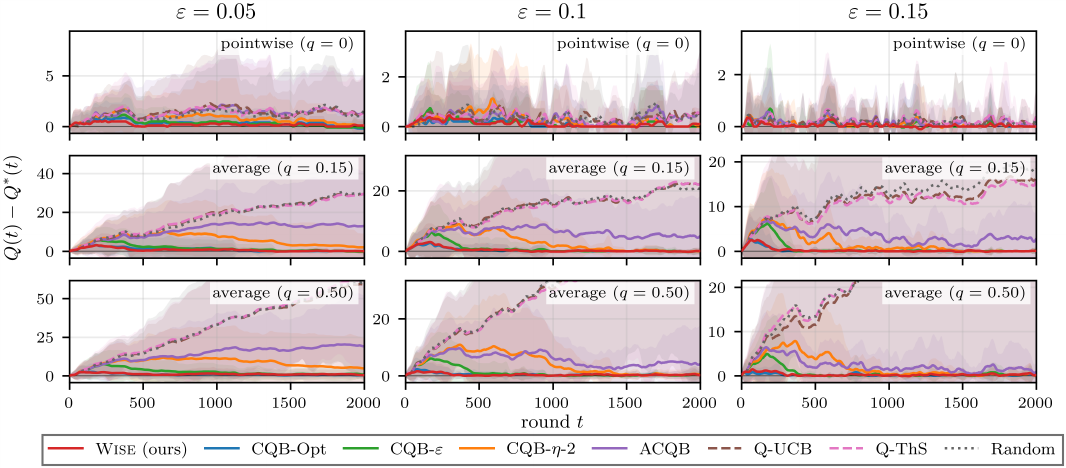}
\caption{Queue length regret averaged over ten runs, with $\pm2$ sample standard deviations and $25$-round smoothing. Columns use normalized capacity slack $\eps=0.05,0.1,0.15$. Rows use pointwise, average $(q=0.15)$, and average $(q=0.50)$ distributions, where $q=\Pp_{X\sim\calD}(p^*(X)<\lambda)$.}
 \label{fig:experiments}
\end{figure}

\subsection{Lower Bounds}
\label{sec:lower-bounds}

We prove lower bounds for two questions. \Cref{thm:time-lower} determines how fast regret can decay with $t$ when the normalized capacity slack is bounded away from zero. \Cref{thm:heavy-traffic-lower} shows how a small normalized capacity slack increases regret.

For \Cref{thm:time-lower}, we sharpen the $\lambda$ dependence in \citet{bae2026algorithm} by counting service observations. In our construction, jobs require $O(1)$ expected service attempts, giving $O(\lambda t)$ observations in expectation. Service probabilities separated by order $\min\{1,1/\sqrt{\lambda t}\}$ can therefore remain hard to distinguish. An independent arrival before the last decision forces service with probability $\lambda$, giving a regret lower bound of order $\lambda\min\{1,1/\sqrt{\lambda t}\} =\min\{\lambda,\sqrt{\lambda/t}\}$. See \Cref{app:lower-bound} for the full construction, including $d$ and $K$.
\begin{theorem}
\label{thm:time-lower}
There is an absolute constant $C_0\ge1$ such that for every $d\ge2$, $K\ge2$, $\lambda\in(0,1/4]$, and $t\ge3$, there is a finite family $\mathcal V$ of problem instances with $S=1/4$ and normalized capacity slack $\eps=1-2\lambda\ge1/2$ for which every algorithm $\pi$ satisfies
\begin{equation*}
 \sup_{\nu\in\mathcal V}R_t^\pi(\nu)
 \ge C_0^{-1}\min\{\lambda,\min\{d,\log K\}\sqrt{\lambda/t}\}.
\end{equation*}
\end{theorem}

Together with \Cref{thm:main-upper}, this establishes rate-optimality in $t$ up to logarithmic factors.

For \Cref{thm:heavy-traffic-lower}, a small normalized capacity slack lets service errors persist for $\Theta(\eps^{-2})$ rounds. In our construction, $\lambda=1/4$ and the best service probability is $p^*=\lambda/(1-\eps)$. Since $p^*-\lambda=\lambda\eps/(1-\eps)=\Theta(\eps)$, a queue with $\Theta(\eps^{-1})$ jobs takes $\Theta(\eps^{-2})$ rounds to empty in expectation. Take $q_0=\Theta(\eps^{-1})$ and $\ell=\Theta(\eps^{-2})$ as in \Cref{lem:heavy-traffic-queue}. With constant probability, $Q^\pi(i)\ge q_0$. A service error on this event leaves $Q^{\pi_i}(i+1)=Q^{\pi_{i-1}}(i+1)+1$, and this extra job remains for $\ell$ rounds with probability at least $1/2$ even under the best server. Consequently, $R_t^\pi(\nu)\ge(1/2)\sum_{i=t-\ell}^{t-1} \E^\nu[g_i\ind{Q^\pi(i)\ge q_0}]$. Averaging each expectation over the constructed instances gives a lower bound of order $\min\{1,\min\{d,\log K\}/\sqrt t\}$. Summing over $\ell=\Theta(\eps^{-2})$ rounds gives the $\eps^{-2}$ factor. See \Cref{app:heavy-traffic-lower} for the proof.
\begin{theorem}
\label{thm:heavy-traffic-lower}
There is an absolute constant $C_0\ge1$ such that for every $d,K\ge2$, $0<\eps\le1/8$, and $t\ge64\eps^{-2}$, there is a finite family $\mathcal V_\eps$ of problem instances with $\lambda=1/4$, $S=2$, and normalized capacity slack $\eps$ for which every algorithm $\pi$ satisfies
\begin{equation*}
 \sup_{\nu\in\mathcal V_\eps}R_t^\pi(\nu)
 \ge C_0^{-1}\eps^{-2}\min\{1,\min\{d,\log K\}/\sqrt t\}.
\end{equation*}
\end{theorem}

\section{Experiments}
\label{sec:experiments}

We evaluate \algname{} with and without context-wise slack. We compare it with implementations of CQB-Opt and CQB-$\eps$ \citep{bae2026queue}, CQB-$\eta$-2 \citep{bae2026algorithm}, ACQB \citep{bae2026learning}, Q-UCB and Q-ThS \citep{krishnasamy2021learning}, and uniformly random selection.

We use $d=5$, $K=10$, $\lambda=0.6$, $\norm{\v\phi(x,k)}_2=5$, and one shared parameter with $\norm{\v\theta^*}_2=S=10$. All distributions have the same normalized capacity slack within each column. The pointwise distribution satisfies context-wise slack, whereas the two average distributions increase both the frequency and the size of violations. All contextual algorithms use constrained regularized logistic estimates, a fixed confidence or sampling radius $\beta=100$, and an initial design matrix $\m V_0=5\m I_5$, with input $\kappa=100$. We run each setting 10 times for $2000$ rounds. WISE has the smallest mean queue length regret over rounds $1$-$2000$ in five of the nine settings and the second-smallest in the remaining four. Its average regret over rounds $1500$-$2000$ is below $1$ in all nine settings, including those without context-wise slack. Implementation and distribution details are in \Cref{app:experiments}.

\label{main-content-end}
\bibliographystyle{plainnat}
\bibliography{ref}
\clearpage
\appendix
\section{Additional Related Work}
\label{app:related}

\para{Queueing systems and learning.} Queueing systems model service allocation in call centers and computing platforms \citep{koole2002queueing,vilaplana2014queuing}. Queueing bandits address scheduling when service rates are unknown and must be learned from service outcomes \citep{krishnasamy2021learning,choudhury2021job,stahlbuhk2021learning}. Their guarantees include queue length regret at a given time \citep{krishnasamy2021learning}, cumulative queue length or holding-cost regret \citep{stahlbuhk2021learning,krishnasamy2018learning}, and the maximum increase in time-averaged queue length due to learning \citep{freund2023quantifying}. The $\widetilde{\mathcal O}(t^{-1})$ guarantee of \citet{krishnasamy2021learning} depends on a fixed difference between the best and second-best service rates, whereas our lower bounds allow this difference to decrease with $t$. Contextual extensions assign a fixed feature to each queue \citep{kim2024queueing} or allow heterogeneous contexts for individual jobs \citep{bae2026queue,bae2026learning,bae2026algorithm}. For the latter setting, existing bounds decay as $\widetilde{\mathcal O}(t^{-1/4})$ without horizon knowledge \citep{bae2026learning} or as $\widetilde{\mathcal O}(t^{-1/2})$ with a known horizon \citep{bae2026algorithm}. These results require context-wise slack and a covariance lower bound. We replace context-wise slack by capacity slack and remove the covariance lower bound. Other learning approaches for queueing systems study MaxWeight scheduling and stability under unknown service rates \citep{neely2012max,stahlbuhk2019scheduling,yang2023learning,wijewardena2025bandit}, scheduling with predicted service times or job classes \citep{mitzenmacher2019scheduling,scully2021uniform,lee2024design,mitzenmacher2025queueing}, and learning when to defer decisions to congested human reviewers \citep{lykouris2024learning}.

\para{Logistic bandits.} Our service model connects CQB to logistic bandits, which model binary feedback through a logistic function of action features and an unknown parameter. Following the generalized linear bandit formulation of \citet{filippi2010parametric}, subsequent work improves the dependence on curvature and derives instance-dependent regret bounds \citep{faury2020improved,abeille2021instance,lee2024unified}. Related directions include sharper confidence bounds for finite-action problems \citep{li2017provably,jun2021improved}, computationally efficient updates \citep{faury2022jointly,zhang2023online}, and Thompson sampling \citep{dong2019performance,kim2023double}. \citet{bae2026logistic} remove the assumption of a positive lower bound on the minimum eigenvalue of the context covariance matrix. These works study cumulative reward regret, whereas our analysis requires a bound on the expected departure rate gap at each specified time.

\para{Elimination algorithms.} Elimination algorithms use confidence bounds to discard actions that cannot be optimal. In Auer's framework, exploration selects a remaining action whose confidence width exceeds a prescribed threshold \citep{auer2002using,chu2011contextual,li2017provably,jun2021improved}. Related methods for Gaussian process optimization select the action with the largest posterior variance, either in combination with UCB \citep{contal2013parallel} or after eliminating suboptimal actions \citep{cai2021lenient,li2022gaussian,iwazaki2025improved,ma2026batched}. Confidence interval intersections also appear in level-set estimation and safe optimization \citep{gotovos2013active,sui2015safe}. The BLCE algorithm of \citet{yu2026practical} explores by selecting $\v x_i\in\argmax_{\v x\in\calC_i}\norm{\v x}_{\m H_{i-1}^{-1}}$ among surviving candidates $\calC_i$, where $\m H_i=\m I+\sum_{j=1}^i\v x_j\v x_j^\top$. For $\ell$ exploration rounds with $\norm{\v x}_2\le1$ and a fixed elimination rule, their analysis gives $\E[\max_{\v x\in\calC_{\ell+1}} \norm{\v x}_{\m H_\ell^{-1}}^2] \le\ell^{-1}\sum_{i=1}^\ell \E[\norm{\v x_i}_{\m H_{i-1}^{-1}}^2] \le(2d/\ell)\log(1+\ell/d)$, where $\calC_{\ell+1}$ comes from a fresh i.i.d.\ context set. The first inequality uses context independence and the selection rule, and the second follows from the elliptical potential lemma. This analysis does not directly apply to CQB because the set of waiting jobs depends on past service decisions and departures, so the first inequality cannot be justified by context independence. We instead combine observations after empty rounds with workload bounds and interval intersections to control $\E[g_t]$ (\Cref{sec:learning-tools}).

\section{Proof of \texorpdfstring{\Cref{thm:main-upper}}{Theorem \getrefnumber{thm:main-upper}}}
\label{app:upper}

We first prove the lemmas used to bound $\E[g_iq_{i,t}]$ and then prove \Cref{thm:main-upper} in \Cref{app:upper-proof}. Throughout this section, $C_4$ denotes a positive absolute constant and may increase between inequalities. Set $C_2=64/(\eps\underline p)^3$, $C_3=128/(\eps\underline p)^2$, and
\begin{equation}
 \ell_{0,t}=\lceil C_2d\beta_t^2\log t\rceil,
 \qquad \ell_{1,t}=\lceil C_3\log t\rceil.
 \label{eq:proof-lengths}
\end{equation}
These lengths are used only in the proof. We use the conditions
\begin{equation}
 t\ge3,\qquad
 t\ge\frac{C_4d\beta_t^2\log t}{(\eps\underline p)^3},
 \label{eq:learning-threshold}
\end{equation}
with a sufficiently large absolute $C_4$.

\subsection{Proofs of \texorpdfstring{Lemmas~\ref{lem:confidence} and~\ref{lem:width-count}}{Lemmas \getrefnumber{lem:confidence} and \getrefnumber{lem:width-count}}}
\label{app:new-confidence}

We use the filtration $\calF_t$ in \Cref{sec:setting}, including $U_0$. For $t\ge0$, define
\begin{equation*}
 \calC_t=\{\v\theta\in\calB:
       \norm{\v\theta-\widehat{\v\theta}_t}_{\m V_t}
          \le\beta_t\}.
\end{equation*}
Recall that $w_t(x,k)=\overline p_t(x,k)-\underline p_t(x,k)$. Write $w_t=w_t(x_t,k_t)$ when $Q(t)>0$ and $w_t=0$ otherwise. \Cref{lem:confidence} gives confidence bounds and relates $g_t$ to $w_t$. \Cref{lem:width-count} bounds the number of rounds with $w_t>r$.

\begin{lemma}
\label{lem:confidence}
\leavevmode
\begin{enumerate}
\renewcommand{\labelenumi}{(\roman{enumi})}
\item For every $t_1\ge1$, conditional on $U_0$,
\begin{equation}
 \Pp\!\left(\exists t\ge t_1:\v\theta^*\notin\calC_t\mid U_0\right)
 \le(t_1+1)^{-2}.
 \label{eq:new-confidence-probability}
\end{equation}
\item If $\v\theta^*\in\calC_i$ for every $i=\tau_{m_t}-1,\ldots,t-1$, then
\begin{equation*}
 \underline p_t(x,k)\le p(x,k)\le\overline p_t(x,k)
 \quad\text{for all }(x,k)\in\calX\times[K],
 \qquad g_t\le2w_t.
\end{equation*}
\end{enumerate}
\end{lemma}

\subsubsection{Proof of \texorpdfstring{\Cref{lem:confidence}}{Lemma \getrefnumber{lem:confidence}}}

\begin{proof}
We first prove part (i). Fix $U_0$. Define
\begin{equation*}
 \eta_i=\bigl(D(i)-\mu(\ip{\v\phi_i}{\v\theta^*})\bigr)\ind{Q(i)>0},
 \qquad
 \v G_t=\sum_{i=1}^t\eta_i\v\phi_i.
\end{equation*}
\Cref{lem:aux-mle} gives
\begin{equation*}
 \norm{\widehat{\v\theta}_t-\v\theta^*}_{\m V_t}
 \le\kappa\norm{\v G_t}_{\m V_t^{-1}}+\sqrt\kappa S.
\end{equation*}

By the conditional Bernoulli service law and Hoeffding's lemma (\Cref{lem:aux-hoeffding}), $\E[\eta_i\mid\calF_i]=0$ and $\E[e^{z\eta_i}\mid\calF_i]\le e^{z^2/8}$ for every $z\in\R$. Theorem~1 of \citet{abbasi2011improved} (\Cref{lem:aux-self-normalized}), with initial matrix $\kappa\m I$ and noise parameter $1/2$, gives, with probability at least $1-\delta$, simultaneously for all $t\ge1$,
\begin{align}
 \norm{\v G_t}_{\m V_t^{-1}}
 &\le\frac12\sqrt{\log\frac{\det\m V_t}{\det(\kappa\m I)}
                         +2\log(1/\delta)}
 \le\frac12\sqrt{d\log(1+t/(d\kappa))+2\log(1/\delta)}.
 \label{eq:new-score-bound}
\end{align}
The second inequality follows from the arithmetic-geometric mean inequality and $\sum_{i=1}^t\norm{\v\phi_i}_2^2\le t$.

Set $\delta=(t_1+1)^{-2}$. Applying \Cref{eq:new-score-bound} gives, with probability at least $1-\delta$, simultaneously for every $i\ge t_1$,
\begin{align*}
 \norm{\widehat{\v\theta}_i-\v\theta^*}_{\m V_i}
 &\le\frac\kappa2\sqrt{d\log(1+i/(d\kappa))+4\log(t_1+1)}
            +\sqrt\kappa S
 \le\beta_i.
\end{align*}
The last inequality follows from $i\ge t_1$ and \Cref{eq:confidence-radius}. Thus $\v\theta^*\in\calC_i$ for every $i\ge t_1$, proving part (i).

We next prove part (ii), first showing that $p(x,k)\in[\underline p_t(x,k),\overline p_t(x,k)]$. If $\v\theta^*\in\calC_i$ and $i\ge1$, then
\begin{equation*}
 |\ip{\v\phi(x,k)}{\widehat{\v\theta}_i-\v\theta^*}|
 \le\norm{\v\phi(x,k)}_{\m V_i^{-1}}
          \norm{\widehat{\v\theta}_i-\v\theta^*}_{\m V_i}
 \le\beta_i\norm{\v\phi(x,k)}_{\m V_i^{-1}}.
\end{equation*}
The first inequality follows from Cauchy-Schwarz, and the second follows from the definition of $\calC_i$. By the definition of $\calI_i(x,k)$, we have $p(x,k)\in\calI_i(x,k)$. Also $p(x,k)\in\calI_0(x,k)=[\underline p,\overline p]$. Therefore,
\begin{equation*}
 p(x,k)\in\bigcap_{i=\tau_{m_t}-1}^{t-1}\calI_i(x,k)
          =[\underline p_t(x,k),\overline p_t(x,k)].
\end{equation*}
In particular, none of these intersections is empty.

Suppose $Q(t)>0$. To prove $g_t\le2w_t$, we first show that $(x_t^*,k_t^*)\in\calZ_t$. We have
\begin{equation*}
 \overline p_t(x_t^*,k_t^*)
 \ge p(x_t^*,k_t^*)
 \ge\max_{(x,k)\in\calX_t\times[K]}\underline p_t(x,k).
\end{equation*}
Thus $(x_t^*,k_t^*)\in\calZ_t$, and
\begin{align*}
 g_t
 &=p(x_t^*,k_t^*)-p(x_t,k_t)\\
 &\le\overline p_t(x_t^*,k_t^*)-\underline p_t(x_t,k_t)\\
 &=w_t(x_t^*,k_t^*)
       +\underline p_t(x_t^*,k_t^*)-\underline p_t(x_t,k_t)\\
 &\le w_t(x_t^*,k_t^*)
       +\overline p_t(x_t,k_t)-\underline p_t(x_t,k_t)
 \le2w_t,
\end{align*}
where the first inequality follows from $p(x_t^*,k_t^*)\le\overline p_t(x_t^*,k_t^*)$ and $p(x_t,k_t)\ge\underline p_t(x_t,k_t)$, the second follows from $\underline p_t(x_t^*,k_t^*)\le\overline p_t(x_t,k_t)$ since $(x_t,k_t)\in\calZ_t$, and the last follows from $w_t(x_t^*,k_t^*)\le w_t$ by \Cref{alg:wise}. For $Q(t)=0$, $g_t=w_t=0$.
\end{proof}

\begin{lemma}
\label{lem:width-count}
For $1\le t_1\le t_2$, $0<r\le1$, and $\ell=t_2-t_1+1$,
\begin{equation*}
 \sum_{i=t_1}^{t_2}\ind{w_i>r}
 \le\frac{d\beta_{t_2}^2}{2r^2}\log(1+\ell/(d\kappa))
\end{equation*}
on every sample path.
\end{lemma}

\subsubsection{Proof of \texorpdfstring{\Cref{lem:width-count}}{Lemma \getrefnumber{lem:width-count}}}

\begin{proof}
For $i\ge2$, \Cref{alg:wise} gives
\begin{equation*}
 [\underline p_i(x,k),\overline p_i(x,k)]
 \subseteq\calI_{i-1}(x,k)
 \qquad\text{if }w_i(x,k)>0.
\end{equation*}
If an intersection is empty, \Cref{alg:wise} sets $\underline p_i(x,k)=\overline p_i(x,k)$, so $w_i(x,k)=0$. For $w_i(x,k)>0$, $\dot\mu\le1/4$ gives
\begin{align}
 w_i(x,k)
 &\le\mu(\ip{\v\phi(x,k)}{\widehat{\v\theta}_{i-1}}
          +\beta_{i-1}\norm{\v\phi(x,k)}_{\m V_{i-1}^{-1}})\notag\\
 &\quad{}-\mu(\ip{\v\phi(x,k)}{\widehat{\v\theta}_{i-1}}
          -\beta_{i-1}\norm{\v\phi(x,k)}_{\m V_{i-1}^{-1}})
           \notag\\
 &\le\frac{\beta_{i-1}}{2}
                \norm{\v\phi(x,k)}_{\m V_{i-1}^{-1}}.
 \label{eq:new-width-bound}
\end{align}
The same bound holds if $w_i(x,k)=0$.

Since $\m V_{i-1}\succeq\kappa\m I+ \sum_{j=t_1}^{i-1}\v\phi_j\v\phi_j^\top$ and $\norm{\v\phi_i}_2\le1$, Lemma~11 of \citet{abbasi2011improved} (\Cref{lem:aux-elliptical}), applied to $\v\phi_{t_1},\ldots,\v\phi_{t_2}$ with initial matrix $\kappa\m I$, gives
\begin{equation}
 \sum_{i=t_1}^{t_2}\norm{\v\phi_i}_{\m V_{i-1}^{-1}}^2
 \le2d\log(1+\ell/(d\kappa)).
 \label{eq:new-elliptical-sum}
\end{equation}
Combining \Cref{eq:new-width-bound,eq:new-elliptical-sum} gives
\begin{align*}
 \sum_{i=t_1}^{t_2}\ind{w_i>r}
 &\le\frac1{r^2}\sum_{i=t_1}^{t_2}w_i^2\\
 &\le\frac{\beta_{t_2}^2}{4r^2}
       \sum_{i=t_1}^{t_2}
                  \norm{\v\phi_i}_{\m V_{i-1}^{-1}}^2\\
 &\le\frac{d\beta_{t_2}^2}{2r^2}\log(1+\ell/(d\kappa)),
\end{align*}
where the second inequality uses $\beta_{i-1}\le\beta_{t_2}$.
\end{proof}

\subsection{Proofs of \texorpdfstring{Lemmas~\ref{lem:workload} and~\ref{lem:recent-empty}}{Lemmas \getrefnumber{lem:workload} and \getrefnumber{lem:recent-empty}}}
\label{app:new-workload}

\Cref{lem:workload} bounds the change in workload on a nonempty interval. We use this bound in \Cref{lem:recent-empty} to show that $Q(i)=0$ for some $i\in[t_3-\ell_{0,t},t_3]$ with high probability.

\begin{lemma}
\label{lem:workload}
\leavevmode
\begin{enumerate}
\renewcommand{\labelenumi}{(\roman{enumi})}
\item Under any policy, $|W(i+1)-W(i)|\le1$, and
\begin{equation*}
 \E[W(i+1)-W(i)\mid\calF_i]\le-\eps\underline p+g_i
 \qquad\text{if }Q(i)>0.
\end{equation*}
\item Under \Cref{alg:wise}, fix $1\le t_1\le t_2$, $\ell=t_2-t_1+1$, and $\delta\in(0,1)$. With conditional probability at least $1-\delta$ given $\calF_{t_1}$, if $Q(i)>0$ and the waiting intervals contain their true service probabilities for every $i\in[t_1,t_2]$, then
\begin{equation*}
 W(t_2+1)-W(t_1)
 \le-\eps\underline p\,\ell/2
       +\frac{8d\beta_{t_2}^2}{\eps^2\underline p^2}\log(1+\ell/(d\kappa))
       +\sqrt{8\ell\log(1/\delta)}.
\end{equation*}
\end{enumerate}
\end{lemma}

\subsubsection{Proof of \texorpdfstring{\Cref{lem:workload}}{Lemma \getrefnumber{lem:workload}}}

\begin{proof}
We first prove part (i). Write $\Delta W(i)=W(i+1)-W(i)$. By the definition of $W(i)$,
\begin{equation*}
 \Delta W(i)=\frac{\underline p A(i)}{p^*(X_i)}
                    -\frac{\underline p D(i)}{p^*(x_i)},
 \qquad |\Delta W(i)|\le1,
\end{equation*}
where the departure term is zero when $Q(i)=0$. The inequality follows because each term belongs to $[0,1]$. When $Q(i)>0$,
\begin{align*}
 \E[\Delta W(i)\mid\calF_i]
 &=\lambda\underline p\,\E_{X\sim\calD}[1/p^*(X)]
             -\underline p\,p(x_i,k_i)/p^*(x_i)\\
 &=-\eps\underline p
        +\frac{\underline p}{p^*(x_i)}
                   \bigl(p^*(x_i)-p(x_i,k_i)\bigr)\\
 &\le-\eps\underline p+g_i\\
 &\le-\eps\underline p/2+\ind{g_i>\eps\underline p/2}.
\end{align*}
The first equality follows from the arrival and service laws, and the second follows from \Cref{ass:average-workload}. The first inequality uses $\underline p\le p^*(x_i)\le p(x_i^*,k_i^*)$, and the second uses $0\le g_i\le1$.

We next prove part (ii). We write
\begin{equation*}
 W(t_2+1)-W(t_1)
 =\sum_{i=t_1}^{t_2}\E[\Delta W(i)\mid\calF_i]
  +\sum_{i=t_1}^{t_2}
       \bigl(\Delta W(i)-\E[\Delta W(i)\mid\calF_i]\bigr).
\end{equation*}
We first bound the first sum. The proof of \Cref{lem:confidence}-(ii) gives $g_i\le2w_i$ whenever the intervals contain their true service probabilities. Therefore, on the event in part (ii),
\begin{align}
 \sum_{i=t_1}^{t_2}\E[\Delta W(i)\mid\calF_i]
 &\le-\eps\underline p\,\ell/2
          +\sum_{i=t_1}^{t_2}\ind{w_i>\eps\underline p/4}\notag\\
 &\le-\eps\underline p\,\ell/2
          +\frac{8d\beta_{t_2}^2}{\eps^2\underline p^2}\log(1+\ell/(d\kappa)),
 \label{eq:new-summed-drift}
\end{align}
where the second inequality follows from \Cref{lem:width-count} with $r=\eps\underline p/4$. For the second sum, since $|\Delta W(i)-\E[\Delta W(i)\mid\calF_i]|\le2$, the Azuma-Hoeffding inequality (\Cref{lem:aux-azuma}) gives
\begin{equation*}
 \Pp\!\left(
 \sum_{i=t_1}^{t_2}\bigl(\Delta W(i)-\E[\Delta W(i)\mid\calF_i]\bigr)
       >\sqrt{8\ell\log(1/\delta)}
          \,\middle|\,\calF_{t_1}\right)\le\delta.
\end{equation*}
Combining this probability bound with \Cref{eq:new-summed-drift} proves part (ii).
\end{proof}

\begin{lemma}
\label{lem:recent-empty}
Let $\ell_{0,t}$ be defined in \Cref{eq:proof-lengths}. If $t$ satisfies \Cref{eq:learning-threshold}, every $t_3\in[t/3,t]$ satisfies
\begin{enumerate}
\renewcommand{\labelenumi}{(\roman{enumi})}
\item $\displaystyle \Pp(Q(i)>0\text{ for all }i\in[t_3-\ell_{0,t},t_3]\mid U_0) \le \frac{C_4}{\eps^2\underline p^2t^2}.$
\item $\displaystyle \Pp(W(t_3)>\ell_{0,t}\mid U_0)\le \frac{C_4}{\eps^2\underline p^2t^2}.$
\item $\displaystyle \E[W(t_3)\mid U_0]\le\ell_{0,t}+\frac{C_4}{\eps^2\underline p^2t}.$
\end{enumerate}
The bounds hold uniformly in $U_0$.
\end{lemma}

\subsubsection{Proof of \texorpdfstring{\Cref{lem:recent-empty}}{Lemma \getrefnumber{lem:recent-empty}}}

\begin{proof}
We first prove part (i). Fix $U_0$, and set
\begin{equation*}
 t_1=\lfloor\eps\underline p\,t_3/64\rfloor,
 \qquad t_2=\lfloor t_3/2\rfloor,
 \qquad
 \calE_{0,t}=\bigcap_{j=\lfloor\eps\underline p\,t/1024\rfloor}^{t}
                     \{\v\theta^*\in\calC_j\}.
\end{equation*}
We bound $\Pp(\calE_{0,t}^{\mathrm c}\mid U_0)$ separately. On $\calE_{0,t}$, we consider two cases depending on whether the queue is empty at some round in $[t_1,t_2]$. If it is, we bound the probability that the queue is nonempty for $\ell_{0,t}$ consecutive rounds after an empty round. Taking $C_4\ge1536$ in \Cref{eq:learning-threshold} gives $t\ge C_4$ and
\begin{equation}
 \ell_{0,t}
 \le\frac{64d\beta_t^2\log t}{(\eps\underline p)^3}+1
 \le64t/C_4+1
 \le t/24+1
 \le t/12.
 \label{eq:recent-empty-length-upper}
\end{equation}
To see why these two bounds suffice, suppose that $Q(i)>0$ for every $i\in[t_3-\ell_{0,t},t_3]$. By \Cref{eq:recent-empty-length-upper}, $t_2+\ell_{0,t}\le t_3/2+t/12\le3t_3/4\le t_3$. Either $Q(i)>0$ for all $i\in[t_1,t_2]$, or an empty round occurs in $[t_1,t_2]$. In the latter case, the last empty round $j=\max\{i\in[t_1,t_3]:Q(i)=0\}$ satisfies $j\in[t_1,t_3-\ell_{0,t}]$ and $Q(i)>0$ for all $i\in[j+1,j+\ell_{0,t}]$. The union bound therefore gives
\begin{align}
 &\Pp(Q(i)>0\text{ for all }i\in[t_3-\ell_{0,t},t_3]\mid U_0)\notag\\
 &\quad\le\underbrace{\Pp(\calE_{0,t}^{\mathrm c}\mid U_0)}_{I_1}
       +\underbrace{\Pp(\calE_{0,t}\cap
            \{Q(i)>0\text{ for all }i\in[t_1,t_2]\}\mid U_0)}_{I_2}\notag\\
 &\qquad{}+\underbrace{\sum_{j=t_1}^{t_3-\ell_{0,t}}
       \Pp\!\left(\calE_{0,t}\cap
       \{Q(j)=0,\ Q(i)>0\text{ for all }i\in[j+1,j+\ell_{0,t}]\}
       \mid U_0\right)}_{I_3}.
 \label{eq:recent-empty-decomposition}
\end{align}
We bound $I_1,I_2,I_3$ in turn.

For $I_1$, \Cref{lem:confidence}-(i) gives
\begin{equation}
 I_1
 \le(\lfloor\eps\underline p\,t/1024\rfloor+1)^{-2}
 \le\frac{1024^2}{\eps^2\underline p^2t^2}.
 \label{eq:new-recent-confidence}
\end{equation}

To bound $I_2$, we argue by contradiction. On the event defining $I_2$, the queue is nonempty throughout $[t_1,t_2]$. We apply \Cref{lem:workload}-(ii) on $[t_1,t_2-1]$ with $W(t_1)\le t_1$ and show that its upper bound on $W(t_2)$ is negative. Thus $I_2$ is bounded by the probability that the bound in \Cref{lem:workload}-(ii) fails. We first verify its confidence condition. By \Cref{lem:confidence}-(ii), it suffices to show that $\tau_{m_i}-1,\ldots,i-1$ lie in the range defining $\calE_{0,t}$. We verify this for every $i\in[t_1,t]$, so that the same confidence condition also holds on the intervals used to bound $I_3$. We have
\begin{equation*}
 \tau_{m_i}-1\ge i/2-1
 \ge\eps\underline p\,t_3/128-3/2
 \ge\eps\underline p\,t/384-3/2
 \ge\eps\underline p\,t/1024.
\end{equation*}
The first inequality follows from $i<\tau_{m_i+1}\le2\tau_{m_i}$ by \Cref{eq:epoch}, the second from $i\ge t_1\ge\eps\underline p\,t_3/64-1$, the third from $t_3\ge t/3$, and the last follows by taking $C_4\ge1024$, since \Cref{eq:learning-threshold} gives $\eps\underline p\,t\ge C_4$. Consequently,
\begin{equation*}
 [\tau_{m_i}-1,i-1]
 \subseteq[\lfloor\eps\underline p\,t/1024\rfloor,t]
 \qquad\text{for every }i\in[t_1,t].
\end{equation*}
On $\calE_{0,t}$, we therefore have $\v\theta^*\in\calC_j$ for every $j=\tau_{m_i}-1,\ldots,i-1$. Applying \Cref{lem:confidence}-(ii) gives
\begin{equation*}
 \underline p_i(x,k)\le p(x,k)\le\overline p_i(x,k)
 \qquad\text{for all }(x,k)\in\calX\times[K],\ i\in[t_1,t].
\end{equation*}
Thus the confidence bounds required by \Cref{lem:workload}-(ii) hold for both applications below.

Since $t_3\ge t/3\ge10$, we have
\begin{align}
 \frac{8d\beta_t^2}{\eps^2\underline p^2}\log t
 &\le\frac{8\eps\underline p\,t}{C_4}
 \le\frac{24\eps\underline p\,t_3}{C_4}
 \le\eps\underline p\,t_3/64,
 \label{eq:first-empty-count}\\
 t_2-t_1
 &\ge t_3/2-1-\eps\underline p\,t_3/64
 \ge31t_3/64-1
 \ge3t_3/8,\qquad t_2-t_1\le t_3,
 \label{eq:first-empty-length}
\end{align}
where we use \Cref{eq:learning-threshold}, $t\le3t_3$, $\eps\underline p\le1$, and the definitions of $t_1,t_2$.

Applying \Cref{lem:workload}-(ii) to $[t_1,t_2-1]$ with $\delta=\exp(-\eps^2\underline p^2t_3/512)$ gives, on $\calE_{0,t}\cap\{Q(i)>0\text{ for all }i\in[t_1,t_2]\}$,
\begin{align*}
 W(t_2)
 &\le W(t_1)-\frac{\eps\underline p}{2}(t_2-t_1)
       +\frac{8d\beta_t^2}{\eps^2\underline p^2}\log t
       +\frac{\eps\underline p}{8}\sqrt{(t_2-t_1)t_3}\\
 &\le\frac{\eps\underline p\,t_3}{64}
       -\frac{3\eps\underline p\,t_3}{16}
       +\frac{\eps\underline p\,t_3}{64}
       +\frac{\eps\underline p\,t_3}{8}\\
 &=-\frac{\eps\underline p\,t_3}{32}<0,
\end{align*}
except on an event of probability at most $\delta$ conditional on $\calF_{t_1}$, where the second inequality follows from $W(t_1)\le t_1\le\eps\underline p\,t_3/64$ and \Cref{eq:first-empty-count,eq:first-empty-length}. This contradicts $W(t_2)\ge0$. Thus the event defining $I_2$ has conditional probability at most $\delta$ given $\calF_{t_1}$. Taking conditional expectations given $U_0$, we obtain
\begin{align}
 I_2
 &=\E\!\left[
   \Pp(\calE_{0,t}\cap\{Q(i)>0\text{ for all }i\in[t_1,t_2]\}
       \mid\calF_{t_1})\,\middle|\,U_0\right]\notag\\
 &\le e^{-\eps^2\underline p^2t_3/512}\le t^{-2}.
 \label{eq:new-first-empty}
\end{align}
The last inequality follows from
\begin{equation*}
 \frac{\eps^2\underline p^2t_3}{512}
 \ge\frac{\eps^2\underline p^2t}{1536}
 \ge\frac{C_4d\beta_t^2\log t}{1536\eps\underline p}
 \ge2\log t,
\end{equation*}
where we use $t_3\ge t/3$, \Cref{eq:learning-threshold}, $d\beta_t^2\ge16\log t$, and $C_4\ge1536$.

To bound $I_3$, fix $j\in[t_1,t_3-\ell_{0,t}]$. If $Q(j)=0$, then $W(j+1)\le1$. We again argue by contradiction using \Cref{lem:workload}-(ii), this time on $[j+1,j+\ell_{0,t}-1]$. We show that its upper bound on $W(j+\ell_{0,t})$ is negative when $\delta=t^{-3}$. Hence each term is at most $t^{-3}$, and summing over $j$ gives $I_3\le t^{-2}$. For this application of \Cref{lem:workload}-(ii), we use
\begin{align}
 \frac{8d\beta_t^2}{\eps^2\underline p^2}
                  \log(1+\ell_{0,t}/(d\kappa))
 &\le\frac{8d\beta_t^2}{\eps^2\underline p^2}\log t
 \le\frac{\eps\underline p\,\ell_{0,t}}8,
 \label{eq:recent-empty-count}\\
 \ell_{0,t}
 &\ge\frac{64d\beta_t^2\log t}{(\eps\underline p)^3}
 \ge\frac{1024\log^2 t}{(\eps\underline p)^3}
 \ge\frac{384\log t}{\eps^2\underline p^2}
 \ge\frac{16}{\eps\underline p}.
 \label{eq:recent-empty-length-lower}
\end{align}
The first inequality in \Cref{eq:recent-empty-count} uses \Cref{eq:recent-empty-length-upper} and $d\kappa\ge1$, and the second follows from \Cref{eq:proof-lengths}. In \Cref{eq:recent-empty-length-lower}, the first inequality follows from the definition of $\ell_{0,t}$, and the second uses $d\beta_t^2\ge16\log t$ by \Cref{eq:confidence-radius} and $\kappa\ge4$. The last two inequalities use $t\ge3$ and $\eps\underline p\le1$.

Applying \Cref{lem:workload}-(ii) to $[j+1,j+\ell_{0,t}-1]$ with $\delta=t^{-3}$ gives, on $\calE_{0,t}\cap\{Q(j)=0,\ Q(i)>0\text{ for all }i\in[j+1,j+\ell_{0,t}]\}$,
\begin{align*}
 W(j+\ell_{0,t})
 &\le W(j+1)-\frac{\eps\underline p}{2}(\ell_{0,t}-1)
       +\frac{8d\beta_t^2}{\eps^2\underline p^2}
                    \log(1+\ell_{0,t}/(d\kappa))\\
 &\quad{}+\sqrt{24(\ell_{0,t}-1)\log t}\\
 &\le1-\frac{\eps\underline p}{2}(\ell_{0,t}-1)
          +\frac{\eps\underline p\,\ell_{0,t}}8
          +\frac{\eps\underline p\,\ell_{0,t}}4\\
 &=1+\eps\underline p/2-\eps\underline p\,\ell_{0,t}/8\\
 &\le-1/2<0,
\end{align*}
except on an event of probability at most $t^{-3}$ conditional on $\calF_{j+1}$. The second and last inequalities follow from $W(j+1)\le1$ on $\{Q(j)=0\}$, \Cref{eq:recent-empty-count,eq:recent-empty-length-lower}, and $\eps\underline p\le1$. This contradicts $W(j+\ell_{0,t})\ge0$. Thus the event in each term of $I_3$ has conditional probability at most $t^{-3}$ given $\calF_{j+1}$. Taking conditional expectations given $U_0$, we obtain
\begin{align}
 &\Pp\!\left(\calE_{0,t}\cap
   \{Q(j)=0,\ Q(i)>0\text{ for all }i\in[j+1,j+\ell_{0,t}]\}
       \mid U_0\right)\notag\\
 &\quad=\E\!\left[
   \Pp\!\left(\calE_{0,t}\cap
     \{Q(j)=0,\ Q(i)>0\text{ for all }i\in[j+1,j+\ell_{0,t}]\}
       \mid\calF_{j+1}\right)\,\middle|\,U_0\right]\notag\\
 &\quad\le t^{-3}.
 \label{eq:new-long-busy}
\end{align}

Summing \Cref{eq:new-long-busy} over $j$ gives
\begin{equation}
 I_3\le(t_3-\ell_{0,t}-t_1+1)t^{-3}\le t^{-2},
 \label{eq:recent-empty-third-term}
\end{equation}
Substituting \Cref{eq:new-recent-confidence,eq:new-first-empty,eq:recent-empty-third-term} into \Cref{eq:recent-empty-decomposition}, we obtain
\begin{align*}
 &\Pp(Q(i)>0\text{ for all }i\in[t_3-\ell_{0,t},t_3]\mid U_0)\\
 &\quad\le\frac{1024^2}{\eps^2\underline p^2t^2}+2t^{-2}
 \le\frac{C_4}{\eps^2\underline p^2t^2},
\end{align*}
where the last inequality follows by taking $C_4\ge1024^2+2$. This proves part (i).

For part (ii), if $Q(j)=0$ for some $j\in[t_3-\ell_{0,t},t_3]$, then
\begin{equation*}
 W(t_3)=W(j)+\sum_{i=j}^{t_3-1}\Delta W(i)
       \le t_3-j\le\ell_{0,t},
\end{equation*}
where the first inequality follows from $W(j)=0$ and \Cref{lem:workload}-(i). Therefore,
\begin{align*}
 &\Pp(W(t_3)>\ell_{0,t}\mid U_0)\\
 &\quad\le\Pp(Q(i)>0\text{ for all }i\in[t_3-\ell_{0,t},t_3]\mid U_0)\\
 &\quad\le\frac{C_4}{\eps^2\underline p^2t^2},
\end{align*}
where the last inequality follows from part (i).

For part (iii), we have
\begin{equation*}
 \E[W(t_3)\mid U_0]
 \le\ell_{0,t}+t\Pp(W(t_3)>\ell_{0,t}\mid U_0)
 \le\ell_{0,t}+\frac{C_4}{\eps^2\underline p^2t},
\end{equation*}
where the first inequality uses $0\le W(t_3)\le t$, and the second follows from part (ii).
\end{proof}

\subsection{Proofs of \texorpdfstring{Lemmas~\ref{lem:fresh-context} to~\ref{lem:service-error}}{Lemmas \getrefnumber{lem:fresh-context}, \getrefnumber{lem:width-transfer}, and \getrefnumber{lem:service-error}}}
\label{app:new-learning}

Recall that $\Delta W(i)=W(i+1)-W(i)$ and $\calG_i=\calF_i\vee\sigma(D(i))$. The arrival $(A(i),X_i)$ is independent of $\calG_i$. For $r>0$, define
\begin{equation}
 z_i(r)=\Pp_{X\sim\calD}\!\left(
 \exists k\in[K]:\ \overline p_i(X,k)\ge\max_{k'}\underline p_i(X,k'),
 \quad w_i(X,k)>r\right).
 \label{eq:new-context-mass}
\end{equation}

\begin{lemma}
\label{lem:fresh-context}
Under \Cref{alg:wise}, fix $U_0$, $1\le t_1\le t_2$, $\ell=t_2-t_1+1$, and $0<r\le1$. Suppose the nonempty, $\calF_i$-measurable intervals satisfy
\begin{equation*}
 [\underline p_{i+1}(x,k),\overline p_{i+1}(x,k)]
 \subseteq[\underline p_i(x,k),\overline p_i(x,k)]
 \quad\text{for all }(x,k)\in\calX\times[K],\ i\in[t_1,t_2],
\end{equation*}
and remain unchanged when $Q(i)=0$. Then
\begin{align*}
 \E[z_{t_2}(r)\mid U_0]
 &\le\frac2{\eps\underline p\,\ell}
 \left(\frac{\underline p\,d\beta_{t_2+1}^2}{2\lambda r^2}\log(1+\ell/(d\kappa))
       +\E[W(t_1)\mid U_0]\right.\\
 &\hphantom{\le\frac2{\eps\underline p\,\ell}}\quad\left.{}
       +\E\!\left[\sum_{i=t_1}^{t_2}\ind{g_i>\eps\underline p/2}\,\middle|\,U_0\right]
 \right).
\end{align*}
\end{lemma}

\subsubsection{Proof of \texorpdfstring{\Cref{lem:fresh-context}}{Lemma \getrefnumber{lem:fresh-context}}}

\begin{proof}
Fix $U_0$. We first express $\E[z_{t_2}(r)\mid U_0]$ in terms of three quantities and then bound them separately. The variables $z_i(r)$ are $\calF_i$-measurable. If $x,k$ satisfy the inequalities in \Cref{eq:new-context-mass} at $i+1$, then
\begin{align*}
 \overline p_i(x,k)
 &\ge\overline p_{i+1}(x,k)
 \ge\max_{k'}\underline p_{i+1}(x,k')
 \ge\max_{k'}\underline p_i(x,k'),\\
 w_i(x,k)&\ge w_{i+1}(x,k)>r.
\end{align*}
Thus $z_{i+1}(r)\le z_i(r)$ and $\ell z_{t_2}(r)\le\sum_{i=t_1}^{t_2}z_i(r)$.

By \Cref{lem:workload}-(i) and $\E[\Delta W(i)\mid\calF_i]=(1-\eps)\underline p$ when $Q(i)=0$,
\begin{align*}
 \E[\Delta W(i)\mid\calF_i]
 &\le-\eps\underline p+g_i+\underline p\ind{Q(i)=0}\\
 &\le-\eps\underline p/2
        +\underline p\ind{Q(i)=0}+\ind{g_i>\eps\underline p/2}.
\end{align*}
The second inequality uses $g_i\le\eps\underline p/2+\ind{g_i>\eps\underline p/2}$. Multiplying by $z_i(r)$, summing, and taking expectations conditional on $U_0$ gives
\begin{align}
 \frac{\eps\underline p\,\ell}{2}\E[z_{t_2}(r)\mid U_0]
 &\le\frac{\eps\underline p}{2}
          \E\!\left[\sum_{i=t_1}^{t_2}z_i(r)\,\middle|\,U_0\right]\notag\\
 &\le\underbrace{\underline p\,\E\!\left[
                  \sum_{i=t_1}^{t_2}z_i(r)\ind{Q(i)=0}\,\middle|\,U_0\right]}_{I_1}
       +\underbrace{\E\!\left[
                  -\sum_{i=t_1}^{t_2}z_i(r)\Delta W(i)\,\middle|\,U_0\right]}_{I_2}
          \notag\\
 &\quad{}+\underbrace{\E\!\left[\sum_{i=t_1}^{t_2}
                  z_i(r)\ind{g_i>\eps\underline p/2}\,\middle|\,U_0\right]}_{I_3}.
 \label{eq:new-weighted-mass}
\end{align}
Here we use the tower rule
\begin{equation*}
 \E[z_i(r)\Delta W(i)\mid U_0]
 =\E\!\left[z_i(r)\E[\Delta W(i)\mid\calF_i]\,\middle|\,U_0\right],
\end{equation*}
which holds because $z_i(r)$ is $\calF_i$-measurable and $\sigma(U_0)\subseteq\calF_i$. We bound $I_1,I_2,I_3$ in turn.

To bound $I_1$, we count arrivals after empty rounds that have a candidate server with width greater than $r$. Their expected number conditional on $U_0$ is $\lambda I_1/\underline p$. We show that each such arrival gives $w_{i+1}>r$, so that \Cref{lem:width-count} bounds $I_1$.

On $\{Q(i)=0,A(i)=1\}$, we have $\calX_{i+1}=\{X_i\}$. The intervals at rounds $i$ and $i+1$ are equal by assumption. Thus \Cref{alg:wise} gives
\begin{equation*}
 w_{i+1}
 =\max_{\substack{k\in[K]:\\
       \overline p_i(X_i,k)\ge\max_{k'}\underline p_i(X_i,k')}}
       w_i(X_i,k).
\end{equation*}
Consequently,
\begin{align*}
 &\E[\ind{w_{i+1}>r}\mid\calF_i]\\
 &\quad\ge\Pp\!\left(
   \begin{gathered}
    Q(i)=0,\ A(i)=1,\\
    \exists k\in[K]:\
       \overline p_i(X_i,k)\ge\max_{k'}\underline p_i(X_i,k'),\
       w_i(X_i,k)>r
   \end{gathered}\,\middle|\,\calF_i\right)\\
 &\quad=\lambda\ind{Q(i)=0}\,
   \Pp_{X\sim\calD}\!\left(
    \exists k\in[K]:\
       \overline p_i(X,k)\ge\max_{k'}\underline p_i(X,k'),\
       w_i(X,k)>r\right)\\
 &\quad=\lambda z_i(r)\ind{Q(i)=0}.
\end{align*}
The first equality follows from the independence of $A(i)$, $X_i$, and $\calF_i$, and the second follows from \Cref{eq:new-context-mass}. Taking expectations conditional on $U_0$ and summing gives
\begin{align}
 I_1
 &=\underline p\sum_{i=t_1}^{t_2}
       \E[z_i(r)\ind{Q(i)=0}\mid U_0]\notag\\
 &\le\frac{\underline p}{\lambda}\sum_{i=t_1}^{t_2}
       \E\!\left[\E[\ind{w_{i+1}>r}\mid\calF_i]\,\middle|\,U_0\right]\notag\\
 &=\frac{\underline p}{\lambda}
       \E\!\left[\sum_{i=t_1+1}^{t_2+1}\ind{w_i>r}\,\middle|\,U_0\right]\notag\\
 &\le\frac{\underline p\,d\beta_{t_2+1}^2}{2\lambda r^2}
                     \log(1+\ell/(d\kappa)),
 \label{eq:new-empty-observation}
\end{align}
where the second equality follows from the tower rule and a change of summation index, and the last inequality follows from \Cref{lem:width-count}.

To bound $I_2$, summation by parts gives
\begin{align*}
 I_2
 &=\E[z_{t_1}(r)W(t_1)\mid U_0]-\E[z_{t_2}(r)W(t_2+1)\mid U_0]\\
 &\quad{}-\E\!\left[
       \sum_{i=t_1+1}^{t_2}(z_{i-1}(r)-z_i(r))W(i)\,\middle|\,U_0\right]\\
 &\le\E[W(t_1)\mid U_0],
\end{align*}
where the inequality follows from $0\le z_i(r)\le1$, $z_{i-1}(r)\ge z_i(r)$, and $W(i)\ge0$. For $I_3$, $z_i(r)\le1$ gives $I_3\le\E[\sum_{i=t_1}^{t_2}\ind{g_i>\eps\underline p/2}\mid U_0]$. Substituting \Cref{eq:new-empty-observation} and the bounds on $I_2,I_3$ into \Cref{eq:new-weighted-mass} and dividing by $\eps\underline p\,\ell/2$ proves the result. Taking expectations over $U_0$ gives the same bound without conditioning.
\end{proof}

For use in the next two proofs, define
\begin{equation*}
 \calE_{1,i}
 =\bigcap_{j=\tau_{m_i}-1}^{i-1}\{\v\theta^*\in\calC_j\}.
\end{equation*}

\begin{lemma}
\label{lem:width-transfer}
For every $t$ satisfying \Cref{eq:learning-threshold}, every $0<r\le1$, every $2\le\ell\le t/8$, and every event $\calE$,
\begin{align*}
 \Pp(w_t>r,\calE)
 &\le\frac{C_4\ell\,d\beta_t^2\log^2 t}t
       \left[\frac1{\eps r^2}+\frac\lambda{(\eps\underline p)^4}\right]\\
 &\quad{}+\Pp\bigl(\calE,\ Q(t)>0,\
        \text{the job selected at }t\text{ arrived before }t-\ell\bigr).
\end{align*}
\end{lemma}

\subsubsection{Proof of \texorpdfstring{\Cref{lem:width-transfer}}{Lemma \getrefnumber{lem:width-transfer}}}

\begin{proof}
We bound $\Pp(w_t>r,\calE)$ by comparing the width of $(x_t,k_t)$ at round $t$ with its width just after $x_t$ arrived. To use the interval intersections, both rounds must be in the same epoch. We therefore separate jobs arriving before $t-\ell$ and the event $t-\tau_{m_t}+1\le2\ell$. In the remaining case, every possible arrival round $t-\ell,\ldots,t-1$ is in the current epoch. Thus
\begin{align}
 \Pp(w_t>r,\calE)
 &\le\underbrace{\Pp(t-\tau_{m_t}+1\le2\ell)}_{I_1}\notag\\
 &\quad{}+\underbrace{\Pp\bigl(w_t>r,\
       \text{the selected job arrived at or after }t-\ell,\
       t-\tau_{m_t}+1>2\ell\bigr)}_{I_2}\notag\\
 &\quad{}+\Pp\bigl(\calE,\ Q(t)>0,\
       \text{the selected job arrived before }t-\ell\bigr).
 \label{eq:width-transfer-decomposition}
\end{align}
We bound $I_1$ and $I_2$ in turn.

The term $I_1$ is the probability that the current epoch contains at most $2\ell$ rounds up to $t$. We will show that $I_1\le4\ell/(t\log2)$ using the uniform shift $U_0$. To do so, we first express the number of rounds $t-\tau_{m_t}+1$ in terms of $U_0$. By \Cref{eq:epoch}, $m_t=1+\lfloor\log_2t-U_0\rfloor$. Define $U=\log_2t-(m_t-1+U_0)\in[0,1)$, so that
\begin{equation*}
 \tau_{m_t}=\lceil2^{m_t-1+U_0}\rceil=\lceil t2^{-U}\rceil.
\end{equation*}
To remove the ceiling, set $Y=1-2^{-U}$. Then $tY=t-2^{m_t-1+U_0}$ is the elapsed time from the epoch start before rounding, and
\begin{equation}
 tY<t-\tau_{m_t}+1\le tY+1.
 \label{eq:new-reset-age}
\end{equation}
Thus $I_1\le\Pp(Y\le2\ell/t)$. We compute this probability by first finding the distribution of $U$. Since $U=\log_2t-U_0-\lfloor\log_2t-U_0\rfloor$ and $U_0\sim\Unif(0,1)$, for $0\le u\le1$,
\begin{align*}
 \Pp(U\le u)
 &=\int_{\log_2t-1}^{\log_2t}\ind{z-\lfloor z\rfloor\le u}\,dz\\
 &=\int_{\log_2t-\lfloor\log_2t\rfloor}^{1}\ind{z\le u}\,dz
   +\int_0^{\log_2t-\lfloor\log_2t\rfloor}\ind{z\le u}\,dz=u,
\end{align*}
where we split the first integral at $\lfloor\log_2t\rfloor$ and subtract the integer part on each interval. Hence $U\sim\Unif(0,1)$. For $0\le y\le1/2$,
\begin{equation*}
 \Pp(Y\le y)
 =\Pp(U\le-\log_2(1-y))
 =-\frac{\log(1-y)}{\log2}.
\end{equation*}
Hence $Y$ has density $f_Y(y)=1/((1-y)\log2)\le2/\log2$ on $[0,1/2]$, and
\begin{align}
 I_1
 &\le\Pp(Y\le2\ell/t)
 =\int_0^{2\ell/t}f_Y(y)\,dy
 \le\frac{4\ell}{t\log2}.
 \label{eq:new-reset-probability}
\end{align}
The first inequality follows from \Cref{eq:new-reset-age}, and the last uses $f_Y(y)\le2/\log2$ and $2\ell/t\le1/4$.

The term $I_2$ is the probability that $w_t>r$, the selected job arrived at or after $t-\ell$, and the current epoch contains more than $2\ell$ rounds up to $t$. To bound it, we show that, just after $x_t$ arrived, server $k_t$ already had width greater than $r$ and an upper bound at least as large as every lower bound for $x_t$. We can then sum over the possible arrival rounds and apply \Cref{lem:fresh-context} to bound $\E[z_i(r)\mid U_0]$.

Suppose the event in $I_2$ holds, and let $i-1$ be the arrival round of $x_t$. Then $A(i-1)=1$, $X_{i-1}=x_t$, and
\begin{equation*}
 \tau_{m_t}<t-2\ell+1<t-\ell+1\le i\le t.
\end{equation*}
Thus $i$ and $t$ are in the same epoch. For the pair $(x_t,k_t)$ selected at $t$, we have
\begin{align*}
 \overline p_i(x_t,k_t)
 &\ge\overline p_t(x_t,k_t)
 \ge\max_{k'}\underline p_t(x_t,k')
 \ge\max_{k'}\underline p_i(x_t,k'),\\
 w_i(x_t,k_t)&\ge w_t>r.
\end{align*}
The first and third inequalities follow from $[\underline p_t(x,k),\overline p_t(x,k)] \subseteq[\underline p_i(x,k),\overline p_i(x,k)]$, and the second follows from $(x_t,k_t)\in\calZ_t$. Substituting $X_{i-1}=x_t$ gives
\begin{equation*}
 A(i-1)=1,\qquad
 \overline p_i(X_{i-1},k_t)\ge\max_{k'}\underline p_i(X_{i-1},k'),
 \qquad w_i(X_{i-1},k_t)>r.
\end{equation*}
These inequalities hold for at least one $i\in[t-\ell+1,t]$ whenever the event in $I_2$ occurs. The union bound therefore gives
\begin{align}
 I_2
 &\le\sum_{i=t-\ell+1}^t
 \Pp\!\left(
 \begin{gathered}
  A(i-1)=1,\ t-\tau_{m_t}+1>2\ell,\\
  \exists k\in[K]:\
   \overline p_i(X_{i-1},k)\ge\max_{k'}\underline p_i(X_{i-1},k'),\
   w_i(X_{i-1},k)>r
 \end{gathered}\right)\notag\\
 &=\lambda\sum_{i=t-\ell+1}^t
       \E\!\left[z_i(r)\ind{t-\tau_{m_t}+1>2\ell}\right]\notag\\
 &=\lambda\sum_{i=t-\ell+1}^t
       \E_{U_0}\!\left[\E[z_i(r)\mid U_0]
                  \ind{t-\tau_{m_t}+1>2\ell}\right].
 \label{eq:new-arrival-probability}
\end{align}
The first equality follows from independence of $(A(i-1),X_{i-1})$ and $\calG_{i-1}$, since the bounds $\underline p_i(x,k)$, $\overline p_i(x,k)$ and $U_0$ are $\calG_{i-1}$-measurable. The second equality follows from the tower rule, since $\ind{t-\tau_{m_t}+1>2\ell}$ depends only on $U_0$.

We now bound $\E[z_i(r)\mid U_0]$ using \Cref{lem:fresh-context}. Fix $U_0$, set $m=m_t$, and let $\tau_m<i\le t$. We will use this lemma on $[\tau_m,i-1]$, so we need to bound $\E[\sum_{j=\tau_m}^{i-1}\ind{g_j>\eps\underline p/2}\mid U_0]$ and $\E[W(\tau_m)\mid U_0]$. We first verify its two interval conditions. Within epoch $m$, \Cref{alg:wise} gives $[\underline p_{j+1}(x,k),\overline p_{j+1}(x,k)] \subseteq[\underline p_j(x,k),\overline p_j(x,k)]$ for every $(x,k)$ and $j\in[\tau_m,i-1]$. When $Q(j)=0$, $\widehat{\v\theta}_j$ and $\m V_j$ are unchanged, while $\beta_j\ge\beta_{j-1}$. Thus $\calI_j(x,k)\supseteq\calI_{j-1}(x,k)$, and the intersection is unchanged.

By \Cref{lem:confidence}-(ii),
\begin{equation*}
 \ind{g_j>\eps\underline p/2}
 \le\ind{w_j>\eps\underline p/4}+\ind{\calE_{1,j}^{\mathrm c}}.
\end{equation*}
Summing over $j=\tau_m,\ldots,i-1$ and taking expectations conditional on $U_0$ gives
\begin{align}
 &\E\!\left[\sum_{j=\tau_m}^{i-1}\ind{g_j>\eps\underline p/2}
                  \,\middle|\,U_0\right]\notag\\
 &\quad\le\E\!\left[\sum_{j=\tau_m}^{i-1}\ind{w_j>\eps\underline p/4}
                  \,\middle|\,U_0\right]
       +\sum_{j=\tau_m}^{i-1}\Pp(\calE_{1,j}^{\mathrm c}\mid U_0)\notag\\
 &\quad\le\frac{8d\beta_t^2}{\eps^2\underline p^2}
                    \log(1+(i-\tau_m)/(d\kappa))
       +(i-\tau_m)\Pp(\exists j'\ge\tau_m-1:
                    \v\theta^*\notin\calC_{j'}\mid U_0)\notag\\
 &\quad\le\frac{8d\beta_t^2}{\eps^2\underline p^2}
                    \log(1+(i-\tau_m)/(d\kappa))
       +\frac{i-\tau_m}{\tau_m^2}\notag\\
 &\quad\le\frac{8d\beta_t^2\log t}{\eps^2\underline p^2}
       +\frac2t
 \le\frac{C_4d\beta_t^2\log t}{\eps^2\underline p^2}.
 \label{eq:new-epoch-gap-count}
\end{align}
The second inequality follows from \Cref{lem:width-count} with $r=\eps\underline p/4$, $\beta_{i-1}\le\beta_t$, and
\begin{equation*}
 \calE_{1,j}^{\mathrm c}
 =\{\exists j'\in[\tau_m-1,j-1]:\v\theta^*\notin\calC_{j'}\}
 \subseteq\{\exists j'\ge\tau_m-1:\v\theta^*\notin\calC_{j'}\},
\end{equation*}
since every $j\in[\tau_m,i-1]$ is in epoch $m$. The third inequality follows from \Cref{lem:confidence}-(i) with $t_1=\tau_m-1$. The fourth uses $\tau_m>t/2$, $i\le t$, and $d\kappa\ge1$. The last uses $d\beta_t^2\log t\ge1$ and $\eps\underline p\le1$. Since $t/2<\tau_m\le t$, \Cref{lem:recent-empty}-(iii) gives
\begin{equation}
 \E[W(\tau_m)\mid U_0]
 \le\ell_{0,t}+\frac{C_4}{\eps^2\underline p^2t}
 \le\frac{C_4d\beta_t^2\log t}{(\eps\underline p)^3}.
 \label{eq:new-epoch-workload}
\end{equation}
Since $U_0$ is fixed, $\tau_m$ is fixed, and we obtain
\begin{align}
 \E[z_i(r)\mid U_0]
 &\le\E[z_{i-1}(r)\mid U_0]\notag\\
 &\le\frac2{\eps\underline p(i-\tau_m)}
   \left(\frac{\underline p\,d\beta_i^2}{2\lambda r^2}\log(1+(i-\tau_m)/(d\kappa))
        +\E[W(\tau_m)\mid U_0]\right.\notag\\
 &\hphantom{\le\frac2{\eps\underline p(i-\tau_m)}}\quad\left.{}
        +\E\!\left[\sum_{j=\tau_m}^{i-1}\ind{g_j>\eps\underline p/2}
                       \,\middle|\,U_0\right]\right)\notag\\
 &\le\frac2{\eps\underline p(i-\tau_m)}
   \left(\frac{\underline p\,d\beta_t^2\log t}{2\lambda r^2}
        +\frac{C_4d\beta_t^2\log t}{(\eps\underline p)^3}
        +\frac{C_4d\beta_t^2\log t}{(\eps\underline p)^2}\right)\notag\\
 &=\frac{d\beta_t^2\log t}{i-\tau_m}
   \left(\frac1{\eps\lambda r^2}
        +\frac{2C_4}{(\eps\underline p)^4}
        +\frac{2C_4}{(\eps\underline p)^3}\right)\notag\\
 &\le\frac{C_4d\beta_t^2\log t}{i-\tau_m}
       \left[\frac1{\eps\lambda r^2}
                   +\frac1{(\eps\underline p)^4}\right],
 \label{eq:new-epoch-mass}
\end{align}
where the first inequality follows from $z_i(r)\le z_{i-1}(r)$, and the second from \Cref{lem:fresh-context} on $[\tau_m,i-1]$. The third inequality follows from \Cref{eq:new-epoch-gap-count,eq:new-epoch-workload}, $\beta_i\le\beta_t$, and $\log(1+(i-\tau_m)/(d\kappa))\le\log t$. The last inequality uses $\eps\underline p\le1$, increasing $C_4$.

We next sum over the possible arrival rounds in \Cref{eq:new-arrival-probability}. On $\{t-\tau_{m_t}+1>2\ell\}$, every $i\in[t-\ell+1,t]$ satisfies
\begin{equation*}
 i-\tau_{m_t}
 \ge t-\tau_{m_t}+1-\ell
 >\frac{t-\tau_{m_t}+1}{2}.
\end{equation*}
Thus \Cref{eq:new-epoch-mass} bounds all terms in the sum using the same denominator $t-\tau_{m_t}+1$. We obtain
\begin{align}
 I_2
 &\le\lambda\sum_{i=t-\ell+1}^{t}
       \E_{U_0}\!\left[\E[z_i(r)\mid U_0]
                  \ind{t-\tau_{m_t}+1>2\ell}\right]\notag\\
 &\le C_4\lambda d\beta_t^2\log t
       \left(\frac1{\eps\lambda r^2}+\frac1{(\eps\underline p)^4}\right)
       \sum_{i=t-\ell+1}^{t}\E_{U_0}\!\left[
       \frac{\ind{t-\tau_{m_t}+1>2\ell}}{t-\tau_{m_t}+1}\right]\notag\\
 &=C_4\ell\,d\beta_t^2\log t
       \left(\frac1{\eps r^2}+\frac\lambda{(\eps\underline p)^4}\right)
       \E_{U_0}\!\left[
       \frac{\ind{t-\tau_{m_t}+1>2\ell}}{t-\tau_{m_t}+1}\right].
 \label{eq:new-width-mass-average}
\end{align}
The first inequality is \Cref{eq:new-arrival-probability}. The second follows from \Cref{eq:new-epoch-mass} and the preceding inequality, increasing $C_4$ to absorb the factor $2$. The equality follows because the sum has $\ell$ identical terms.

It remains to bound the expectation in \Cref{eq:new-width-mass-average}. Since $t-\tau_{m_t}+1=t-\lceil t(1-Y)\rceil+1=\lfloor tY\rfloor+1$, we have
\begin{align}
 &\E_{U_0}\!\left[
   \frac{\ind{t-\tau_{m_t}+1>2\ell}}{t-\tau_{m_t}+1}\right]\notag\\
 &\quad=\int_0^{1/2}
   \frac{\ind{\lfloor ty\rfloor+1>2\ell}}{\lfloor ty\rfloor+1}
   f_Y(y)\,dy\notag\\
 &\quad\le\int_{(2\ell-1)/t}^{1/2}
          \frac{f_Y(y)}{\lfloor ty\rfloor+1}\,dy\notag\\
 &\quad\le\frac1t\int_{(2\ell-1)/t}^{1/2}\frac{f_Y(y)}{y}\,dy\notag\\
 &\quad\le\frac2{t\log2}\int_{(2\ell-1)/t}^{1/2}\frac{dy}{y}\notag\\
 &\quad=\frac2{t\log2}\log\!\left(\frac{t}{4\ell-2}\right)
 \le\frac{2\log t}{t\log2}.
 \label{eq:new-reset-reciprocal}
\end{align}
The first inequality follows from $\{\lfloor ty\rfloor+1>2\ell\}\subseteq\{ty>2\ell-1\}$. The second inequality follows from $\lfloor ty\rfloor+1>ty$, and the third from $f_Y(y)\le2/\log2$. The last inequality follows from $\ell\ge2$. Substituting \Cref{eq:new-reset-reciprocal} into \Cref{eq:new-width-mass-average} gives
\begin{equation*}
 I_2\le\frac{C_4\ell\,d\beta_t^2\log^2 t}t
       \left(\frac1{\eps r^2}+\frac\lambda{(\eps\underline p)^4}\right).
\end{equation*}
Substituting this bound and \Cref{eq:new-reset-probability} into \Cref{eq:width-transfer-decomposition}, and using $\eps,r\le1$ and $d\beta_t^2\log^2 t\ge1$, proves the result.
\end{proof}

\begin{lemma}
\label{lem:service-error}
Let $t$ satisfy \Cref{eq:learning-threshold}, set $t_1=t-\ell_{0,t}-\ell_{1,t}$ and $t_2=t-\ell_{1,t}$, and define
\begin{equation*}
 \calE_{2,t}
 =\{Q(j)=0\text{ for some }j\in[t_1,t_2]\}
   \cap\bigcap_{i=t_1}^{t}
        (\calE_{1,i}\cap\{w_i\le\eps\underline p/4\}).
\end{equation*}
Then the following hold.
\begin{enumerate}
\renewcommand{\labelenumi}{(\roman{enumi})}
\item $\displaystyle \Pp(\calE_{2,t}^{\mathrm c}) \le\frac{C_4(\eps\underline p^2+\lambda) (d\beta_t^2)^3\log^4 t}{(\eps\underline p)^{10}t}.$
\item For every $z\ge0$,
\begin{equation*}
 \Pp(W(t)>z)
 \le\frac{C_4}{\eps^2\underline p^2}e^{-\eps\underline p z/2}
       +\frac{C_4(\eps\underline p^2+\lambda)
                    (d\beta_t^2)^3\log^4 t}{(\eps\underline p)^{10}t}.
\end{equation*}
\item For every $2\le\ell\le\ell_{1,t}$,
\begin{equation*}
 \Pp\bigl(\calE_{2,t},\
       \text{the job selected at }t\text{ arrived before }t-\ell\bigr)
 \le\frac{C_4}{\eps^2\underline p^2}e^{-\ell/C_3}.
\end{equation*}
\item $\displaystyle \Pp(Q(t)>0,\calE_{2,t})\le\frac{C_4\lambda}{\eps^2\underline p^2}.$
\end{enumerate}
\end{lemma}

\subsubsection{Proof of \texorpdfstring{\Cref{lem:service-error}}{Lemma \getrefnumber{lem:service-error}}}

\begin{proof}
We first prove parts (i) and (ii). Splitting $\{W(t)>z\}$ according to whether $\calE_{2,t}$ holds gives
\begin{equation}
 \Pp(W(t)>z)
 \le\underbrace{\Pp(\calE_{2,t}^{\mathrm c})}_{I_1}
       +\underbrace{\Pp(W(t)>z,\calE_{2,t})}_{I_2}.
 \label{eq:workload-tail-decomposition}
\end{equation}
We bound $I_1$ and $I_2$ in turn.

The term $I_1$ is the probability that either no empty round occurs in $[t_1,t_2]$, or $\calE_{1,i}$ fails or $w_i>\eps\underline p/4$ at some $i\in[t_1,t]$. We bound these two probabilities separately. The union bound gives
\begin{equation*}
 I_1
 \le\Pp(Q(j)>0\text{ for all }j\in[t_1,t_2])
       +\sum_{i=t_1}^t
           \Pp(\calE_{1,i}^{\mathrm c}\text{ or }w_i>\eps\underline p/4).
\end{equation*}
For the first term,
\begin{align*}
 &\Pp(Q(j)>0\text{ for all }j\in[t_1,t_2])\\
 &\quad=\E_{U_0}\!\left[
       \Pp(Q(j)>0\text{ for all }j\in[t_1,t_2]\mid U_0)\right]
 \le\frac{C_4}{\eps^2\underline p^2t^2}.
\end{align*}
The equality follows from the tower rule. The inequality follows from \Cref{lem:recent-empty}-(i) with $t_3\gets t_2$, since $t_1=t_2-\ell_{0,t}$ and $t_2\in[t/3,t]$. For the sum, \Cref{lem:width-transfer} bounds $\Pp(w_i>\eps\underline p/4)$ in terms of the probability that the selected job arrived before $i-\ell_{0,i}$. Such a job must remain in the queue throughout $[i-\ell_{0,i},i]$, so this last probability is bounded by \Cref{lem:recent-empty}-(i). We first check that these lemmas apply. By \Cref{eq:learning-threshold}, $t_1\ge3t/4$. Choose the absolute constant in that condition large enough to apply \Cref{lem:recent-empty,lem:width-transfer} at every $i\in[t_1,t]$, with $2\le\ell_{0,i}\le i/8$. We obtain
\begin{align}
 &\Pp(\calE_{1,i}^{\mathrm c}\text{ or }w_i>\eps\underline p/4)\notag\\
 &\quad\le\Pp(\calE_{1,i}^{\mathrm c})+\Pp(w_i>\eps\underline p/4)\notag\\
 &\quad\le C_4(i+1)^{-2}
       +\frac{C_4\ell_{0,i}d\beta_i^2\log^2 i}i
          \left(\frac1{\eps^3\underline p^2}
                     +\frac\lambda{(\eps\underline p)^4}\right)\notag\\
 &\qquad{}+\Pp(Q(j)>0\text{ for all }j\in[i-\ell_{0,i},i])\notag\\
 &\quad\le\frac{C_4}{\eps^2\underline p^2i^2}
       +\frac{C_4(d\beta_i^2)^2\log^3 i}{(\eps\underline p)^3i}
          \left(\frac1{\eps^3\underline p^2}
                     +\frac\lambda{(\eps\underline p)^4}\right)\notag\\
 &\quad\le\frac{C_4(\eps\underline p^2+\lambda)
                  (d\beta_i^2)^2\log^3 i}{(\eps\underline p)^7i}.
 \label{eq:new-fixed-threshold}
\end{align}
The second inequality follows from \Cref{lem:confidence}-(i), averaged over $U_0$ with $\tau_{m_i}>i/2$, and \Cref{lem:width-transfer} with $t\gets i$, $r=\eps\underline p/4$, and $\ell=\ell_{0,i}$. The third inequality follows from \Cref{lem:recent-empty}-(i) with $(t,t_3)\gets(i,i)$, averaged over $U_0$, and $\ell_{0,i}\le C_4d\beta_i^2\log i/(\eps\underline p)^3$. The last uses $\eps\underline p^2+\lambda\ge(\eps\underline p)^2$, $d\beta_i^2\ge1$, and $\log i\ge1$.

Substituting \Cref{eq:new-fixed-threshold} into the sum gives
\begin{align}
 I_1
 &\le\frac{C_4}{\eps^2\underline p^2t^2}
       +\sum_{i=t_1}^t
          \frac{C_4(\eps\underline p^2+\lambda)
                       (d\beta_i^2)^2\log^3 i}{(\eps\underline p)^7i}
                  \notag\\
 &\le\frac{C_4}{\eps^2\underline p^2t^2}
       +\frac{C_4(\ell_{0,t}+\ell_{1,t}+1)(\eps\underline p^2+\lambda)
                       (d\beta_t^2)^2\log^3 t}{(\eps\underline p)^7t}\notag\\
 &\le\frac{C_4(\eps\underline p^2+\lambda)
                      (d\beta_t^2)^3\log^4 t}{(\eps\underline p)^{10}t}.
 \label{eq:new-clean-probability}
\end{align}
The second inequality follows from $i\ge t_1\ge3t/4$, $\beta_i\le\beta_t$, and $t-t_1+1=\ell_{0,t}+\ell_{1,t}+1$. The last uses $\ell_{0,t}+\ell_{1,t}+1\le C_4d\beta_t^2\log t/(\eps\underline p)^3$ and $\eps\underline p^2+\lambda\ge(\eps\underline p)^2$.

This proves part (i). We next bound $I_2$ for part (ii). On $\calE_{2,t}$, there is an empty round $j\in[t_1,t_2]$, so $W(j)=0$. Starting from this round, we use negative workload drift on subsequent nonempty rounds to bound the probability that $W(t)>z$. We first obtain a bound for every $t_3\in[t_2,t]$, since part (iii) also uses the workload at $t_3=t-\ell$. To average over the possible choices of $j$, we use the first empty round in $[t_1,t_2]$. For $j\in[t_1,t_2]$, define
\begin{equation*}
 \calE_{3,j}=\{Q(j)=0,\ Q(i)>0\text{ for all }i\in[t_1,j-1]\}.
\end{equation*}
These events are disjoint, so $\sum_{j=t_1}^{t_2}\Pp(\calE_{3,j})\le1$. We will obtain the same conditional probability bound for every $j$ and use this inequality when averaging over the first empty round. On nonempty rounds, the negative drift requires $\calE_{1,i}$ and $w_i\le\eps\underline p/4$. Define the first round at which one of these conditions fails by
\begin{equation*}
 \tau_{\mathrm f,j}
 =\inf\{i\ge j:\calE_{1,i}^{\mathrm c}\text{ occurs or }
                              w_i>\eps\underline p/4\}.
\end{equation*}
For every $t_3\in[t_2,t]$, we have
\begin{align}
 \Pp(W(t_3)>z,\calE_{2,t})
 &=\sum_{j=t_1}^{t_2}
       \Pp(W(t_3)>z,\calE_{2,t},\calE_{3,j})\notag\\
 &\le\sum_{j=t_1}^{t_2}
       \Pp(W(t_3)>z,\tau_{\mathrm f,j}>t_3,\calE_{3,j})\notag\\
 &=\sum_{j=t_1}^{t_2}
       \E\!\left[\ind{\calE_{3,j}}
           \Pp(W(t_3)>z,\tau_{\mathrm f,j}>t_3\mid\calF_j)\right].
 \label{eq:workload-first-empty-decomposition}
\end{align}
The first equality follows because the events $\calE_{3,j}$ are disjoint and their union contains $\calE_{2,t}$. The inequality follows from $\calE_{2,t}\cap\calE_{3,j}\subseteq\{\tau_{\mathrm f,j}>t_3\}$. The last equality follows from the tower rule and $\calE_{3,j}\in\calF_j$.

Fix $j\in[t_1,t_2]$. To bound the conditional probability in \Cref{eq:workload-first-empty-decomposition}, set $Z_i=e^{\eps\underline p W(i)/2}\ind{\tau_{\mathrm f,j}>i}$ for $i\ge j$. On $\{Q(j)=0\}$, we will show that
\begin{equation*}
 \E[Z_{t_3}\mid\calF_j]\le\frac{C_4}{\eps^2\underline p^2}
 \qquad\text{for every }t_3\in[t_2,t],
\end{equation*}
and then apply Markov's inequality. Since $Z_i=0$ for $i\ge\tau_{\mathrm f,j}$, it suffices to consider rounds before $\tau_{\mathrm f,j}$. We first bound $\E[Z_{i+1}\mid\calF_i]$ separately for $W(i)>0$ and $W(i)=0$. On $\{W(i)>0,\tau_{\mathrm f,j}>i\}$, \Cref{lem:workload}-(i) and \Cref{lem:confidence}-(ii) give
\begin{equation}
 \E[\Delta W(i)\mid\calF_i]
 \le-\eps\underline p+g_i
 \le-\eps\underline p+2w_i
 \le-\eps\underline p/2.
 \label{eq:new-good-round-drift}
\end{equation}
We subtract the conditional mean of $\Delta W(i)$ to apply Hoeffding's lemma. We have
\begin{align}
 &\E[e^{\eps\underline p\Delta W(i)/2}\mid\calF_i]\notag\\
 &\quad=e^{(\eps\underline p/2)\E[\Delta W(i)\mid\calF_i]}
       \E\!\left[
       e^{(\eps\underline p/2)(\Delta W(i)-\E[\Delta W(i)\mid\calF_i])}
       \,\middle|\,\calF_i\right]\notag\\
 &\quad\le e^{(\eps\underline p/2)\E[\Delta W(i)\mid\calF_i]}
             e^{\eps^2\underline p^2/8}\notag\\
 &\quad\le e^{-\eps^2\underline p^2/4+\eps^2\underline p^2/8}
       =e^{-\eps^2\underline p^2/8}.
 \label{eq:workload-exponential-increment}
\end{align}
The first inequality follows from Hoeffding's lemma (\Cref{lem:aux-hoeffding}) with $z=\eps\underline p/2$ and $-1\le\Delta W(i)\le1$. The second inequality follows from \Cref{eq:new-good-round-drift}. On $\{W(i)=0\}$, we have $W(i+1)\le1$. Hence
\begin{align*}
 \E[Z_{i+1}\mid\calF_i]
 &\le\ind{\tau_{\mathrm f,j}>i}
          \E[e^{\eps\underline p W(i+1)/2}\mid\calF_i]\\
 &\le e^{-\eps^2\underline p^2/8}Z_i\ind{W(i)>0}
       +e^{\eps\underline p/2}\ind{\tau_{\mathrm f,j}>i,\ W(i)=0}\\
 &\le e^{-\eps^2\underline p^2/8}Z_i+e^{\eps\underline p/2}.
\end{align*}
The first inequality follows from $\{\tau_{\mathrm f,j}>i+1\}\subseteq\{\tau_{\mathrm f,j}>i\}\in\calF_i$. The second inequality follows from \Cref{eq:workload-exponential-increment} on $\{W(i)>0\}$ and $W(i+1)\le1$ on $\{W(i)=0\}$. Taking conditional expectations given $\calF_j$ gives, for $i\ge j$,
\begin{align*}
 \E[Z_{i+1}\mid\calF_j]
 &=\E\!\left[\E[Z_{i+1}\mid\calF_i]\mid\calF_j\right]\\
 &\le e^{-\eps^2\underline p^2/8}\E[Z_i\mid\calF_j]
       +e^{\eps\underline p/2}.
\end{align*}
The equality follows from the tower rule. By induction on $t_3-j$, on $\{Q(j)=0\}$ we obtain
\begin{align*}
 \E[Z_{t_3}\mid\calF_j]
 &\le e^{-\eps^2\underline p^2(t_3-j)/8}Z_j
       +e^{\eps\underline p/2}
             \sum_{i=0}^{t_3-j-1}e^{-\eps^2\underline p^2i/8}\\
 &\le e^{-\eps^2\underline p^2(t_3-j)/8}
       +e^{\eps\underline p/2}
             \sum_{i=0}^{t_3-j-1}e^{-\eps^2\underline p^2i/8}\\
 &\le1+\frac{e^{\eps\underline p/2}}
                   {1-e^{-\eps^2\underline p^2/8}}
 \le\frac{C_4}{\eps^2\underline p^2}.
\end{align*}
The second inequality follows from $Z_j\le1$ on $\{Q(j)=0\}$. The last inequality uses $1-e^{-z}\ge z/(1+z)$ and $0<\eps\underline p\le1$. On $\{Q(j)=0\}$,
\begin{align}
 \Pp(W(t_3)>z,\tau_{\mathrm f,j}>t_3\mid\calF_j)
 &=\Pp(Z_{t_3}>e^{\eps\underline p z/2}\mid\calF_j)\notag\\
 &\le e^{-\eps\underline p z/2}\E[Z_{t_3}\mid\calF_j]\notag\\
 &\le\frac{C_4}{\eps^2\underline p^2}e^{-\eps\underline p z/2},
 \label{eq:new-empty-workload-tail}
\end{align}
where the first inequality follows from Markov's inequality. Substituting \Cref{eq:new-empty-workload-tail} into \Cref{eq:workload-first-empty-decomposition} gives, for every $t_3\in[t_2,t]$,
\begin{align}
 \Pp(W(t_3)>z,\calE_{2,t})
 &\le\frac{C_4}{\eps^2\underline p^2}e^{-\eps\underline p z/2}
                      \sum_{j=t_1}^{t_2}\Pp(\calE_{3,j})\notag\\
 &\le\frac{C_4}{\eps^2\underline p^2}e^{-\eps\underline p z/2}.
 \label{eq:new-clean-workload-tail}
\end{align}
The last inequality follows because the events $\calE_{3,j}$ are disjoint. Taking $t_3=t$ in \Cref{eq:new-clean-workload-tail} bounds $I_2$. Substituting this bound and \Cref{eq:new-clean-probability} into \Cref{eq:workload-tail-decomposition} proves part (ii).

We next prove part (iii). If the job selected at $t$ arrived before $t-\ell$, it is present throughout $[t-\ell,t]$, so the queue cannot be empty in this interval. We separate the cases $W(t-\ell)>\eps\underline p\,\ell/4$ and $W(t-\ell)\le\eps\underline p\,\ell/4$. We bound $\Pp(W(t-\ell)>\eps\underline p\,\ell/4,\calE_{2,t})$ by \Cref{eq:new-clean-workload-tail}. In the second case, on $\calE_{2,t}$, the conditional mean increments over this nonempty interval sum to at most $-\eps\underline p\,\ell/2$. For the workload to remain nonnegative, the sum of $\Delta W(i)-\E[\Delta W(i)\mid\calF_i]$ must therefore be at least $\eps\underline p\,\ell/4$. We bound this probability by the Azuma-Hoeffding inequality. Indeed, for $2\le\ell\le\ell_{1,t}$, on $\calE_{2,t}$ with the job selected at $t$ arriving before $t-\ell$ and $W(t-\ell)\le\eps\underline p\,\ell/4$, \Cref{eq:new-good-round-drift} gives
\begin{align*}
 \sum_{i=t-\ell}^{t-1}
       (\Delta W(i)-\E[\Delta W(i)\mid\calF_i])
 &=W(t)-W(t-\ell)-\sum_{i=t-\ell}^{t-1}\E[\Delta W(i)\mid\calF_i]\\
 &\ge-\eps\underline p\,\ell/4+\eps\underline p\,\ell/2
 =\eps\underline p\,\ell/4.
\end{align*}
It follows that
\begin{align}
 &\Pp\bigl(\calE_{2,t},\
       \text{the job selected at }t\text{ arrived before }t-\ell\bigr)
       \notag\\
 &\quad\le\Pp(W(t-\ell)>\eps\underline p\,\ell/4,\calE_{2,t})\notag\\
 &\qquad{}+\Pp\!\left(
       \sum_{i=t-\ell}^{t-1}
        (\Delta W(i)-\E[\Delta W(i)\mid\calF_i])
                \ge\eps\underline p\,\ell/4\right)\notag\\
 &\quad\le\frac{C_4}{\eps^2\underline p^2}
                  e^{-\eps^2\underline p^2\ell/8}
           +e^{-\eps^2\underline p^2\ell/128}
 \le\frac{C_4}{\eps^2\underline p^2}e^{-\ell/C_3},
 \label{eq:new-short-age}
\end{align}
where the second inequality follows from \Cref{eq:new-clean-workload-tail} with $t-\ell\in[t_2,t]$ and the Azuma-Hoeffding inequality (\Cref{lem:aux-azuma}). The last inequality uses $C_3=128/(\eps\underline p)^2$. In particular, taking $\ell=\ell_{1,t}=\lceil C_3\log t\rceil$ bounds \Cref{eq:new-short-age} by $C_4/(\eps^2\underline p^2t)$.

Finally, we prove part (iv). We obtain the factor $\lambda$ by using the arrival immediately after the last empty round. On $\{Q(t)>0\}\cap\calE_{2,t}$, let $j=\max\{i\in[t_1,t-1]:Q(i)=0\}$. Then $A(j)=1$ and $Q(i)>0$ for every $i\in[j+1,t]$. We will bound the probability of remaining nonempty from $j+1$ to $t$ by an exponential function of $t-j$ and sum over $j$. We have
\begin{align}
 &\Pp(Q(t)>0,\calE_{2,t})\notag\\
 &\quad\le\sum_{j=t_1}^{t-1}
       \Pp\bigl(Q(j)=0,A(j)=1,\ 
          Q(i)>0\text{ for all }i\in[j+1,t],\calE_{2,t}\bigr)\notag\\
 &\quad=\sum_{j=t_1}^{t-1}\E\!\left[
       \ind{Q(j)=0,A(j)=1}
       \Pp\bigl(\calE_{2,t},\
          Q(i)>0\text{ for all }i\in[j+1,t]\mid\calF_{j+1}\bigr)\right].
 \label{eq:nonempty-last-empty-decomposition}
\end{align}
The equality follows from the tower rule and $\{Q(j)=0,A(j)=1\}\in\calF_{j+1}$.

Fix $j\in[t_1,t-1]$. On $\{Q(j)=0,A(j)=1\}$, we have $W(j+1)\le1$. We bound the conditional probability in \Cref{eq:nonempty-last-empty-decomposition} by $2e^{-\eps^2\underline p^2(t-j)/128}$. As in part (iii), remaining nonempty forces the sum of $\Delta W(i)-\E[\Delta W(i)\mid\calF_i]$ to be large enough to offset the negative drift. We first consider $t-j\ge1+4/(\eps\underline p)$. If $Q(i)>0$ for every $i\in[j+1,t]$ and $\calE_{2,t}$ holds, then
\begin{align*}
 &\sum_{i=j+1}^{t-1}
        (\Delta W(i)-\E[\Delta W(i)\mid\calF_i])\\
 &\quad=W(t)-W(j+1)-\sum_{i=j+1}^{t-1}\E[\Delta W(i)\mid\calF_i]\\
 &\quad\ge-1+\eps\underline p(t-j-1)/2
 \ge\eps\underline p(t-j-1)/4.
\end{align*}
The first inequality follows from $W(t)\ge0$, $W(j+1)\le1$, and \Cref{eq:new-good-round-drift}. The second follows from $t-j-1\ge4/(\eps\underline p)$. Consequently, on $\{Q(j)=0,A(j)=1\}$,
\begin{align*}
 &\Pp\bigl(\calE_{2,t},\
     Q(i)>0\text{ for all }i\in[j+1,t]\mid\calF_{j+1}\bigr)\\
 &\quad\le\Pp\!\left(
       \sum_{i=j+1}^{t-1}
        (\Delta W(i)-\E[\Delta W(i)\mid\calF_i])
       \ge\eps\underline p(t-j-1)/4
       \,\middle|\,\calF_{j+1}\right)\\
 &\quad\le e^{-\eps^2\underline p^2(t-j-1)/128}\\
 &\quad\le2e^{-\eps^2\underline p^2(t-j)/128}.
\end{align*}
The second inequality follows from the Azuma-Hoeffding inequality (\Cref{lem:aux-azuma}) and $|\Delta W(i)-\E[\Delta W(i)\mid\calF_i]|\le2$. The last uses $e^{\eps^2\underline p^2/128}\le2$. For $t-j<1+4/(\eps\underline p)$, we instead bound the conditional probability by $1$. Since $0<\eps\underline p\le1$, we have
\begin{align*}
 &\Pp\bigl(\calE_{2,t},\
     Q(i)>0\text{ for all }i\in[j+1,t]\mid\calF_{j+1}\bigr)\\
 &\quad\le1\le2e^{-\eps^2\underline p^2(t-j)/128}.
\end{align*}
Thus this bound holds for every $j\in[t_1,t-1]$. Substituting into \Cref{eq:nonempty-last-empty-decomposition} gives
\begin{align*}
 \Pp(Q(t)>0,\calE_{2,t})
 &\le2\sum_{j=t_1}^{t-1}e^{-\eps^2\underline p^2(t-j)/128}
                         \Pp(Q(j)=0,A(j)=1)\\
 &\le2\lambda\sum_{j=t_1}^{t-1}e^{-\eps^2\underline p^2(t-j)/128}\\
 &\le\frac{2\lambda}{e^{\eps^2\underline p^2/128}-1}\\
 &\le\frac{C_4\lambda}{\eps^2\underline p^2}.
\end{align*}
The second inequality follows from $\Pp(Q(j)=0,A(j)=1)=\lambda\Pp(Q(j)=0)\le\lambda$, since $A(j)$ is independent of $\calF_j$. The last inequality uses $e^z-1\ge z$.
\end{proof}

\subsection{Proof of \texorpdfstring{\Cref{lem:departure-rate-gap}}{Lemma \getrefnumber{lem:departure-rate-gap}}}
\label{app:departure-rate-gap}

\begin{proof}
Suppose $t$ satisfies \Cref{eq:learning-threshold} and
\begin{equation}
 t\ge\frac{C_4\eps(\eps\underline p^2+\lambda)^2
                    (d\beta_t^2)^5\log^6 t}
              {\lambda(\eps\underline p)^{16}\log(e/\lambda)}.
 \label{eq:gap-threshold}
\end{equation}
Let $\calE_{2,t}$ be defined in \Cref{lem:service-error}. Since $0\le g_t\le1$, we have
\begin{align}
 \E[g_t]
 &=\E[g_t\ind{\calE_{2,t}}]
       +\E[g_t\ind{\calE_{2,t}^{\mathrm c}}]\notag\\
 &\le\underbrace{\int_0^1\Pp(g_t>z,\calE_{2,t})\,dz}_{I_1}
       +\underbrace{\Pp(\calE_{2,t}^{\mathrm c})}_{I_2}.
 \label{eq:gap-expectation-decomposition}
\end{align}
We first bound $I_2$ by \Cref{lem:service-error}-(i),
\[
 I_2\le\frac{C_4(\eps\underline p^2+\lambda)
                    (d\beta_t^2)^3\log^4 t}{(\eps\underline p)^{10}t}.
\]
To bound $I_1$, we use two bounds on $\Pp(g_t>z,\calE_{2,t})$. For large $z$, we combine \Cref{lem:width-transfer} with \Cref{lem:service-error}-(iii). For small $z$, we use \Cref{lem:service-error}-(iv) to obtain a bound independent of $z$. We then choose where to split the integral to balance the two contributions.

We first obtain the bound for large $z$. We distinguish whether the job selected at $t$ arrived at or after $t-\ell$, or before $t-\ell$. Applying \Cref{lem:width-transfer} with $r=z/2$ leaves the probability of the second case on $\calE_{2,t}$ as a separate term. We bound this probability using \Cref{lem:service-error}-(iii). For $0<z\le1$ and $2\le\ell\le\ell_{1,t}\le t/8$, we have
\begin{align}
 \Pp(g_t>z,\calE_{2,t})
 &\le\Pp(w_t>z/2,\calE_{2,t})\notag\\
 &\le\frac{C_4\ell\,d\beta_t^2\log^2 t}t
       \left(\frac4{\eps z^2}+\frac\lambda{(\eps\underline p)^4}\right)\notag\\
 &\quad{}+\Pp\bigl(\calE_{2,t},\
        \text{the job selected at }t\text{ arrived before }t-\ell\bigr)\notag\\
 &\le\frac{C_4\ell\,d\beta_t^2\log^2 t}t
       \left(\frac4{\eps z^2}+\frac\lambda{(\eps\underline p)^4}\right)
       +\frac{C_4}{\eps^2\underline p^2}e^{-\ell/C_3}.
 \label{eq:gap-width-tail}
\end{align}
The first inequality follows from \Cref{lem:confidence}-(ii) on $\calE_{2,t}$, the second from \Cref{lem:width-transfer} with $r=z/2$, and the last from \Cref{lem:service-error}-(iii). The first term in the last line of \Cref{eq:gap-width-tail} increases linearly with $\ell$. The last term bounds the probability that the selected job arrived before $t-\ell$ and decreases exponentially with $\ell$. We choose $\ell$ as a function of $z$ to control both terms. For this choice, write the term containing $z^{-2}$ as
\[
 \frac{4C_4\ell\,d\beta_t^2\log^2 t}{\eps tz^2}
 =\frac{4C_4\ell a_t}{z^2},
 \qquad a_t=\frac{d\beta_t^2\log^2 t}{\eps t}.
\]
We choose $\ell$ so that $e^{-\ell/C_3}\le a_t/z^2$. This bounds the exponential term in \Cref{eq:gap-width-tail} by a multiple of $a_t/z^2$, without taking $\ell$ larger than needed. If $z<\sqrt{a_t}$, then $a_t/z^2>1$, so the term $4C_4\ell a_t/z^2$ already exceeds $1$. We therefore use this bound for $\sqrt{a_t}\le z\le1$, where $\log(z^2/a_t)\ge0$, and set
\[
 \ell=\max\{2,\lceil C_3\log(z^2/a_t)\rceil\}.
\]
Since $a_t\ge1/t$ and $\ell_{1,t}=\lceil C_3\log t\rceil$, we have $2\le\ell\le\ell_{1,t}$ and
\[
 \ell\le2+C_3\log(z^2/a_t),\qquad e^{-\ell/C_3}\le a_t/z^2.
\]
We use these inequalities for the terms containing $z^{-2}$ and $e^{-\ell/C_3}$, and $\ell\le\ell_{1,t}$ for the term containing $\lambda$. Substituting into \Cref{eq:gap-width-tail} with $C_3=128/(\eps\underline p)^2$ gives
\begin{equation}
 \Pp(g_t>z,\calE_{2,t})
 \le\frac{C_4a_t}{\eps^2\underline p^2z^2}
                   \bigl(1+\log(z^2/a_t)\bigr)
       +\frac{C_4\lambda d\beta_t^2\log^3 t}{(\eps\underline p)^6t}.
 \label{eq:gap-variable-tail}
\end{equation}
For small $z$, we instead use the fact that $g_t>z$ requires $Q(t)>0$. This gives a bound proportional to $\lambda$ that does not increase as $z$ decreases. For every $z>0$,
\begin{align}
 \Pp(g_t>z,\calE_{2,t})
 &=\Pp(g_t>z,Q(t)>0,\calE_{2,t})\notag\\
 &\le\Pp(Q(t)>0,\calE_{2,t})
 \le\frac{C_4\lambda}{\eps^2\underline p^2}.
 \label{eq:gap-nonempty-tail}
\end{align}
The equality follows from $g_t=0$ when $Q(t)=0$, and the last inequality follows from \Cref{lem:service-error}-(iv).

We now integrate \Cref{eq:gap-nonempty-tail} over $(0,z_0)$ and \Cref{eq:gap-variable-tail} over $[z_0,1]$. We require $\sqrt{a_t}\le z_0<1$ so that \Cref{eq:gap-variable-tail} applies throughout $[z_0,1]$. We have
\begin{align}
 I_1
 &=\int_0^{z_0}\Pp(g_t>z,\calE_{2,t})\,dz
       +\int_{z_0}^1\Pp(g_t>z,\calE_{2,t})\,dz\notag\\
 &\le\int_0^{z_0}\frac{C_4\lambda}{\eps^2\underline p^2}\,dz
       +\int_{z_0}^1
       \left[
       \frac{C_4a_t(1+\log(z^2/a_t))}{\eps^2\underline p^2z^2}
       +\frac{C_4\lambda d\beta_t^2\log^3 t}{(\eps\underline p)^6t}
       \right]\,dz\notag\\
 &=\frac{C_4\lambda z_0}{\eps^2\underline p^2}
   +\frac{C_4a_t}{\eps^2\underline p^2}
      \int_{z_0}^1\frac{1+\log(z^2/a_t)}{z^2}\,dz
   +\frac{C_4\lambda d\beta_t^2\log^3 t}{(\eps\underline p)^6t}(1-z_0)\notag\\
 &\le\frac{C_4\lambda z_0}{\eps^2\underline p^2}
   +\frac{C_4a_t}{\eps^2\underline p^2}
      \int_{z_0}^{\infty}\frac{1+\log(z^2/a_t)}{z^2}\,dz
   +\frac{C_4\lambda d\beta_t^2\log^3 t}{(\eps\underline p)^6t}\notag\\
 &=\frac{C_4}{\eps^2\underline p^2}
       \left[\lambda z_0+\frac{a_t}{z_0}
                         \bigl(3+\log(z_0^2/a_t)\bigr)\right]
   +\frac{C_4\lambda d\beta_t^2\log^3 t}{(\eps\underline p)^6t}.
 \label{eq:gap-cutoff-integral}
\end{align}
The second equality integrates the terms that do not depend on $z$ over intervals of lengths $z_0$ and $1-z_0$. For the last equality,
\[
 \int_{z_0}^{\infty}\frac{1+\log(z^2/a_t)}{z^2}\,dz
 =\left[-\frac{3+\log(z^2/a_t)}z\right]_{z_0}^{\infty}
 =\frac{3+\log(z_0^2/a_t)}{z_0}.
\]

We now choose $z_0$ in \Cref{eq:gap-cutoff-integral}. The term $\lambda z_0$ comes from the integral over $(0,z_0)$. The term $a_t(3+\log(z_0^2/a_t))/z_0$ bounds the integral over $[z_0,1]$. Increasing $z_0$ increases the first term and decreases the second. To bound both terms by the same order, choose
\[
 \lambda z_0=\frac{a_t\log(e/\lambda)}{z_0},
 \qquad
 z_0=\sqrt{a_t\log(e/\lambda)/\lambda}.
\]
This choice satisfies $z_0\ge\sqrt{a_t}$ and, by \Cref{eq:gap-threshold}, $z_0<1$. The factor $\log(e/\lambda)$ also bounds $3+\log(z_0^2/a_t)$ up to a numerical constant, since
\[
 3+\log(z_0^2/a_t)
 =2+\log\log(e/\lambda)+\log(e/\lambda)
 \le3\log(e/\lambda).
\]
Consequently,
\begin{align*}
 \lambda z_0+\frac{a_t}{z_0}\bigl(3+\log(z_0^2/a_t)\bigr)
 &\le\lambda z_0+\frac{3a_t\log(e/\lambda)}{z_0}\\
 &=4\sqrt{\lambda a_t\log(e/\lambda)}.
\end{align*}
Substituting into \Cref{eq:gap-cutoff-integral} gives
\[
 I_1\le\frac{C_4}{\eps^2\underline p^2}
                  \sqrt{\lambda a_t\log(e/\lambda)}
   +\frac{C_4\lambda d\beta_t^2\log^3 t}{(\eps\underline p)^6t}.
\]

Substituting the bounds on $I_1$ and $I_2$ into \Cref{eq:gap-expectation-decomposition} gives
\begin{align}
 \E[g_t]
 &\le\frac{C_4\beta_t\log t}{\eps^{5/2}\underline p^2}
                 \sqrt{\lambda d\log(e/\lambda)/t}
       +\frac{C_4\lambda d\beta_t^2\log^3 t}{(\eps\underline p)^6t}
       +\frac{C_4(\eps\underline p^2+\lambda)
                    (d\beta_t^2)^3\log^4 t}{(\eps\underline p)^{10}t}\notag\\
 &\le\frac{C_4\beta_t\log t}{\eps^{5/2}\underline p^2}
                 \sqrt{\lambda d\log(e/\lambda)/t}
       +\frac{C_4(\eps\underline p^2+\lambda)
                    (d\beta_t^2)^3\log^4 t}{(\eps\underline p)^{10}t}.
 \label{eq:gap-improved-expectation}
\end{align}
The first inequality uses the definition of $a_t$. The second follows from $\eps\underline p\le1$, $d\beta_t^2\ge1$, and $\log t\ge1$.

It remains to bound the second term in \Cref{eq:gap-improved-expectation} by the first. The condition \Cref{eq:gap-threshold} gives
\[
 \frac{(\eps\underline p^2+\lambda)
                    (d\beta_t^2)^3\log^4 t}{(\eps\underline p)^{10}t}
 \le\frac{\beta_t\log t}{\eps^{5/2}\underline p^2}
                    \sqrt{\lambda d\log(e/\lambda)/t}.
\]
Thus \Cref{eq:gap-improved-expectation} proves \Cref{lem:departure-rate-gap} after increasing the absolute constant.
\end{proof}

\subsection{Proof of \texorpdfstring{\Cref{lem:oracle-forgetting}}{Lemma \getrefnumber{lem:oracle-forgetting}}}
\label{app:new-switching}

We relate the departure rate gaps to $R_t$ by comparing policies whose decision rules differ at only one round. For each round, this separates the probability that only the oracle completes a job from the probability that the resulting queue length difference remains at $t$. For $i=0,\ldots,t-1$, let $\pi_i$ follow the learner through round $i$ and the oracle afterwards. Then $Q^{\pi_0}(t)=Q^*(t)$ and $Q^{\pi_{t-1}}(t)=Q(t)$. We compare $\pi_i$ with $\pi_{i-1}$, whose decision rules differ only at round $i$. All queues start empty and share arrivals, contexts, algorithmic randomization, and independent service variables $U_j\sim\Unif(0,1)$. A service attempt with probability $p$ succeeds if $U_j\le p$. The oracle orders jobs by decreasing $p^*(x)$ and breaks ties by arrival time. Repeated contexts correspond to distinct jobs.

At the beginning of round $i$, we have $\calX_i^{\pi_i}=\calX_i^{\pi_{i-1}}=\calX_i$. Write $p_i=p(x_i,k_i)$ and $p_i^*=p(x_i^*,k_i^*)$, with both zero when $Q(i)=0$. Then $g_i=p_i^*-p_i$. Define
\begin{equation*}
 q_{i,t}=\Pp\bigl(Q^{\pi_i}(t)-Q^{\pi_{i-1}}(t)=1
                 \mid\calF_i,\ p_i<U_i\le p_i^*\bigr)
 \quad\text{if }g_i>0,
 \qquad q_{i,t}=0\quad\text{if }g_i=0.
\end{equation*}
The event $p_i<U_i\le p_i^*$ has conditional probability $g_i$ given $\calF_i$. On this event, $\pi_{i-1}$ completes a job and $\pi_i$ does not, so $Q^{\pi_i}(i+1)-Q^{\pi_{i-1}}(i+1)=1$. The quantity $q_{i,t}$ is the conditional probability that this queue length difference is still $1$ at round $t$. The following lemma bounds $R_t$ by summing $\E[g_iq_{i,t}]$ over these comparisons.

\begin{lemma}[\citealt{bae2026queue}, Appendices B.1-B.2]
\label{lem:switching-decomposition}
For every $1\le i<t$ and $j\ge i+1$, $Q^{\pi_i}(j)-Q^{\pi_{i-1}}(j)\in\{-1,0,1\}$. This difference is nonpositive unless $p_i<U_i\le p_i^*$. On that event,
\begin{equation*}
 \calX_j^{\pi_{i-1}}\subseteq\calX_j^{\pi_i},\qquad
 Q^{\pi_i}(j)=0\ \Longrightarrow\
 \calX_{j'}^{\pi_i}=\calX_{j'}^{\pi_{i-1}}\quad\text{for every }j'\ge j.
\end{equation*}
Consequently,
\begin{equation*}
 R_t=\sum_{i=1}^{t-1}\E[Q^{\pi_i}(t)-Q^{\pi_{i-1}}(t)]
     \le\sum_{i=1}^{t-1}\E[g_iq_{i,t}].
\end{equation*}
\end{lemma}
The cited coupling uses the service probabilities and the common priority order. It therefore applies under \Cref{ass:bounded-logistic,ass:average-workload} without a covariance lower bound or context-wise slack. For the last inequality, we have
\begin{align*}
 &\E[Q^{\pi_i}(t)-Q^{\pi_{i-1}}(t)\mid\calF_i]\\
 &\quad\le\Pp\bigl(Q^{\pi_i}(t)-Q^{\pi_{i-1}}(t)=1,\
                   p_i<U_i\le p_i^*\mid\calF_i\bigr)\\
 &\quad=(p_i^*-p_i)q_{i,t}=g_iq_{i,t}.
\end{align*}
The equality follows from the definition of $q_{i,t}$ and $\Pp(p_i<U_i\le p_i^*\mid\calF_i)=p_i^*-p_i$. Taking expectations and summing over $i$ gives the result.

The next lemma bounds $q_{i,t}$ using $W(i)$ and $t-i$, so that contributions from early rounds can be small even when $g_i$ is large.
\begin{lemma}
\label{lem:oracle-forgetting}
For every $1\le i<t$,
\begin{equation*}
 q_{i,t}\le\ind{W(i)+1>\eps\underline p(t-i)/32}
             +e^{-\eps^2\underline p^2(t-i)/128}.
\end{equation*}
\end{lemma}

\begin{proof}
If $g_i=0$, the result follows from $q_{i,t}=0$. Set $\ell=t-i$. Since $q_{i,t}\le1$, it suffices to consider $W(i)+1\le\eps\underline p\ell/32$. Condition on $\calF_i$ and $p_i<U_i\le p_i^*$. By \Cref{lem:switching-decomposition}, $\calX_j^{\pi_{i-1}}\subseteq\calX_j^{\pi_i}$ for $j\ge i+1$. If $Q^{\pi_i}(j)=0$ for some $j\in[i+1,t]$, both queues are empty at $j$. They then use the same oracle rule, arrivals, and service variables, so $Q^{\pi_i}(t)=Q^{\pi_{i-1}}(t)$. Therefore, the difference in the definition of $q_{i,t}$ can be $1$ only if $Q^{\pi_i}(j)>0$ for every $j\in[i+1,t]$. We bound this probability using the negative workload drift under the oracle. Define
\begin{equation*}
 W^{\pi_i}(j)=\underline p\sum_{x\in\calX_j^{\pi_i}}1/p^*(x),
 \qquad \tau=\inf\{j\ge i+1:W^{\pi_i}(j)=0\}.
\end{equation*}
On $\{p_i<U_i\le p_i^*\}$, service fails at round $i$ and an arrival adds at most one unit of workload, so $0<W^{\pi_i}(i+1)\le W(i)+1$. It follows that $\{Q^{\pi_i}(t)-Q^{\pi_{i-1}}(t)=1\}\subseteq\{\tau>t\}$, and
\begin{equation}
 q_{i,t}\le\Pp(\tau>t\mid\calF_i,\ p_i<U_i\le p_i^*).
 \label{eq:later-queue-empty-new}
\end{equation}
Let $\calF_j^+$ contain the coupled history before service at $j$. For $j\ge i+1$, policy $\pi_i$ serves each chosen job $x$ on a best server. On $\{j<\tau\}$,
\begin{align}
 \E[W^{\pi_i}(j+1)-W^{\pi_i}(j)\mid\calF_j^+]
 &=\lambda\underline p\,\E_{X\sim\calD}[1/p^*(X)]
       -p^*(x)\underline p/p^*(x)
 =-\eps\underline p,
 \label{eq:oracle-drift-new}
\end{align}
where the last equality follows from \Cref{ass:average-workload}. Also, $|W^{\pi_i}(j+1)-W^{\pi_i}(j)|\le1$.

On $\{\tau>t\}$, the conditional means in \Cref{eq:oracle-drift-new} sum to $-\eps\underline p(\ell-1)$. The initial workload is at most $W(i)+1\le\eps\underline p\ell/32$. We will show that remaining nonempty until $t$ therefore requires the sum of the mean-zero workload increments to exceed $\eps\underline p\ell/4$, and bound this probability by the Azuma-Hoeffding inequality. We write, on $\{\tau>t\}$,
\begin{equation*}
 W^{\pi_i}(t)
 =W^{\pi_i}(i+1)-\eps\underline p(\ell-1)+M_t,
\end{equation*}
where
\begin{equation*}
 M_j=\sum_{j'=i+1}^{j-1}\ind{j'<\tau}
       \bigl(W^{\pi_i}(j'+1)-W^{\pi_i}(j')+\eps\underline p\bigr),
 \qquad i+1\le j\le t.
\end{equation*}
The indicator restricts the sum to $j'<\tau$, where \Cref{eq:oracle-drift-new} applies. Hence
\begin{equation*}
 \E[M_{j+1}-M_j\mid\calF_j^+]
 =\ind{j<\tau}
   \bigl(\E[W^{\pi_i}(j+1)-W^{\pi_i}(j)\mid\calF_j^+]
         +\eps\underline p\bigr)=0.
\end{equation*}
Hence $M_j$ is a martingale with $M_{i+1}=0$ and $|M_{j+1}-M_j|\le2$. On $\{\tau>t\}$,
\begin{align*}
 M_t&=W^{\pi_i}(t)-W^{\pi_i}(i+1)+\eps\underline p(\ell-1)\\
 &>-\eps\underline p\ell/32+\eps\underline p(\ell-1)
 \ge\eps\underline p\ell/4,
\end{align*}
where the first inequality follows from $W^{\pi_i}(t)>0$ and $W^{\pi_i}(i+1)\le W(i)+1\le\eps\underline p\ell/32$. Therefore,
\begin{align*}
 \Pp(\tau>t\mid\calF_{i+1}^+)
 &\le\Pp(M_t>\eps\underline p\ell/4\mid\calF_{i+1}^+)\\
 &\le\exp\!\left(-\frac{\eps^2\underline p^2\ell^2}{128(\ell-1)}\right)
 \le e^{-\eps^2\underline p^2\ell/128}.
\end{align*}
The second inequality follows from the Azuma-Hoeffding inequality (\Cref{lem:aux-azuma}) with $\ell-1$ increments bounded by $2$. Substituting into \Cref{eq:later-queue-empty-new} gives
\begin{align*}
 q_{i,t}
 &\le\Pp(\tau>t\mid\calF_i,\ p_i<U_i\le p_i^*)\\
 &=\E\!\left[\Pp(\tau>t\mid\calF_{i+1}^+)
          \,\middle|\,\calF_i,\ p_i<U_i\le p_i^*\right]\\
 &\le e^{-\eps^2\underline p^2\ell/128},
\end{align*}
where the equality follows from the tower rule. Combining this bound with $q_{i,t}\le1$ when $W(i)+1>\eps\underline p\ell/32$ proves the lemma.
\end{proof}

\subsection{Proof of \texorpdfstring{\Cref{thm:main-upper}}{Theorem \getrefnumber{thm:main-upper}}}
\label{app:upper-proof}

\begin{proof}
Take $t\ge t_0\ge3$, where $t_0$ is chosen so that, for every $t\ge t_0$,
\begin{equation}
 t\ge\frac{C_4d\beta_t^2\log t}{(\eps\underline p)^4},
 \qquad
 t\ge\frac{C_4\eps(\eps\underline p^2+\lambda)^2
                    (d\beta_t^2)^7\log^6 t}
              {\lambda(\eps\underline p)^{20}\log(e/\lambda)}.
 \label{eq:upper-threshold}
\end{equation}
By \Cref{eq:confidence-radius}, $\beta_t^2=O(\log t)$. Hence both conditions in \Cref{eq:upper-threshold} hold for all sufficiently large $t$, with $t_0$ polynomial in $d,\kappa,S,\lambda^{-1},\eps^{-1},\underline p^{-1}$ up to logarithmic factors.

By \Cref{lem:switching-decomposition},
\begin{align}
 R_t&\le\sum_{i=1}^{t-1}\E[g_iq_{i,t}]\notag\\
 &=\underbrace{\sum_{i=1}^{\tau_{1,t}}\E[g_iq_{i,t}]}_{I_1}
 +\underbrace{\sum_{i=\tau_{1,t}+1}^{\tau_{2,t}}\E[g_iq_{i,t}]}_{I_2}
 +\underbrace{\sum_{i=\tau_{2,t}+1}^{t-1}\E[g_iq_{i,t}]}_{I_3},
 \label{eq:upper-decomposition}
\end{align}
where
\begin{equation*}
 \tau_{1,t}=t-\lceil32(\ell_{0,t}+1)/(\eps\underline p)\rceil,
 \qquad \tau_{2,t}=t-\ell_{1,t}-1.
\end{equation*}
The first condition in \Cref{eq:upper-threshold} ensures $1\le\tau_{1,t}<\tau_{2,t}<t$, $\tau_{1,t}\ge t/2$, and \Cref{eq:learning-threshold} at every $i>\eps\underline p\,t/96$. The second ensures \Cref{eq:gap-threshold} at every $i\ge t/2$. Since $0\le g_i,q_{i,t}\le1$, we can bound $\E[g_iq_{i,t}]$ by either $\E[q_{i,t}]$ or $\E[g_i]$. We use \Cref{lem:oracle-forgetting} to bound $\E[q_{i,t}]$ in $I_1,I_2$, and \Cref{lem:departure-rate-gap} to bound $\E[g_i]$ in $I_3$.

For $I_1$, since $g_i\le1$, \Cref{lem:oracle-forgetting} gives
\begin{equation}
 I_1\le\sum_{i=1}^{\tau_{1,t}}
       \Pp(W(i)+1>\eps\underline p(t-i)/32)
       +\sum_{i=1}^{\tau_{1,t}}e^{-\eps^2\underline p^2(t-i)/128}.
 \label{eq:far-sum}
\end{equation}
We first bound the first sum in \Cref{eq:far-sum}. For $i\le\eps\underline p\,t/96$, we use $W(i)\le Q(i)\le i-1$, which gives
\begin{equation*}
 W(i)+1\le i\le\eps\underline p\,t/96
             \le\eps\underline p(t-i)/32.
\end{equation*}
Thus these indices contribute zero to the first sum in \Cref{eq:far-sum}. For $\eps\underline p\,t/96<i\le\tau_{1,t}$, we use \Cref{lem:recent-empty}-(ii), which bounds $\Pp(W(i)>\ell_{0,i}\mid U_0)$. To apply this result, we compare $\eps\underline p(t-i)/32$ with $\ell_{0,i}+1$.
\begin{equation*}
 \ell_{0,i}+1\le\ell_{0,t}+1
 \le\eps\underline p(t-\tau_{1,t})/32
 \le\eps\underline p(t-i)/32.
\end{equation*}
The first inequality follows from the monotonicity of $\beta_i$, and the second from the definition of $\tau_{1,t}$. Thus $W(i)+1>\eps\underline p(t-i)/32$ implies $W(i)>\ell_{0,i}$, and we obtain
\begin{align*}
 &\sum_{i=1}^{\tau_{1,t}}\Pp(W(i)+1>\eps\underline p(t-i)/32)\\
 &\quad\le\sum_{\eps\underline p\,t/96<i\le\tau_{1,t}}
           \Pp(W(i)>\ell_{0,i})\\
 &\quad=\sum_{\eps\underline p\,t/96<i\le\tau_{1,t}}
           \E_{U_0}\!\left[\Pp(W(i)>\ell_{0,i}\mid U_0)\right]\\
 &\quad\le\frac{C_4}{\eps^2\underline p^2}
           \sum_{\eps\underline p\,t/96<i\le\tau_{1,t}}i^{-2}
 \le\frac{C_4}{(\eps\underline p)^3t}.
\end{align*}
The equality follows from the tower rule, and the last inequality uses $\sum_{i>\eps\underline p\,t/96}i^{-2}\le C_4/(\eps\underline p\,t)$.

We next bound the second sum in \Cref{eq:far-sum}. For any $1\le j\le t-\ell_{1,t}$,
\begin{align}
 \sum_{i=1}^{j}e^{-\eps^2\underline p^2(t-i)/128}
 &=e^{-\eps^2\underline p^2(t-j)/128}
       \sum_{i=0}^{j-1}e^{-\eps^2\underline p^2i/128}\notag\\
 &\le\frac{e^{-\eps^2\underline p^2(t-j)/128}}
                   {1-e^{-\eps^2\underline p^2/128}}\notag\\
 &\le\frac{C_4}{\eps^2\underline p^2}
          e^{-\eps^2\underline p^2\ell_{1,t}/128}
 \le\frac{C_4}{\eps^2\underline p^2t}.
 \label{eq:upper-exponential-sum}
\end{align}
The second inequality uses $t-j\ge\ell_{1,t}$ and $1/(1-e^{-\eps^2\underline p^2/128})\le C_4/(\eps\underline p)^2$. The last follows from $\ell_{1,t}\ge128\log t/(\eps\underline p)^2$. Applying \Cref{eq:upper-exponential-sum} with $j=\tau_{1,t}$ and substituting both bounds into \Cref{eq:far-sum} gives
\begin{equation}
 I_1\le\frac{C_4}{(\eps\underline p)^3t}.
 \label{eq:proof-first-contribution}
\end{equation}

For $I_2$, the threshold $\eps\underline p(t-i)/32-1$ can be smaller than $\ell_{0,i}$, so the comparison used to apply \Cref{lem:recent-empty}-(ii) in $I_1$ is no longer guaranteed. We instead use \Cref{lem:service-error}-(ii), which bounds $\Pp(W(i)>z)$ for every $z\ge0$. Every index in $I_2$ satisfies $i\ge t/2$ and $\eps\underline p(t-i)/32-1\ge0$. Taking $z=\eps\underline p(t-i)/32-1$ gives
\begin{align*}
 &\Pp(W(i)+1>\eps\underline p(t-i)/32)\\
 &\quad\le\frac{C_4}{\eps^2\underline p^2}
            e^{-\eps\underline p(\eps\underline p(t-i)/32-1)/2}
   +\frac{C_4(\eps\underline p^2+\lambda)
                    (d\beta_i^2)^3\log^4 i}{(\eps\underline p)^{10}i}\\
 &\quad\le\frac{C_4}{\eps^2\underline p^2}
            e^{-\eps^2\underline p^2(t-i)/128}
   +\frac{C_4(\eps\underline p^2+\lambda)
                    (d\beta_t^2)^3\log^4 t}{(\eps\underline p)^{10}t}.
\end{align*}
The second inequality uses $i^{-1}\le2/t$, $\beta_i\le\beta_t$, and $e^{\eps\underline p/2}\le e^{1/2}$. To sum the second term, which does not depend on $i$, we use
\begin{equation}
 \tau_{2,t}-\tau_{1,t}
 =\lceil32(\ell_{0,t}+1)/(\eps\underline p)\rceil-\ell_{1,t}-1
 \le\frac{32(\ell_{0,t}+1)}{\eps\underline p}
 \le\frac{C_4d\beta_t^2\log t}{(\eps\underline p)^4}.
 \label{eq:upper-middle-length}
\end{equation}
The last inequality follows from \Cref{eq:proof-lengths}. Since $\tau_{2,t}=t-\ell_{1,t}-1$, \Cref{eq:upper-exponential-sum} with $j=\tau_{2,t}$ gives
\begin{equation*}
 \sum_{i=\tau_{1,t}+1}^{\tau_{2,t}}e^{-\eps^2\underline p^2(t-i)/128}
 \le\sum_{i=1}^{\tau_{2,t}}e^{-\eps^2\underline p^2(t-i)/128}
 \le\frac{C_4}{\eps^2\underline p^2t}.
\end{equation*}
Combining this bound with \Cref{eq:upper-middle-length}, we obtain
\begin{align}
 I_2
 &\le\sum_{i=\tau_{1,t}+1}^{\tau_{2,t}}
       \Pp(W(i)+1>\eps\underline p(t-i)/32)
       +\sum_{i=\tau_{1,t}+1}^{\tau_{2,t}}
          e^{-\eps^2\underline p^2(t-i)/128}\notag\\
 &\le\frac{C_4}{\eps^2\underline p^2}
       \sum_{i=\tau_{1,t}+1}^{\tau_{2,t}}
                   e^{-\eps^2\underline p^2(t-i)/128}
       +\frac{C_4(\tau_{2,t}-\tau_{1,t})(\eps\underline p^2+\lambda)
                    (d\beta_t^2)^3\log^4 t}{(\eps\underline p)^{10}t}\notag\\
 &\le\frac{C_4}{(\eps\underline p)^4t}
   +\frac{C_4(\eps\underline p^2+\lambda)
                    (d\beta_t^2)^4\log^5 t}{(\eps\underline p)^{14}t}.
 \label{eq:proof-second-contribution}
\end{align}
The first inequality follows from \Cref{lem:oracle-forgetting} and $g_i\le1$.

For $I_3$, the remaining $t-i$ rounds can be too few for $Q^{\pi_i}$ and $Q^{\pi_{i-1}}$ to become equal by $t$. We therefore use $q_{i,t}\le1$ and bound $\E[g_i]$ by \Cref{lem:departure-rate-gap}. There are only $\ell_{1,t}$ indices in this interval.
\begin{align}
 I_3&\le\sum_{i=\tau_{2,t}+1}^{t-1}\E[g_i]\notag\\
 &\le\sum_{i=\tau_{2,t}+1}^{t-1}
       \frac{C_4\beta_i\log i}{\eps^{5/2}\underline p^2}
                 \sqrt{\lambda d\log(e/\lambda)/i}\notag\\
 &\le\frac{C_4\ell_{1,t}\beta_t\log t}{\eps^{5/2}\underline p^2}
                 \sqrt{\lambda d\log(e/\lambda)/t}\notag\\
 &\le\frac{C_4\beta_t\log^2 t}{\eps^{9/2}\underline p^4}
                 \sqrt{\lambda d\log(e/\lambda)/t}.
 \label{eq:proof-third-contribution}
\end{align}
The third inequality uses $i\ge t/2$, $\beta_i\le\beta_t$, and $t-1-\tau_{2,t}=\ell_{1,t}$. The last follows from $\ell_{1,t}\le C_4\log t/(\eps\underline p)^2$. Substituting \Cref{eq:proof-first-contribution,eq:proof-second-contribution,eq:proof-third-contribution} into \Cref{eq:upper-decomposition} gives
\begin{align}
 R_t&\le\frac{C_4}{(\eps\underline p)^3t}
       +\frac{C_4}{(\eps\underline p)^4t}
       +\frac{C_4(\eps\underline p^2+\lambda)
                    (d\beta_t^2)^4\log^5 t}{(\eps\underline p)^{14}t}
       \notag\\
 &\quad{}+\frac{C_4\beta_t\log^2 t}{\eps^{9/2}\underline p^4}
                 \sqrt{\lambda d\log(e/\lambda)/t}.
 \label{eq:finite-upper}
\end{align}
The first two terms are bounded by the third term, since
\begin{equation*}
 \frac{1}{(\eps\underline p)^3t}
 \le\frac{1}{(\eps\underline p)^4t}
 \le\frac{(\eps\underline p^2+\lambda)
                    (d\beta_t^2)^4\log^5 t}{(\eps\underline p)^{14}t},
\end{equation*}
where we use $\eps\underline p^2+\lambda\ge(\eps\underline p)^2$, $\eps\underline p\le1$, and $d\beta_t^2,\log t\ge1$. By the second condition in \Cref{eq:upper-threshold},
\[
 \frac{(\eps\underline p^2+\lambda)
                    (d\beta_t^2)^4\log^5 t}{(\eps\underline p)^{14}t}
 \le\frac{\beta_t\log^2 t}{\eps^{9/2}\underline p^4}
                    \sqrt{\lambda d\log(e/\lambda)/t}.
\]
Substituting these bounds into \Cref{eq:finite-upper} proves \Cref{thm:main-upper} with an absolute $C_1$.
\end{proof}

\section{Proof of \texorpdfstring{\Cref{thm:time-lower}}{Theorem \getrefnumber{thm:time-lower}}}
\label{app:lower-bound}

\begin{proof}
Fix an algorithm $\pi$. We construct instances in which the best server is specified by a vector of signs. We will relate the departure rate gap to the number of incorrect coordinates in the chosen server's sign vector. Let $d_0=\min\{d-1,\lfloor\log_2K\rfloor\}$, $p_0=\mu(-1/4)$, and
\[
 \eta=\min\{1/8,\sqrt{p_0}\,d_0/(2\sqrt{\lambda t})\}.
\]
There is one context $x_0$. Index $2^{d_0}$ servers by sign vectors $\v u=(u_1,\ldots,u_{d_0})\in\{-1,+1\}^{d_0}$ and instances by $\v v=(v_1,\ldots,v_{d_0})\in\{-1,+1\}^{d_0}$. Here $u_a$ and $v_a$ are the $a$th coordinates of $\v u$ and $\v v$, and $k_{\v u}$ denotes the server indexed by $\v u$. For coordinate vectors $\v e_0,\ldots,\v e_{d_0}\in\R^d$, define
\begin{align}
 \v\phi(x_0,k_{\v u})
 &=\frac1{\sqrt2}\left(\v e_0+\frac1{\sqrt{d_0}}
       \sum_{a=1}^{d_0}u_a\v e_a\right),
 \label{eq:lower-features}\\
 \v\theta^{\v v}
 &=\sqrt2\left(-\eta\v e_0+\frac{\eta}{\sqrt{d_0}}
       \sum_{a=1}^{d_0}v_a\v e_a\right).
 \label{eq:lower-family}
\end{align}
If $K>2^{d_0}$, the remaining servers duplicate these servers and retain their sign vectors. Define the Hamming distance $d_{\rm H}(\v u,\v v)=\sum_{a=1}^{d_0}\ind{u_a\ne v_a}$, the number of coordinates at which the two sign vectors differ. Then
\begin{align*}
 \ip{\v\phi(x_0,k_{\v u})}{\v\theta^{\v v}}
 &=-\eta+\frac{\eta}{d_0}\sum_{a=1}^{d_0}u_av_a
 =-\frac{2\eta}{d_0}d_{\rm H}(\v u,\v v),\\
 \norm{\v\phi(x_0,k_{\v u})}_2^2&=1,
 \qquad \norm{\v\theta^{\v v}}_2^2=4\eta^2\le1/16,
 \qquad -1/4\le\ip{\v\phi(x_0,k_{\v u})}{\v\theta^{\v v}}\le0.
\end{align*}
Thus $\v u=\v v$ gives the best server, with service probability $p^*(x_0)=1/2$, and every server has service probability at least $p_0$. The restriction $\eta\le1/8$ ensures the parameter bound above. The second bound on $\eta$ will give a KL divergence of at most $1/8$ between the history distributions of instances differing in one coordinate. The instances satisfy \Cref{ass:bounded-logistic,ass:average-workload} with $S=1/4$ and normalized capacity slack $\eps=1-2\lambda\ge1/2$. Write $\nu^{\v v}$ for the instance defined by \Cref{eq:lower-family}, and $\Pp^{\v v},\E^{\v v}$ for its probability and expectation.

We first show that $R_t^\pi(\nu^{\v v})\ge\E^{\v v}[g_{t-1}]$. To isolate the service error at round $t-1$, let $\pi'$ follow $\pi$ through round $t-2$ and use the best server in round $t-1$. Couple the queues using the same arrivals and uniform random variables $U_i\sim\Unif(0,1)$, independent across rounds and of arrivals. Under this coupling, $Q^{\pi'}(t)\ge Q^*(t)$. Also, $\pi$ and $\pi'$ coincide through round $t-2$, so $Q^\pi(t-1)=Q^{\pi'}(t-1)$. Write $p_{t-1}=1/2-g_{t-1}$. Therefore,
\begin{align}
 R_t^\pi(\nu^{\v v})
 &=\E^{\v v}[Q^\pi(t)-Q^*(t)]\notag\\
 &\ge\E^{\v v}[Q^\pi(t)-Q^{\pi'}(t)]\notag\\
 &=\E^{\v v}\bigl[\ind{Q(t-1)>0,\ p_{t-1}<U_{t-1}\le1/2}\bigr]\notag\\
 &=\E^{\v v}[g_{t-1}].
 \label{eq:lower-regret-gap}
\end{align}
The last equality follows from $\Pp(Q(t-1)>0,\ p_{t-1}<U_{t-1}\le1/2\mid\calF_{t-1})=g_{t-1}$. To obtain an instance with large regret, we lower-bound the average over all $2^{d_0}$ instances. By \Cref{eq:lower-regret-gap},
\begin{equation*}
 \sup_{\v v}R_t^\pi(\nu^{\v v})
 \ge\frac1{2^{d_0}}\sum_{\v v}R_t^\pi(\nu^{\v v})
 \ge\frac1{2^{d_0}}\sum_{\v v}\E^{\v v}[g_{t-1}].
\end{equation*}
We bound this average by pairing instances that differ in one coordinate. The learner cannot choose the correct sign with high probability in both instances of each pair. We first bound the number of service observations available for distinguishing these instances. Let $N(t-2)=\sum_{i=1}^{t-2}\ind{Q(i)>0}$. Then,
\begin{align}
 p_0\E^{\v v}[N(t-2)]
 &\le\E^{\v v}\left[\sum_{i=1}^{t-2}D(i)\right]
 =\E^{\v v}\left[\sum_{i=1}^{t-2}A(i)\right]-\E^{\v v}[Q(t-1)]
 \le\lambda t,
 \label{eq:lower-attempt-count}
\end{align}
where the first inequality follows from $p_i\ge p_0$ on nonempty rounds, the equality uses $Q(1)=0$, and the last inequality uses $Q(t-1)\ge0$.

For $a\in[d_0]$, let $\v v^{(a)}$ be $\v v$ with coordinate $a$ flipped. Let $\Pp_{t-2}^{\v v}$ be the law of the history through service in round $t-2$, including policy randomization but excluding $A(t-2)$. We will condition on $A(t-2)=1$ to ensure that the queue is nonempty at the last decision. Independence of this arrival leaves $\Pp_{t-2}^{\v v}$ unchanged under that conditioning. For $z_1,z_2\in\R$, we have
\begin{align}
 \KL\bigl(\Bern(\mu(z_1))\|\Bern(\mu(z_2))\bigr)
 &=\log\frac{1+e^{z_2}}{1+e^{z_1}}-\mu(z_1)(z_2-z_1)\notag\\
 &=\int_{z_1}^{z_2}(\mu(z)-\mu(z_1))\,dz\notag\\
 &\le\frac14\int_{\min\{z_1,z_2\}}^{\max\{z_1,z_2\}}|z-z_1|\,dz\notag\\
 &=(z_1-z_2)^2/8,
 \label{eq:canonical-logistic-kl}
\end{align}
where the first equality follows from the Bernoulli relative entropy formula, the second uses $\frac{d}{dz}\log(1+e^z)=\mu(z)$, and the inequality follows because $\mu$ is $1/4$-Lipschitz. For every server, $|\ip{\v\phi(x_0,k)}{\v\theta^{\v v}-\v\theta^{\v v^{(a)}}}|=2\eta/d_0$. By \Cref{lem:aux-kl-chain}, we have
\begin{align*}
 &\KL(\Pp_{t-2}^{\v v}\|\Pp_{t-2}^{\v v^{(a)}})\\
 &\quad=\E^{\v v}\!\left[\sum_{i=1}^{t-2}\ind{Q(i)>0}
   \KL\!\left(
    \Bern(\mu(\ip{\v\phi_i}{\v\theta^{\v v}}))
    \,\middle\|\,
    \Bern(\mu(\ip{\v\phi_i}{\v\theta^{\v v^{(a)}}}))
   \right)\right]\\
 &\quad\le\E^{\v v}\left[\sum_{i=1}^{t-2}
       \ind{Q(i)>0}\frac18\left(\frac{2\eta}{d_0}\right)^2\right]\\
 &\quad=\frac{\eta^2}{2d_0^2}\E^{\v v}[N(t-2)]\\
 &\quad\le\frac{\lambda t\eta^2}{2p_0d_0^2}\le\frac18.
\end{align*}
The first equality follows because, given the same history, only the distribution of $D(i)$ depends on $\v v$. The first inequality follows from \Cref{eq:canonical-logistic-kl}, the second follows from \Cref{eq:lower-attempt-count}, and the last follows from $\eta\le\sqrt{p_0}d_0/(2\sqrt{\lambda t})$. Pinsker's inequality (\Cref{lem:aux-pinsker}) therefore gives
\begin{equation*}
 \norm{\Pp_{t-2}^{\v v}-\Pp_{t-2}^{\v v^{(a)}}}_{\rm TV}
 \le\sqrt{\frac12\KL(\Pp_{t-2}^{\v v}\|\Pp_{t-2}^{\v v^{(a)}})}
 \le\frac14.
\end{equation*}

On $A(t-2)=1$, we have $Q(t-1)>0$. Let $\v u_{t-1}$ be the sign vector of the chosen server, with $a$th coordinate $u_{t-1,a}$. To compare its distribution in the two instances, let $H_{t-2}$ denote the history with law $\Pp_{t-2}^{\v v}$, and define $f_a(h)=\Pp_\pi(u_{t-1,a}=+1\mid H_{t-2}=h,\ A(t-2)=1)$. Here the probability is over the algorithm's randomization after observing the history $h$ and the arrival. The same algorithm is used in both instances, so $f_a$ is the same function in both and $0\le f_a(h)\le1$. Since $A(t-2)$ is independent of $H_{t-2}$, the tower rule gives $\Pp^{\v v}(u_{t-1,a}=+1\mid A(t-2)=1)=\E^{\v v}[f_a(H_{t-2})]$, and the same identity holds under $\v v^{(a)}$. The correct sign in coordinate $a$ is opposite in the two instances. For $v_a=+1$, we have
\begin{align*}
 &\Pp^{\v v}(u_{t-1,a}\ne v_a\mid A(t-2)=1)
 +\Pp^{\v v^{(a)}}(u_{t-1,a}\ne v_a^{(a)}\mid A(t-2)=1)\\
 &\quad=1-\Pp^{\v v}(u_{t-1,a}=+1\mid A(t-2)=1)
          +\Pp^{\v v^{(a)}}(u_{t-1,a}=+1\mid A(t-2)=1)\\
 &\quad=1-\E^{\v v}[f_a(H_{t-2})]
          +\E^{\v v^{(a)}}[f_a(H_{t-2})]\\
 &\quad\ge1-\norm{\Pp_{t-2}^{\v v}-\Pp_{t-2}^{\v v^{(a)}}}_{\rm TV}
 \ge3/4.
\end{align*}
The first inequality follows from \Cref{lem:aux-pinsker} applied to $f_a$, since $0\le f_a\le1$. Summing over the $2^{d_0-1}$ pairs in each coordinate yields
\begin{align}
 \frac1{2^{d_0}}\sum_{\v v}
 \E^{\v v}[d_{\rm H}(\v u_{t-1},\v v)\mid A(t-2)=1]
 &=\frac1{2^{d_0}}\sum_{a=1}^{d_0}\sum_{\v v}
   \Pp^{\v v}(u_{t-1,a}\ne v_a\mid A(t-2)=1)\notag\\
 &\ge\frac1{2^{d_0}}\sum_{a=1}^{d_0}2^{d_0-1}\cdot\frac34
 =\frac{3d_0}{8}.
 \label{eq:lower-sign-errors}
\end{align}
On $A(t-2)=1$, we have
\begin{align*}
 g_{t-1}
 &=p_{t-1}^*-p_{t-1}
 =\frac12-\mu(-\frac{2\eta}{d_0}d_{\rm H}(\v u_{t-1},\v v))\\
 &=\int_{-2\eta d_{\rm H}(\v u_{t-1},\v v)/d_0}^{0}\dot\mu(z)\,dz\\
 &\ge\int_{-2\eta d_{\rm H}(\v u_{t-1},\v v)/d_0}^{0}\dot\mu(1/4)\,dz\\
 &=\frac{2\dot\mu(1/4)\eta}{d_0}d_{\rm H}(\v u_{t-1},\v v),
\end{align*}
where the inequality follows from $\dot\mu(z)\ge\dot\mu(1/4)$ for $-1/4\le z\le0$. Thus
\begin{align}
 \frac1{2^{d_0}}\sum_{\v v}\E^{\v v}[g_{t-1}]
 &\ge\frac\lambda{2^{d_0}}\sum_{\v v}
             \E^{\v v}[g_{t-1}\mid A(t-2)=1]\notag\\
 &\ge\frac{2\lambda\dot\mu(1/4)\eta}{d_0}\cdot
       \frac1{2^{d_0}}\sum_{\v v}
       \E^{\v v}[d_{\rm H}(\v u_{t-1},\v v)\mid A(t-2)=1]\notag\\
 &\ge\frac{3\dot\mu(1/4)}4\lambda\eta,
 \label{eq:lower-service-loss}
\end{align}
where the first inequality uses $g_{t-1}\ge0$ and $\Pp(A(t-2)=1)=\lambda$, the second uses the preceding bound on $g_{t-1}$, and the last follows from \Cref{eq:lower-sign-errors}. Combining \Cref{eq:lower-regret-gap,eq:lower-service-loss} and substituting the definition of $\eta$ gives
\begin{align*}
 \sup_{\v v}R_t^\pi(\nu^{\v v})
 &\ge\frac1{2^{d_0}}\sum_{\v v}R_t^\pi(\nu^{\v v})\\
 &\ge\frac{3\dot\mu(1/4)}4\lambda\eta\\
 &=\frac{3\dot\mu(1/4)}4
       \min\{\lambda/8,\sqrt{p_0}\,d_0\sqrt{\lambda/t}/2\}.
\end{align*}
Since $d-1\ge d/2$ and $\lfloor\log_2K\rfloor\ge\log K/2$, we have $d_0\ge\min\{d,\log K\}/2$. Choosing a sufficiently large absolute constant $C_0$ proves the theorem.
\end{proof}

\section{Proof of \texorpdfstring{\Cref{thm:heavy-traffic-lower}}{Theorem \getrefnumber{thm:heavy-traffic-lower}}}
\label{app:heavy-traffic-lower}

We first prove two properties of the queue under the best server. They will allow us to add the effects of service errors over the last $\Theta(\eps^{-2})$ rounds in the proof of \Cref{thm:heavy-traffic-lower}.

\begin{lemma}
\label{lem:heavy-traffic-queue}
Let $0<\eps\le1/8$, $\lambda=1/4$, and $p^*=\lambda/(1-\eps)$. Consider a queue with Bernoulli arrivals of mean $\lambda$ and service probability $p^*$ on nonempty rounds. Arrivals and service outcomes are independent, with arrivals after service. Let
\[
 q_0=\left\lceil5/(8\eps)\right\rceil+1,\qquad
 t_0=\left\lceil60/\eps^2\right\rceil,\qquad
 \ell=\left\lfloor1/(16\eps^2)\right\rfloor.
\]
\begin{enumerate}
\renewcommand{\labelenumi}{(\roman{enumi})}
\item If $Q(1)=0$, then $\Pp(Q(i)\ge q_0)\ge1/20$ for every $i\ge t_0$.
\item If $Q(1)\ge q_0-1$, then $\Pp(Q(j)>0\text{ for every }j\in[1,\ell+1])\ge1/2$.
\end{enumerate}
\end{lemma}

\begin{proof}
We first prove (i). Let $\varpi$ be the stationary distribution, with probability mass function $\varpi(q)$. If the initial queue length satisfies $\Pp(Q(1)=q)=\varpi(q)$ for every $q\ge0$, then $\Pp(Q(j)=q)=\varpi(q)$ at every round $j$. We write $Q(1)\sim\varpi$ for this initialization and call the resulting queue a stationary queue. Write $\Pp_q,\E_q$ for probability and expectation when $Q(1)=q$, and $\Pp_\varpi,\E_\varpi$ when $Q(1)\sim\varpi$. Define $\tau_0=\inf\{j\ge1:Q(j)=0\}$.

Couple $Q$ with $Q(1)=0$ and $Q^\varpi$ with $Q^\varpi(1)\sim\varpi$, using the same arrivals $A(j)$ and uniform random variables $U_j\sim\Unif(0,1)$. Since $Q(1)=0\le Q^\varpi(1)$, this coupling gives $Q(j)\le Q^\varpi(j)$ for every $j$. Once $Q^\varpi$ becomes zero, both queues are zero and agree thereafter. Thus the probability that they differ at round $i$ is at most $\Pp_\varpi(\tau_0>i)$, and
\begin{align*}
 \Pp_0(Q(i)\ge q_0)
 &\ge\Pp_\varpi(Q(i)\ge q_0)-\Pp_\varpi(\tau_0>i)\\
 &=\underbrace{\Pp_\varpi(Q(1)\ge q_0)}_{I_1}
   -\underbrace{\Pp_\varpi(\tau_0>i)}_{I_2},
\end{align*}
where the last equality follows from stationarity. We show that $I_1\ge1/10$ and $I_2\le1/20$ for $i\ge t_0$.

To bound $I_1$, we compute $\varpi$. For $q\ge1$, the transition probabilities from $q$ to $q-1$ and $q+1$ are $p^*(1-\lambda)$ and $(1-p^*)\lambda$, respectively. From $q=0$, the probability of moving to $q=1$ is $\lambda$. The balance equations between adjacent queue lengths are
\begin{align*}
 \varpi(0)\lambda&=\varpi(1)p^*(1-\lambda),\\
 \varpi(q)(1-p^*)\lambda&=\varpi(q+1)p^*(1-\lambda),
 \qquad q\ge1.
\end{align*}
Together with $\sum_{q=0}^\infty\varpi(q)=1$, these give
\begin{align}
 \varpi(0)&=1-\lambda/p^*=\eps,\notag\\
 \Pp_\varpi(Q(1)\ge q)
 &=(1-\eps)\left(\frac{(1-p^*)\lambda}{p^*(1-\lambda)}\right)^{q-1}
 =(1-\eps)\left(1-\frac{\eps}{1-\lambda}\right)^{q-1},
 \qquad q\ge1.
 \label{eq:traffic-stationary-tail}
\end{align}
Therefore,
\begin{align*}
 I_1
 &\ge\frac78\exp\left(-\frac{2\eps}{1-\lambda}
                    \left(\frac5{8\eps}+1\right)\right)
 \ge\frac78e^{-2}>\frac1{10},
\end{align*}
where the first inequality follows from $\log(1-z)\ge-2z$ for $z\in[0,1/2]$, $\eps/(1-\lambda)\le1/6$, and $q_0-1\le5/(8\eps)+1$. The second inequality follows from $2\eps(5/(8\eps)+1)/(1-\lambda)=5/3+8\eps/3\le2$.

To bound $I_2$, we bound $\E_\varpi[\tau_0-1]$ and apply Markov's inequality. We first show that $\E_q[\tau_0-1]\le q/(p^*\eps)$ and then average over $Q(1)\sim\varpi$. Before $\tau_0$, $\E[Q(i+1)-Q(i)\mid\calF_i]=\lambda-p^*=-p^*\eps$. For every $j\ge1$,
\begin{align*}
 \E_q[Q(\tau_0\wedge j)]
 &=q+\sum_{i=1}^{j-1}\E_q\!\left[
      \ind{i<\tau_0}(Q(i+1)-Q(i))\right]\\
 &=q-p^*\eps\sum_{i=1}^{j-1}\Pp_q(i<\tau_0)\\
 &=q-p^*\eps\,\E_q[(\tau_0\wedge j)-1],
\end{align*}
where the second equality follows by conditioning on $\calF_i$, and the last uses $\sum_{i=1}^{j-1}\ind{i<\tau_0}=(\tau_0\wedge j)-1$. Since $Q(\tau_0\wedge j)\ge0$, we have $\E_q[(\tau_0\wedge j)-1]\le q/(p^*\eps)$. Letting $j\to\infty$ gives
\begin{equation*}
 \E_q[\tau_0-1]
 =\lim_{j\to\infty}\E_q[(\tau_0\wedge j)-1]
 \le\frac{q}{p^*\eps}.
\end{equation*}
By \Cref{eq:traffic-stationary-tail},
\begin{align*}
 \E_\varpi[Q(1)]
 &=\sum_{q=1}^\infty\Pp_\varpi(Q(1)\ge q)\\
 &=(1-\eps)\sum_{q=1}^\infty
       \left(1-\frac{\eps}{1-\lambda}\right)^{q-1}
 =\frac{(1-\lambda)(1-\eps)}{\eps}.
\end{align*}
Conditioning on $Q(1)$ and applying the bound for each $q$ yields
\begin{align*}
 \E_\varpi[\tau_0-1]
 &=\sum_{q=0}^\infty\varpi(q)\E_q[\tau_0-1]\\
 &\le\frac1{p^*\eps}\sum_{q=0}^\infty q\varpi(q)\\
 &=\frac{\E_\varpi[Q(1)]}{p^*\eps}
 =\frac{3(1-\eps)^2}{\eps^2}.
\end{align*}
For $i\ge t_0$, we have
\begin{align*}
 I_2
 &\le\frac{\E_\varpi[\tau_0-1]}{i-1}
 \le\frac{3(1-\eps)^2}{60-\eps^2}
 \le\frac1{20},
\end{align*}
where the first inequality follows from Markov's inequality, the second follows from $i\ge60/\eps^2$, and the last follows from $60(1-\eps)^2\le60-\eps^2$. Combining the bounds for $I_1$ and $I_2$ gives $\Pp_0(Q(i)\ge q_0)\ge I_1-I_2\ge1/20$, proving (i).

We next prove (ii). Suppose $Q(1)\ge q_0-1$. Since $\{Q(j)>0\text{ for every }j\in[1,\ell+1]\}=\{\tau_0>\ell+1\}$, it suffices to show that $\Pp(\tau_0\le\ell+1)\le1/2$. Before $\tau_0$, $\E[Q(i+1)-Q(i)\mid\calF_i]=-p^*\eps$. Subtracting these conditional expectations from the increments up to $\tau_0$ gives
\[
 M_j=\sum_{i=1}^j\ind{i<\tau_0}
            \bigl(Q(i+1)-Q(i)+p^*\eps\bigr),\qquad M_0=0.
\]
We have $\E[M_j-M_{j-1}\mid\calF_j]=0$ and $|M_j-M_{j-1}|\le2$. On $\{\tau_0\le\ell+1\}$,
\[
 M_\ell=M_{\tau_0-1}
 =-Q(1)+p^*\eps(\tau_0-1)
 \le-\frac5{8\eps}+\frac27\eps\frac1{16\eps^2}
 =-\frac{17}{28\eps},
\]
where the first equality follows from $\ind{i<\tau_0}=0$ for $i\ge\tau_0$, the second follows from $Q(\tau_0)=0$, and the inequality uses $p^*\le2/7$ and $\tau_0-1\le\ell$. Therefore,
\begin{align*}
 \Pp(\tau_0\le\ell+1)
 &=\Pp\left(\tau_0\le\ell+1,\ M_\ell\le-\frac{17}{28\eps}\right)\\
 &\le\Pp\left(M_\ell-M_0\le-\frac{17}{28\eps}\right)\\
 &\le\exp\left(-\frac{(17/(28\eps))^2}{8\ell}\right)<1/2,
\end{align*}
where the first equality uses $M_\ell\le-17/(28\eps)$ on $\{\tau_0\le\ell+1\}$. The second inequality follows from the Azuma-Hoeffding inequality (\Cref{lem:aux-azuma}) with $c_j=2$ and $z=17/(28\eps)$, and the last inequality uses $\ell\le1/(16\eps^2)$. Thus $\Pp(Q(j)>0\text{ for every }j\in[1,\ell+1])\ge1/2$.
\end{proof}

\begin{proof}[Proof of \Cref{thm:heavy-traffic-lower}]
Fix an algorithm $\pi$. We use one context $x_0$ and the features in \Cref{eq:lower-features}, with $d_0=\min\{d-1,\lfloor\log_2K\rfloor\}$. If $K>2^{d_0}$, let the remaining servers duplicate these servers. Set
\[
 \lambda=1/4,\qquad p^*=\lambda/(1-\eps),\qquad
 z^*=\log(p^*/(1-p^*)),\qquad
 \eta=\min\{1/8,d_0/\sqrt t\}/20,
\]
and define
\[
 \v\theta^{\v v}=\sqrt2\left((z^*-\eta)\v e_0+
             \frac{\eta}{\sqrt{d_0}}\sum_{a=1}^{d_0}v_a\v e_a\right),
 \qquad \v v\in\{-1,+1\}^{d_0}.
\]
Then
\begin{equation}
 \ip{\v\phi(x_0,k_{\v u})}{\v\theta^{\v v}}
 =z^*-\frac{2\eta}{d_0}d_{\rm H}(\v u,\v v),
 \qquad \max_k\mu(\ip{\v\phi(x_0,k)}{\v\theta^{\v v}})=p^*.
 \label{eq:heavy-traffic-logit}
\end{equation}
Since $p^*\in[1/4,2/7]$, $|z^*|\le\log3$, and $\eta\le1/160$,
\[
 \norm{\v\theta^{\v v}}_2^2
 =2\bigl((z^*-\eta)^2+\eta^2\bigr)
 \le2\bigl((\log3+1/160)^2+(1/160)^2\bigr)<4.
\]
The instances satisfy \Cref{ass:bounded-logistic,ass:average-workload} with $S=2$ and normalized capacity slack $1-\lambda/p^*=\eps$. Write $\nu^{\v v}$ for each instance and $\Pp^{\v v},\E^{\v v}$ for its probability and expectation.

Let $q_0,t_0,\ell$ be as in \Cref{lem:heavy-traffic-queue}. We will sum the contributions from the last $\ell=\Theta(\eps^{-2})$ rounds. For $i=0,\ldots,t-1$, let $\pi_i$ follow $\pi$ through round $i$ and use the best server thereafter. Since $Q^{\pi_0}(t)=Q^*(t)$ and $Q^{\pi_{t-1}}(t)=Q^\pi(t)$,
\begin{equation}
 R_t^\pi(\nu^{\v v})
 =\sum_{i=1}^{t-1}\E^{\v v}
       [Q^{\pi_i}(t)-Q^{\pi_{i-1}}(t)].
 \label{eq:traffic-switching}
\end{equation}
For $i\in[t-\ell,t-1]$, we will show that
\[
 \E^{\v v}[Q^{\pi_i}(t)-Q^{\pi_{i-1}}(t)]
 \ge\frac12\E^{\v v}[g_i\ind{Q^\pi(i)\ge q_0}].
\]
Write $g_i=p^*-p_i$. Couple the queues using the same arrivals $A(j)$ and uniform random variables $U_j\sim\Unif(0,1)$. In round $i$, both queues have length $Q^\pi(i)$, but $\pi_i$ and $\pi_{i-1}$ use service probabilities $p_i$ and $p^*$, respectively. After round $i$, both queues use $p^*$. Under this coupling, $Q^{\pi_i}(j)\ge Q^{\pi_{i-1}}(j)$ for $j\ge i+1$.

Define $\calE_i=\{Q^\pi(i)\ge q_0,\ p_i<U_i\le p^*\}$. On $\calE_i$, service succeeds under $\pi_{i-1}$ and fails under $\pi_i$. Hence,
\[
 Q^{\pi_{i-1}}(i+1)=Q^\pi(i)-1+A(i)\ge q_0-1,\qquad
 Q^{\pi_i}(i+1)=Q^{\pi_{i-1}}(i+1)+1.
\]
On $\calE_i$, if $Q^{\pi_{i-1}}(j)>0$ for every $j\in[i+1,t]$, then both queues have the same departures, so $Q^{\pi_i}(t)-Q^{\pi_{i-1}}(t)=1$. Since this difference is nonnegative on every sample path,
\begin{equation*}
 Q^{\pi_i}(t)-Q^{\pi_{i-1}}(t)
 \ge\ind{\calE_i}
     \ind{Q^{\pi_{i-1}}(j)>0\text{ for every }j\in[i+1,t]}.
\end{equation*}
Since $Q^{\pi_{i-1}}(i+1)\ge q_0-1$ on $\calE_i$ and $t-i-1\le\ell$, \Cref{lem:heavy-traffic-queue}-(ii) gives, on $\calE_i$,
\begin{equation*}
 \Pp^{\v v}\!\left(
 Q^{\pi_{i-1}}(j)>0\text{ for every }j\in[i+1,t]
 \,\middle|\,\calF_i\vee\sigma(U_i,A(i))\right)\ge1/2.
\end{equation*}
Taking expectations and applying the tower rule gives
\begin{align*}
 &\E^{\v v}[Q^{\pi_i}(t)-Q^{\pi_{i-1}}(t)]\\
 &\quad\ge\E^{\v v}\!\left[\ind{\calE_i}
       \ind{Q^{\pi_{i-1}}(j)>0\text{ for every }j\in[i+1,t]}\right]\\
 &\quad=\E^{\v v}\!\left[\ind{\calE_i}
       \Pp^{\v v}\!\left(
        Q^{\pi_{i-1}}(j)>0\text{ for every }j\in[i+1,t]
        \,\middle|\,\calF_i\vee\sigma(U_i,A(i))\right)\right]\\
 &\quad\ge\frac12\Pp^{\v v}(\calE_i)
 =\frac12\E^{\v v}[g_i\ind{Q^\pi(i)\ge q_0}],
\end{align*}
where the last equality follows from $\Pp(p_i<U_i\le p^*\mid\calF_i)=p^*-p_i=g_i$. Since $Q^{\pi_i}(t)-Q^{\pi_{i-1}}(t)\ge0$ for every $i$, \Cref{eq:traffic-switching} gives
\begin{align}
 R_t^\pi(\nu^{\v v})
 &\ge\sum_{i=t-\ell}^{t-1}
       \E^{\v v}[Q^{\pi_i}(t)-Q^{\pi_{i-1}}(t)]\notag\\
 &\ge\frac12\sum_{i=t-\ell}^{t-1}
       \E^{\v v}[g_i\ind{Q^\pi(i)\ge q_0}].
 \label{eq:traffic-gap-reduction}
\end{align}

Averaging \Cref{eq:traffic-gap-reduction} over $\v v$ gives
\begin{align*}
 \sup_{\v v}R_t^\pi(\nu^{\v v})
 &\ge\frac1{2^{d_0}}\sum_{\v v}R_t^\pi(\nu^{\v v})\\
 &\ge\frac12\sum_{i=t-\ell}^{t-1}
      \frac1{2^{d_0}}\sum_{\v v}
      \E^{\v v}[g_i\ind{Q^\pi(i)\ge q_0}].
\end{align*}
Fix $i\in[t-\ell,t-1]$. We bound the inner sum by comparing instances that differ in one coordinate. Let $\v u_i$ be the sign vector of the chosen server $k_i$, setting $k_j=1$ whenever $Q(j)=0$. Let $\Pp_i^{\v v}$ be the joint distribution of $((k_j,D(j),A(j))_{j=1}^{i-1},k_i)$ under algorithm $\pi$ in instance $\nu^{\v v}$. For $a\in[d_0]$, let $\v v^{(a)}$ be $\v v$ with coordinate $a$ flipped. For $v_a=+1$, we have
\begin{align}
 &\Pp^{\v v}(Q^\pi(i)\ge q_0,\ u_{i,a}\ne v_a)
 +\Pp^{\v v^{(a)}}(Q^\pi(i)\ge q_0,\ u_{i,a}\ne v_a^{(a)})\notag\\
 &\quad=\Pp^{\v v}(Q^\pi(i)\ge q_0)
       -\Pp^{\v v}(Q^\pi(i)\ge q_0,\ u_{i,a}=+1)\notag\\
 &\qquad{}+\Pp^{\v v^{(a)}}(Q^\pi(i)\ge q_0,\ u_{i,a}=+1)\notag\\
 &\quad\ge\Pp^{\v v}(Q^\pi(i)\ge q_0)
      -\norm{\Pp_i^{\v v}-\Pp_i^{\v v^{(a)}}}_{\rm TV}.
 \label{eq:traffic-paired-errors}
\end{align}
The inequality follows from the definition of total variation applied to $\{Q^\pi(i)\ge q_0,\ u_{i,a}=+1\}$. We bound the first term on the right-hand side below by $1/20$ and the second term above by $1/40$.

For the first term, $t\ge64/\eps^2$ gives $t-\ell\ge(64-1/16)/\eps^2\ge60/\eps^2+1\ge t_0$. Since $Q^\pi(i)\ge Q^*(i)$, applying \Cref{lem:heavy-traffic-queue}-(i) gives
\[
 \Pp^{\v v}(Q^\pi(i)\ge q_0)
 \ge\Pp^{\v v}(Q^*(i)\ge q_0)\ge1/20,\qquad i\ge t_0.
\]

For the second term, flipping coordinate $a$ gives $|\ip{\v\phi(x_0,k)}{\v\theta^{\v v}-\v\theta^{\v v^{(a)}}}|=2\eta/d_0$. By the chain rule for relative entropy (\Cref{lem:aux-kl-chain}),
\begin{align*}
 &\KL(\Pp_i^{\v v}\|\Pp_i^{\v v^{(a)}})\\
 &\quad=\E^{\v v}\!\left[\sum_{j=1}^{i-1}\ind{Q(j)>0}
   \KL\!\left(
    \Bern(\mu(\ip{\v\phi_j}{\v\theta^{\v v}}))
    \,\middle\|\,
    \Bern(\mu(\ip{\v\phi_j}{\v\theta^{\v v^{(a)}}}))
   \right)\right]\\
 &\quad\le\E^{\v v}\left[\sum_{j=1}^{i-1}
       \ind{Q(j)>0}\frac18\left(\frac{2\eta}{d_0}\right)^2\right]\\
 &\quad\le\frac{t\eta^2}{2d_0^2}\le\frac1{800},
\end{align*}
where the first inequality follows from \Cref{eq:canonical-logistic-kl}, the second uses $i\le t$, and the last uses $\eta\le d_0/(20\sqrt t)$. Pinsker's inequality (\Cref{lem:aux-pinsker}) then gives
\begin{equation*}
 \norm{\Pp_i^{\v v}-\Pp_i^{\v v^{(a)}}}_{\rm TV}
 \le\sqrt{\frac12\KL(\Pp_i^{\v v}\|\Pp_i^{\v v^{(a)}})}
 \le\frac1{40}.
\end{equation*}

Substituting both bounds into \Cref{eq:traffic-paired-errors} gives
\begin{equation*}
 \Pp^{\v v}(Q^\pi(i)\ge q_0,\ u_{i,a}\ne v_a)
 +\Pp^{\v v^{(a)}}(Q^\pi(i)\ge q_0,\ u_{i,a}\ne v_a^{(a)})
 \ge\frac1{20}-\frac1{40}=\frac1{40}.
\end{equation*}
Summing over $2^{d_0-1}$ pairs per coordinate gives
\begin{align}
 \frac1{2^{d_0}}\sum_{\v v}\E^{\v v}
       [\ind{Q^\pi(i)\ge q_0}d_{\rm H}(\v u_i,\v v)]
 &=\frac1{2^{d_0}}\sum_{a=1}^{d_0}\sum_{\v v}
       \Pp^{\v v}(Q^\pi(i)\ge q_0,\ u_{i,a}\ne v_a)\notag\\
 &\ge\frac1{2^{d_0}}\sum_{a=1}^{d_0}\frac{2^{d_0-1}}{40}
 =\frac{d_0}{80}.
 \label{eq:traffic-sign-errors}
\end{align}
On $\{Q^\pi(i)\ge q_0\}$, \Cref{eq:heavy-traffic-logit} gives
\begin{align*}
 g_i
 &=p^*-p_i
 =\mu(z^*)-\mu(z^*-\frac{2\eta}{d_0}d_{\rm H}(\v u_i,\v v))\\
 &=\int_{z^*-2\eta d_{\rm H}(\v u_i,\v v)/d_0}^{z^*}\dot\mu(z)\,dz\\
 &\ge\frac{2\eta}{d_0(2+e^2+e^{-2})}d_{\rm H}(\v u_i,\v v),
\end{align*}
where the inequality follows from \Cref{ass:bounded-logistic}-2 with $S=2$. Taking expectations and applying \Cref{eq:traffic-sign-errors} gives
\begin{equation}
 \frac1{2^{d_0}}\sum_{\v v}
       \E^{\v v}[g_i\ind{Q^\pi(i)\ge q_0}]
 \ge\frac{2\eta}{d_0(2+e^2+e^{-2})}\cdot\frac{d_0}{80}
 =\frac{\eta}{40(2+e^2+e^{-2})}.
 \label{eq:traffic-service-loss}
\end{equation}
Substituting \Cref{eq:traffic-service-loss} into \Cref{eq:traffic-gap-reduction} yields
\begin{align*}
 \sup_{\v v}R_t^\pi(\nu^{\v v})
 &\ge\frac12\sum_{i=t-\ell}^{t-1}
       \frac1{2^{d_0}}\sum_{\v v}
       \E^{\v v}[g_i\ind{Q^\pi(i)\ge q_0}]\\
 &\ge\frac{\ell\eta}{80(2+e^2+e^{-2})}\\
 &\ge\frac1{51200(2+e^2+e^{-2})}
       \min\{1/(8\eps^2),d_0/(\eps^2\sqrt t)\},
\end{align*}
where the last inequality uses $\ell\ge1/(32\eps^2)$ and $\eta=\min\{1/8,d_0/\sqrt t\}/20$. Since $d_0\ge\min\{d,\log K\}/2$, choosing a sufficiently large absolute constant $C_0$ proves the theorem.
\end{proof}

\section{Auxiliary Lemmas}
\label{app:auxiliary}

\begin{lemma}[Theorem~1, \citealt{abbasi2011improved}]
\label{lem:aux-self-normalized}
Let $\{\calF_i\}_{i\ge1}$ be a filtration, $\v\phi_i\in\R^d$ be $\calF_i$-measurable, and $\eta_i$ be $\calF_{i+1}$-measurable. Suppose that, for some $\sigma>0$ and every $z\in\R$, $\E[e^{z\eta_i}\mid\calF_i]\le e^{\sigma^2z^2/2}$. For a fixed positive definite matrix $\m V_0$, define $\m V_t=\m V_0+\sum_{i=1}^t\v\phi_i\v\phi_i^\top$. For every $\delta\in(0,1)$, with probability at least $1-\delta$, simultaneously for all $t\ge0$,
\begin{equation*}
 \norm{\sum_{i=1}^t\eta_i\v\phi_i}_{\m V_t^{-1}}^2
 \le\sigma^2\left(\log\frac{\det\m V_t}{\det\m V_0}
                         +2\log(1/\delta)\right).
\end{equation*}
\end{lemma}

\begin{lemma}
\label{lem:aux-mle}
Under \Cref{ass:bounded-logistic}, let $\widehat{\v\theta}_t$ be the regularized maximum likelihood estimator in \Cref{eq:estimator}, and define
\begin{equation*}
 \v G_t=\sum_{i=1}^t
       \bigl(D(i)-\mu(\ip{\v\phi_i}{\v\theta^*})\bigr)\v\phi_i.
\end{equation*}
Then, for every $t\ge0$, we have
\begin{equation*}
 \norm{\widehat{\v\theta}_t-\v\theta^*}_{\m V_t}
 \le\kappa\norm{\v G_t}_{\m V_t^{-1}}+\sqrt\kappa S.
\end{equation*}
\end{lemma}

\begin{proof}
Let $f_t$ be the objective in \Cref{eq:estimator}, and write $\v\Delta_t=\widehat{\v\theta}_t-\v\theta^*$. For every $\v\theta\in\calB$, we have
\begin{align*}
 \nabla f_t(\v\theta)
 &=\v\theta+\sum_{i=1}^t
   \bigl(\mu(\ip{\v\phi_i}{\v\theta})-D(i)\bigr)\v\phi_i,\\
 \nabla^2f_t(\v\theta)
 &=\m I+\sum_{i=1}^t
   \dot\mu(\ip{\v\phi_i}{\v\theta})\v\phi_i\v\phi_i^\top
 \succeq\m V_t/\kappa,
\end{align*}
where the inequality follows from \Cref{ass:bounded-logistic}. Since $\widehat{\v\theta}_t$ minimizes $f_t$ over $\calB$ and $\v\theta^*\in\calB$, $\ip{\v\Delta_t}{\nabla f_t(\widehat{\v\theta}_t)}\le0$. Therefore,
\begin{align*}
 \norm{\v\Delta_t}_{\m V_t}^2/\kappa
 &\le\int_0^1
   \v\Delta_t^\top\nabla^2f_t(\v\theta^*+z\v\Delta_t)
   \v\Delta_t\,dz\\
 &=\ip{\v\Delta_t}{
      \nabla f_t(\widehat{\v\theta}_t)-\nabla f_t(\v\theta^*)}\\
 &\le\ip{\v\Delta_t}{\v G_t-\v\theta^*}\\
 &\le\norm{\v\Delta_t}_{\m V_t}
   \bigl(\norm{\v G_t}_{\m V_t^{-1}}
         +\norm{\v\theta^*}_{\m V_t^{-1}}\bigr)\\
 &\le\norm{\v\Delta_t}_{\m V_t}
   \bigl(\norm{\v G_t}_{\m V_t^{-1}}+S/\sqrt\kappa\bigr),
\end{align*}
where the first inequality follows from $\v\theta^*+z\v\Delta_t\in\calB$ for $z\in[0,1]$ and $\nabla^2f_t(\v\theta)\succeq\m V_t/\kappa$ on $\calB$, the second inequality follows from $\ip{\v\Delta_t}{\nabla f_t(\widehat{\v\theta}_t)}\le0$ and $\nabla f_t(\v\theta^*)=\v\theta^*-\v G_t$, the third inequality follows from Cauchy-Schwarz and the triangle inequality, and the last inequality follows from $\m V_t\succeq\kappa\m I$ and $\norm{\v\theta^*}_2\le S$. Dividing by $\norm{\v\Delta_t}_{\m V_t}/\kappa$ gives the result when $\v\Delta_t\ne0$. The result is immediate when $\v\Delta_t=0$.
\end{proof}

\begin{lemma}[Lemma~11, \citealt{abbasi2011improved}]
\label{lem:aux-elliptical}
Let $\m V_0$ be positive definite and $\v\phi_1,\ldots,\v\phi_t\in\R^d$. Define $\m V_i=\m V_0+\sum_{j=1}^i\v\phi_j\v\phi_j^\top$. Then
\begin{equation*}
 \sum_{i=1}^t\min\{1,\norm{\v\phi_i}_{\m V_{i-1}^{-1}}^2\}
 \le2\log\frac{\det\m V_t}{\det\m V_0}.
\end{equation*}
In particular, if $\m V_0=\kappa\m I$, $\kappa\ge1$, and $\norm{\v\phi_i}_2\le1$ for every $i$, then
\begin{equation*}
 \sum_{i=1}^t\norm{\v\phi_i}_{\m V_{i-1}^{-1}}^2
 \le2d\log(1+t/(d\kappa)).
\end{equation*}
\end{lemma}

\begin{lemma}[Hoeffding's lemma]
\label{lem:aux-hoeffding}
Let $Y$ be a random variable with $a\le Y\le b$ almost surely, and let $\calG$ be a sigma-field. For every $z\in\R$,
\begin{equation*}
 \E\!\left[e^{z(Y-\E[Y\mid\calG])}\mid\calG\right]
 \le e^{z^2(b-a)^2/8}.
\end{equation*}
\end{lemma}

\begin{lemma}[Azuma-Hoeffding inequality]
\label{lem:aux-azuma}
Let $\{\calF_i\}_{i\ge1}$ be a filtration and $\eta_i$ be $\calF_{i+1}$-measurable with $\E[\eta_i\mid\calF_i]=0$ and $|\eta_i|\le c_i$, where $c_i>0$ are constants. For $1\le t_1\le t_2$ and $z\ge0$,
\begin{equation*}
 \Pp\!\left(\max_{t_1\le j\le t_2}\sum_{i=t_1}^{j}\eta_i\ge z
                  \,\middle|\,\calF_{t_1}\right)
 \le\exp\!\left(-\frac{z^2}{2\sum_{i=t_1}^{t_2}c_i^2}\right).
\end{equation*}
The same bound holds for $\min_{t_1\le j\le t_2}\sum_{i=t_1}^{j}\eta_i\le-z$ and for either tail of the terminal sum $\sum_{i=t_1}^{t_2}\eta_i$.
\end{lemma}

\begin{lemma}[Chain rule for relative entropy]
\label{lem:aux-kl-chain}
Let $P,Q$ be probability laws of $(Y_1,\ldots,Y_t)$ with regular conditional laws $P_i(\cdot\mid Y_1,\ldots,Y_{i-1})$ and $Q_i(\cdot\mid Y_1,\ldots,Y_{i-1})$, respectively. Then
\begin{equation*}
 \KL(P\|Q)=\sum_{i=1}^t\E_P\!\left[
 \KL\bigl(P_i(\cdot\mid Y_1,\ldots,Y_{i-1})
       \|Q_i(\cdot\mid Y_1,\ldots,Y_{i-1})\bigr)\right],
\end{equation*}
where the $i=1$ term uses the marginal laws of $Y_1$.
\end{lemma}

\begin{lemma}[Pinsker's inequality]
\label{lem:aux-pinsker}
For probability measures $P,Q$ on the same measurable space, define $\norm{P-Q}_{\rm TV}=\sup_{\calE}|P(\calE)-Q(\calE)|$, where the supremum is over measurable events. For every measurable function $f$ with $0\le f\le1$,
\begin{equation*}
 |\E_P[f]-\E_Q[f]|
 \le\norm{P-Q}_{\rm TV}
 \le\sqrt{\KL(P\|Q)/2}.
\end{equation*}
\end{lemma}

\section{Experimental Details}
\label{app:experiments}

\para{Experimental setup.} We use $d=5$, $K=10$, $\lambda=0.6$, $T=2000$, and one parameter $\v\theta^*$ with $\norm{\v\theta^*}_2=10$, shared by all servers and settings. Each job has $K$ independently sampled features, and jobs are independent. The rows of \Cref{fig:experiments} use $q=\Pp(p^*(X)<\lambda)=0,0.15,0.50$, and the columns use $\eps=0.05,0.1,0.15$.

For each server, we sample a service probability from a low interval with probability $q^{1/K}$ and a high interval otherwise. The low intervals are $[0.57,0.59]$ for $q=0.15$ and $[0.54,0.58]$ for $q=0.50$. Each high interval has width $0.02$ and lies above $\lambda$. Within each interval, we sample uniformly from $256$ equally spaced midpoints. Thus $\Pp(p^*(X)<\lambda)=q$. We choose the high-interval center by bisection so that $\lambda\E[1/p^*(X)]=1-\eps$.

For each sampled probability $p$, set
\[
 \v\phi=\frac{\log(p/(1-p))}{10}\v u
       +\sqrt{25-\frac{\log^2(p/(1-p))}{100}}\v v,
 \qquad \v u=\v\theta^*/10,
\]
where $\v v$ is a unit vector orthogonal to $\v u$. Then $\norm{\v\phi}_2=5$ and $\mu(\ip{\v\phi}{\v\theta^*})=p$. We draw $\v v$ uniformly from $64$ fixed directions.

\para{Implementation.} All five contextual algorithms use $\beta=100$ and the constrained logistic MLE in \Cref{eq:estimator} with $S=10$. For widths and sampling, $\m V_t=\m V_0+\sum_{i=1}^t\v\phi_i\v\phi_i^\top$, with $\m V_0=5\m I_5$ for all five contextual algorithms.
\begin{itemize}
\item \textbf{WISE} uses \Cref{alg:wise} with probability bounds $[\mu(-50),\mu(50)]$.
\item \textbf{CQB-Opt} \citep{bae2026queue} maximizes $\mu(\ip{\v\phi(x,k)}{\widehat{\v\theta}_{t-1}}) +\beta\norm{\v\phi(x,k)}_{\m V_{t-1}^{-1}}$ over waiting pairs.
\item \textbf{CQB-$\eps$} \citep{bae2026queue} explores newly arrived jobs using round-robin servers for the first $200$ rounds, then a uniformly random server with probability $T^{-1/2}$. Otherwise, it uses the CQB-Opt rule.
\item \textbf{CQB-$\eta$-2} \citep{bae2026algorithm} explores newly arrived jobs using a uniformly random server for $t\le0.3T$ and with probability $0.3$ for $0.3T<t\le0.7T$. It uses the CQB-Opt rule whenever it does not explore, including throughout $0.7T<t\le T$.
\item \textbf{ACQB} \citep{bae2026learning} draws five parameter samples from $\mathcal N(\widehat{\v\theta}_{t-1},\beta^2\m V_{t-1}^{-1})$, shared across all pairs, and maximizes the sampled service probability. After an arrival, with probability $(t+1)^{-1/2}$, it instead serves the new job with a round-robin server.
\item \textbf{Q-UCB and Q-ThS} \citep{krishnasamy2021learning} serve the earliest arrival. Q-UCB maximizes the empirical success rate plus $\sqrt{\log^2(\max\{t,2\})/(2N_k)}$, where $N_k$ counts attempts on server $k$ and unobserved servers have infinite index. Q-ThS maximizes a draw from each server's Beta posterior with $\operatorname{Beta}(1,1)$ prior. Both choose a uniformly random server with probability $\min\{1,3K\log^2(\max\{t,2\})/t\}$.
\item \textbf{Random} chooses a uniformly random waiting job and server.
\end{itemize}

\para{Results.} We use ten runs with initially empty queues. Within each run, all algorithms and the oracle share arrivals, features, and service uniforms. The parameter and feature distributions are fixed across runs. \Cref{fig:experiments} reports the mean of $Q^\pi(t)-Q^*(t)$ and $\pm2$ sample standard deviations, each smoothed with a $25$-round moving average.

\subsection*{Use of Generative-AI Tools}
The authors developed the workload potential and the analysis based on the elliptical potential count. GPT-5.6 Sol assisted the authors in designing Algorithm~\ref{alg:wise} and establishing the properties required for this analysis. It also assisted with manuscript polishing, reference exploration, and simulation code development. The authors take responsibility for the final arguments and presentation.

\end{document}